%% file: arXiv_main.tex
\documentclass[nohyperref]{article}
\usepackage{fullpage}
\usepackage{times}
\usepackage{natbib}

\newif\iftwocol

\twocoltrue 
\twocolfalse 

\usepackage{algorithm,algcompatible}
\usepackage{microtype}
\usepackage{graphicx}
\usepackage{subfigure}
\usepackage[makeroom]{cancel}
\usepackage{booktabs} 
\usepackage{makecell}

\usepackage{pifont}
\newcommand{\xmark}{\ding{55}}%
\usepackage[colorlinks=true,citecolor=blue]{hyperref}%
\usepackage{multirow}
\usepackage[table, svgnames, dvipsnames]{xcolor}
\usepackage{makecell, cellspace, caption}
\usepackage{booktabs,caption}
\usepackage[flushleft]{threeparttable}

\newcommand{\nosemic}{\renewcommand{\@endalgocfline}{\relax}}
\newcommand{\dosemic}{\renewcommand{\@endalgocfline}{\algocf@endline}}
\let\oldnl\nl
\newcommand{\nonl}{\renewcommand{\nl}{\let\nl\oldnl}}

\usepackage{amsmath, nccmath}
\usepackage{amssymb}
\usepackage{mathtools}
\usepackage{enumitem}
\usepackage{amsthm}
\usepackage{lipsum}

\usepackage[capitalize,noabbrev]{cleveref}
\allowdisplaybreaks

\input{macros}

\title{Federated Minimax Optimization with Client Heterogeneity}

\author{Pranay Sharma,
Rohan Panda, and
Gauri Joshi\\
Department of Electrical and Computer Engineering,\\
Carnegie Mellon University, Pittsburgh, PA 15213\\
\{pranaysh, rohanpan, gaurij\}@andrew.cmu.edu}

\begin{document}

\maketitle
\begin{abstract}
\input{sections/0_abstract}
\end{abstract}

\input{sections/1_Introduction}
\input{sections/2_Related}
\input{sections/3_Preliminaries}
\input{sections/4_Convergence}
\input{sections/5_Simulations}
\input{sections/6_Conclusion}

\section*{Acknowledgments}
This work was supported in part by NSF grants CCF 2045694, CNS-2112471, and ONR N00014-23-1-2149. Jiarui Li helped with some figures in the paper.

\bibliography{References}
\bibliographystyle{icml2023}


\onecolumn
\appendix
\tableofcontents
\input{Supplementary.tex}
	
\end{document}

%% file: macros.tex
\theoremstyle{plain}
\newtheorem{theorem}{Theorem}
\newtheorem*{theorem*}{Theorem}

\newtheorem{lemma}{Lemma}[section]
\newtheorem{cor}{Corollary}[theorem]
\newtheorem*{cor*}{Corollary}
\theoremstyle{definition}
\newtheorem{definition}{Definition}
\newtheorem{assump}{Assumption}
\theoremstyle{remark}
\newtheorem{remark}{Remark}
\newtheorem*{rem}{Remark}

\newcommand{\mc}{\mathcal}
\newcommand{\mco}{\mathcal O}
\newcommand{\mbb}{\mathbb}
\newcommand{\mbf}{\mathbf}
\newcommand{\mbe}{\mathbb E}

\newcommand{\lp}{\left(}
\newcommand{\rp}{\right)}

\newcommand{\lcb}{\left\{}
\newcommand{\rcb}{\right\}}
\newcommand{\lb}{\left[}
\newcommand{\rb}{\right]}
\newcommand{\lnr}{\left\|}
\newcommand{\rnr}{\right\|}
\newcommand{\lan}{\left\langle}
\newcommand{\ran}{\right\rangle}

\newcommand\norm[1]{\lnr#1\rnr}
\newcommand\normb[1]{\big\|#1\big\|}

\newcommand{\bx}{{\mathbf x}}
\newcommand{\bxt}{{\mathbf x^{(t)}}}
\newcommand{\bxtp}{{\mathbf x^{(t+1)}}}
\newcommand{\bxit}{{\mathbf x_{i}^{(t)}}}

\newcommand{\bxitk}{{\mathbf{x}_{i}^{(t,k)}}}
\newcommand{\bxitj}{{\mathbf{x}_{i}^{(t,j)}}}
\newcommand{\bxitkp}{{\mathbf{x}_{i}^{(t,k+1)}}}

\newcommand{\Tbx}{{\widetilde{\bx}}}

\newcommand{\bbx}{\Bar{\bx}}
\newcommand{\bbxT}{\Bar{\bx}^{(T)}}
\newcommand{\bbxt}{\Bar{\bx}^{(t)}}

\newcommand{\Hbx}{{\widehat{\bx}}}
\newcommand{\Hbxs}{\Hbx^{(s)}}

\newcommand{\by}{{\mathbf y}}
\newcommand{\byt}{{\mathbf y^{(t)}}}
\newcommand{\bytp}{{\mathbf y^{(t+1)}}}
\newcommand{\byit}{{\mathbf y_{i}^{(t)}}}

\newcommand{\byitk}{{\mathbf{y}_{i}^{(t,k)}}}
\newcommand{\byitj}{{\mathbf{y}_{i}^{(t,j)}}}
\newcommand{\byitkp}{{\mathbf{y}_{i}^{(t,k+1)}}}
\newcommand{\Tby}{{\widetilde{\by}}}

\newcommand{\bd}{{\mathbf g}}

\newcommand{\bdxt}{{\bd^{(t)}_{\bx}}}
\newcommand{\bdxit}{{\bd^{(t)}_{\bx,i}}}

\newcommand{\bhxit}{{\mathbf h^{(t)}_{\bx,i}}}
\newcommand{\bhxjt}{{\mathbf h^{(t)}_{\bx,j}}}

\newcommand{\bdyt}{{\bd^{(t)}_{\by}}}
\newcommand{\bdyit}{{\bd^{(t)}_{\by,i}}}

\newcommand{\bhyit}{{\mathbf h^{(t)}_{\by,i}}}

\newcommand{\xiitj}{{\xi_{i}^{(t,j)}}}
\newcommand{\xiitk}{{\xi_{i}^{(t,k)}}}
\newcommand{\xiit}{{\xi_{i}^{(t)}}}

\newcommand{\fedsgda}{Fed-Norm-SGDA}
\newcommand{\fedsgdaplus}{Fed-Norm-SGDA+}

\newcommand{\Aw}{A_w}
\newcommand{\Bw}{B_w}
\newcommand{\Cw}{C_w}
\newcommand{\Dw}{D}
\newcommand{\Ew}{E_w}
\newcommand{\Fw}{F_w}

\newcommand{\Dp}{D}

\newcommand{\bai}{\boldsymbol{a}_i}

\newcommand{\nai}{{\left\| \bai \right\|}}
\newcommand{\bait}{{\bar{\boldsymbol{a}}_{i}^{(t)}}}
\newcommand{\bajt}{{\bar{\boldsymbol{a}}_{j}^{(t)}}}
\newcommand{\nait}{{\| \bait \|}}
\newcommand{\najt}{{\| \bajt \|}}

\newcommand{\aijk}{a_{i}^{(j)} (k)}

\newcommand{\aitk}{a_{i}^{(t,k)}}

\newcommand{\aik}{a_{i}^{(k)}}
\newcommand{\aikt}{a_{i}^{(k)} (\sync_i)}

\newcommand{\clientset}{\mc C^{(t)}}

\newcommand{\St}{\clientset}

\newcommand{\obj}{F}
\newcommand{\surloss}{\widetilde{F}}
\newcommand{\csqdist}{\chi^2_{\mathbf{p}\|\mathbf{w}}}

\newcommand{\TPhi}{{\widetilde{\Phi}}}
\newcommand{\TF}{{\widetilde{F}}}
\newcommand{\sumin}{\sum_{i=1}^n}
\newcommand{\sumintext}{\textstyle \sum_{i=1}^n}

\newcommand{\sumijk}{\sum_{j=0}^{k-1}}
\newcommand{\sumikt}{\sum_{k=0}^{\sync_i-1}}
\newcommand{\sumiktt}{\sum_{k=0}^{\syncit-1}}

\newcommand{\sumtT}{\sum_{t=0}^{T-1}}
\newcommand{\sumtTtext}{\textstyle \sum_{t=0}^{T-1}}
\newcommand{\sumiS}{\sum_{i \in \St}}

\newcommand{\avgtT}{\frac{1}{T} \sumtT}

\newcommand{\TMa}{M_{\mbf a_{-1}}}

\newcommand{\G}{\nabla}
\newcommand{\Gx}{\nabla_x}
\newcommand{\Gy}{\nabla_y}

\newcommand{\CExytk}{\Delta^{(t,k)}_{\bx,\by}}
\newcommand{\CExytj}{\Delta^{(t,j)}_{\bx,\by}}

\newcommand{\CEytk}{\Delta^{(t,k)}_{\by}}

\newcommand{\Lp}{L_{\Phi}}
\newcommand{\Lf}{L_f}

\newcommand{\lrcx}{\eta_x^c}
\newcommand{\lrcy}{\eta_y^c}
\newcommand{\lrcxsq}{[\lrcx]^2}
\newcommand{\lrcysq}{[\lrcy]^2}

\newcommand{\lrsx}{\gamma_x^s}
\newcommand{\lrsy}{\gamma_y^s}
\newcommand{\lrsxsq}{[\lrsx]^2}
\newcommand{\lrsysq}{[\lrsy]^2}

\newcommand{\sync}{\tau}

\newcommand{\syncit}{\tau_i^{(t)}}
\newcommand{\seff}{\sync_{\text{eff}}}
\newcommand{\sefft}{\sync_{\text{eff}}^{(t)}}

\newcommand{\localvar}{\sigma_L}
\newcommand{\varscale}{\beta_L}

\newcommand{\hetero}{\sigma_G}
\newcommand{\heteroscale}{\beta_G}

\newcommand{\selclients}{P}
\newcommand{\numclients}{n}

\newcommand{\wi}{w_i}
\newcommand{\wj}{w_j}
\newcommand{\twi}{\Tilde{w}_i}
\newcommand{\twj}{\Tilde{w}_j}

\newcommand{\nn}{\nonumber}

\usepackage{xcolor}

\DeclareMathOperator*{\argmin}{arg\,min}
\DeclareMathOperator*{\argmax}{arg\,max}

\usepackage{tikz}
\usetikzlibrary{calc,trees,positioning,arrows,chains,shapes.geometric,%
	decorations.pathreplacing,decorations.pathmorphing,shapes,%
	matrix,shapes.symbols}

\tikzstyle{startstop} = [rectangle, draw, rounded corners, align=center, minimum width=3cm, minimum height=1cm,text centered]
\tikzstyle{decision} = [diamond, draw, fill=blue!20, 
text width=4.5em, text badly centered, node distance=3cm, inner sep=0pt]
\tikzstyle{block} = [rectangle, draw, fill=blue!10, align=center, rounded corners, minimum width=3cm, minimum height=1cm]
\tikzstyle{blockcast} = [rectangle, draw, fill=red!10, align=center, rounded corners, minimum width=3cm, minimum height=0.45cm]
\tikzstyle{line} = [draw, -latex']
\tikzstyle{cloud} = [draw, ellipse,fill=red!20, node distance=3cm,
minimum height=2em]

%% file: sections/0_abstract.tex
Minimax optimization has seen a surge in interest with the advent of modern applications such as GANs, and it is inherently more challenging than simple minimization. The difficulty is exacerbated by the training data residing at multiple edge devices or \textit{clients}, especially when these clients can have heterogeneous datasets and local computation capabilities. We propose a general federated minimax optimization framework that subsumes such settings and several existing methods like Local SGDA. We show that naive aggregation of heterogeneous local progress results in optimizing a mismatched objective function -- a phenomenon previously observed in standard federated minimization. To fix this problem, we propose normalizing the client updates by the number of local steps undertaken between successive communication rounds. We analyze the convergence of the proposed algorithm for classes of nonconvex-concave and nonconvex-nonconcave functions and characterize the impact of heterogeneous client data, partial client participation, and heterogeneous local computations. Our analysis works under more general assumptions on the intra-client noise and inter-client heterogeneity than so far considered in the literature. 
For all the function classes considered, we significantly improve the existing computation and communication complexity results. Experimental results support our theoretical claims.

%% file: sections/1_Introduction.tex
\section{Introduction}
\label{sec:intro}

The massive surge in machine learning (ML) research in the past decade has brought forth new applications that cannot be modeled as simple minimization problems. Many of these problems, including generative adversarial networks (GANs) \cite{goodfellow14GANs_neurips, arjovsky17WGANs_icml, sanjabi18GANs_neurips}, adversarial neural network training \cite{madry18adversarial_iclr}, robust optimization \cite{namkoong16SG_DRO_neurips, mohajerin18DRO_mathprog}, distributed nonconvex optimization \cite{lu19block_icassp}, and fair machine learning \cite{madras18learning_icml, mohri19agnosticFL_icml}, have an underlying min-max structure. However, the underlying problem is often nonconvex, while classical minimax theory deals almost exclusively with convex-concave problems.

Another feature of modern ML applications is the inherently distributed nature of the training data \cite{xing2016strategies}. The data collection is often outsourced to edge devices or \textit{clients}. However, the clients may then be unable (due to resource constraints) or unwilling (due to privacy concerns) to share their data with a \textit{central server}. Federated Learning (FL) \cite{konevcny16federated, kairouz19advancesFL_arxiv} was proposed to alleviate this problem. In exchange for retaining control of their data, the clients shoulder some of the computational load, and run part of the training process locally, using only their own data. The communication with the server is infrequent, leading to further resource savings. Since its introduction, FL has been an active area of research, with some remarkable successes \cite{smith20FL_SPmag, wang21field_arxiv}.
Research has shown practical benefits of, and provided theoretical justifications for commonly used practical techniques, such as, multiple local updates at the clients \cite{stich18localSGD_iclr, khaled20localSGD_aistats, koloskova20unified_localSGD_icml, wang21coopSGD_jmlr}, partial client participation \cite{yang21partial_client_iclr}, communication compression \cite{hamer20fedboost_icml, chen21communication_pnas}. Further, impact of heterogeneity in the clients' local data \cite{zhao18FL_noniid_arxiv, sattler19robust_ieeetnn}, as well as their system capabilities \cite{joshi20fednova_neurips, mitra21FedLin_neurips} has been studied. However, all this research has been focused almost solely on simple minimization problems. 

\begin{table*}[t]
\begin{center}
\begin{threeparttable}
\caption{Comparison of (\textbf{per client}) stochastic gradient complexity and the number of communication rounds  needed to reach an $\epsilon$-stationary solution (see \cref{defn:stationarity}), for different classes of nonconvex minimax problems. Here $n$ is the total number of clients. For a fair comparison with existing works, our results in this table are specialized to the case when all clients (i) have equal weights ($p_i=1/\numclients$), (ii) perform equal number of local updates ($\sync_i = \sync$), and (iii) use the same local update algorithm SGDA. However, our results (\cref{sec:algo_theory}) apply under more general settings when (i)-(iii) do not hold. 
}
\label{table:comparison}
\vskip 0.35in
\begin{small}
\begin{tabular}{|c|c|c|cc|}
\hline
\multirow{3}{*}{Work} & \multicolumn{2}{c|}{Setting and Assumptions} & \multicolumn{2}{c|}{Full Client Participation (FCP)} \\
\cline{2-3} \cline{4-5}
 & \makecell{System \\ Heterogeneity\tnote{a}} & \makecell{Partial Client \\ Participation} & \makecell{Stochastic Gradient \\ Complexity} & \makecell{Communication \\ Rounds} \\
\hline
\multicolumn{5}{|c|}{Nonconvex-Strongly-concave (NC-SC)/Nonconvex-Polyak-{\L}ojasiewicz (NC-PL)} \\
\hline
($n=1$) \cite{lin_GDA_icml20} & - & - & $\mco ( 1/\epsilon^{4} )$ & -  \\
\cite{sharma22FedMinimax_ICML} & \xmark & \xmark & $\mco ( 1/(n \epsilon^{4}) )$ & $\mco ( 1/\epsilon^{3} )$ \\
\cite{yang22sagda_neurips}\tnote{b} & \xmark & \checkmark & $\mco ( 1/(n \epsilon^{4}) )$ & $\mco ( 1/\epsilon^{2} )$ \\
\rowcolor{Gainsboro!60} \textbf{Our Work} (\cref{thm:NC_SC}, \cref{cor:NC_SC_comm_cost}) & \checkmark & \checkmark & {\color{red}$\mco \lp 1/(n \epsilon^{4}) \rp$} & {\color{red}$\mco \lp 1/\epsilon^{2} \rp$} \\
\hline
\multicolumn{5}{|c|}{Nonconvex-Concave (NC-C)} \\
\hline
($n=1$) \cite{lin_GDA_icml20} & - & - & $\mco ( 1/\epsilon^{8} )$ & - \\
\cite{sharma22FedMinimax_ICML} & \xmark & \xmark & $\mco (1/(n \epsilon^8))$ & $\mco (1/\epsilon^7)$ \\
\rowcolor{Gainsboro!60} \textbf{Our Work:} (\cref{thm:NC_C}, \cref{cor:NC_C_comm_cost}) & \checkmark & \checkmark & {\color{red}$\mco \lp 1/(n \epsilon^{8}) \rp$} & {\color{red}$\mco \lp 1/\epsilon^{4} \rp$} \\
\hline
\multicolumn{5}{|c|}{Nonconvex-One-point-concave (NC-1PC)} \\
\hline
\cite{mahdavi21localSGDA_aistats} & \xmark & \xmark & $\mco ( 1/\epsilon^{12} )$ & $\mco ( n^{1/6}/\epsilon^{8} )$ \\
\cite{sharma22FedMinimax_ICML} & \xmark & \xmark & $\mco ( 1/\epsilon^{8} )$ & $\mco ( 1/\epsilon^{7} )$ \\
\rowcolor{Gainsboro!60} \textbf{Our Work:} (\cref{thm:NC_1PC}) & \checkmark & \checkmark & {\color{red}$\mco \lp 1/(n \epsilon^8) \rp$} & {\color{red}$\mco \lp 1/\epsilon^4 \rp$} \\
\hline
\end{tabular}
\begin{tablenotes}
    \item[a] Individual clients can run an unequal number of local iterations, using different local optimizers (see \cref{sec:algo_theory}). 
    \item[b] We came across \cite{yang22sagda_neurips} during the preparation of this paper. Our algorithm \fedsgda (\cref{alg_NC_minimax}) strictly generalizes their algorithm FSGDA.
\end{tablenotes}
\end{small}
\vskip -0.1in
\end{threeparttable}
\end{center}
\end{table*}

With its increasing usage in large-scale applications, FL systems must adapt to a wide range of clients. Data heterogeneity has received significant attention from the community. However, system-level heterogeneity remains relatively unexplored. The effect of client variability or \textit{heterogeneity} can be controlled by forcing all the clients to carry out an equal number of local updates and utilize the same local optimizer \cite{yu19icml_momentum, haddadpour19local_SHD_neurips}. However, this approach is inefficient if the client dataset sizes are widely different. Also, it would entail faster clients sitting idle for long durations \cite{reisizadeh22stragglerFL_jsait, tziotis22stragglerFL_arxiv}, waiting for stragglers to finish. Additionally, using the same optimizer might be inefficient or expensive for clients, depending on their system capabilities. Therefore, adapting to system-level heterogeneity forms a desideratum for real-world FL schemes.

\paragraph{Contributions.} We consider a general federated minimax optimization framework, in presence of both inter-client data and system heterogeneity. We consider the problem
\iftwocol
\else
    \begin{align}
        \min_{\bx \in \mbb R^{d_x}} \max_{\by \in \mbb R^{d_y}} \lcb F (\bx, \by) := \textstyle \sum_{i=1}^n p_i f_i(\bx, \by) \rcb, \label{eq:problem}
    \end{align}
\fi
where 
$f_i$ is the local loss of client $i$, $p_i$ is the weight assigned to client $i$, which is often the relative sample size at client $i$, and $n$ is the total number of clients. We study several classes of nonconvex minimax problems. Further,
\begin{itemize}[leftmargin=*]
    \setlength\itemsep{-0.5em}
    \item In our generalized federated minimax algorithm, the participating clients in a round may each perform different number of local steps, potentially with different local optimizers. In this setting, naive aggregation of local model updates (as done in existing methods like Local Stochastic Gradient Descent Ascent) can lead to convergence in terms of a mismatched global objective. We propose a simple normalization strategy to fix this problem.
    \item Using independent server and client learning rates, we achieve order-optimal or state-of-the-art computation complexity, and significantly improve the communication complexity of existing methods. 
    \item Under the special case where all the clients (i) are assigned equal weights $p_i=1/n$ in \eqref{eq:problem}, (ii) carry out equal number of local updates ($\sync_i = \sync$ for all $i$), and (iii) utilize the same local-update algorithm, our results become directly comparable with existing work (see \cref{table:comparison}) and improve upon them as follows.
    \begin{enumerate}
        \item For nonconvex-strongly-concave (NC-SC) and nonconvex-PL (NC-PL) problems, our method has the order-optimal gradient complexity $\mco (1/(\numclients \epsilon^4))$. Further, we improve the communication from $\mco (1/\epsilon^3)$ in \cite{sharma22FedMinimax_ICML} to $\mco (1/\epsilon^2)$.\footnote{During the preparation of this manuscript, we came across the recent work \cite{yang22sagda_neurips}, which proposes FSGDA algorithm and achieves $\mco (1/ \epsilon^{2})$ communication cost for NC-PL functions. However, our work is more general since we allow a heterogeneous number of local updates at the clients.}
        \item For nonconvex-concave (NC-C) and nonconvex-one-point-concave (NC-1PC) problems, we achieve state-of-the-art gradient complexity, while significantly improving the communication costs from $\mco (1/\epsilon^7)$ in \cite{sharma22FedMinimax_ICML} to $\mco (1/\epsilon^4)$. For NC-1PC functions, we prove the linear speedup in gradient complexity with $\numclients$ that was conjectured in \cite{sharma22FedMinimax_ICML}, thereby solving an open problem.
        \item As an intermediate result in our proof, we prove the theoretical convergence of Local SGD for one-point-convex function minimization (see \cref{lem:NC_1PC_PCP_local_SGA_concave} in \cref{app:NC_1PC}). The achieved convergence rate is the same as that achieved for convex minimization problems. Therefore, we generalize the convergence of Local SGD to a much larger class of functions.
    \end{enumerate}
\end{itemize}

%% file: sections/2_Related.tex
\section{Related Work}
\label{sec:related}

\subsection{Single-client minimax}
\paragraph{Nonconvex-Strongly-concave (NC-SC).}
To our knowledge, \cite{lin_GDA_icml20} is the first work to analyze a single-loop algorithm for stochastic (and deterministic) NC-SC problems. Although the $\mco (\kappa^3/\epsilon^4)$ complexity shown is optimal in $\epsilon$, the algorithm required $\mco (\epsilon^{-2})$ batch-size. \cite{qiu20single_timescale_ncsc} utilized momentum to achieve $\mco (\epsilon^{-4})$ convergence with $\mco (1)$ batch-size. Recent works \cite{yang22NC_PL_aistats, sharma22FedMinimax_ICML} achieve the same rate without momentum. \cite{yang22NC_PL_aistats} also improved the dependence on the condition number $\kappa$. Second-order stationarity for NC-SC has been recently studied in \cite{luo21SoS_NCSC_arxiv}. Lower bounds for this problem class have appeared in \cite{luo20SREDA_ncsc_neurips, li21lower_bd_NCSC_neurips, zhang21NCSC_uai}.

\paragraph{Nonconvex-Concave (NC-C).}
Again, \cite{lin_GDA_icml20} was the first to analyze a single-loop algorithm for stochastic NC-C problems, proving $\mco (\epsilon^{-8})$ complexity. In deterministic problems, this has been improved using nested \cite{nouiehed19minimax_neurips19, thekumparampil19NC_C_neurips} as well as single-loop \cite{lan_unified_ncc_arxiv20, zhang_1_loop_ncc_neurips20} algorithms. For stochastic problems, \cite{rafique18WCC_oms} and the recent work \cite{zhang2022sapd_NC_C_arxiv} improved the complexity to $\mco (\epsilon^{-6})$. However, both the algorithms have a nested structure, which at every step, solve a simpler problem iteratively. Achieving $\mco (\epsilon^{-6})$ complexity with a single-loop algorithm has so far proved elusive.

\subsection{Distributed/Federated Minimax}
Recent years have also seen an increasing body of work in distributed minimax optimization. Some of this work is focused on decentralized settings, as in \cite{rogozin21dec_local_global_var_cc_arxiv, gasnikov21decen_deter_cc_icoa, beznosikov21dist_sp_neurips, metelev22decSPP_arxiv}.

Of immediate relevance to us is the federated setting, where clients carry out multiple local updates between successive communication rounds. The relevant works which focused on convex-concave problems include \cite{reisizadeh20robustfl_neurips, hou21FedSP_arxiv, liao21local_AdaGrad_CC_arxiv, sun22comm_SCSC_arxiv}. Special classes of nonconvex minimax problems in the federated setting have been studied in recent works, such as, nonconvex-linear \cite{mahdavi20dist_robustfl_neurips}, nonconvex-PL \cite{mahdavi21localSGDA_aistats, xie21NC_PL_FL_arxiv}, and nonconvex-one-point-concave \cite{mahdavi21localSGDA_aistats}. The complexity guarantees for several function classes considered in \cite{mahdavi21localSGDA_aistats} were further improved in \cite{sharma22FedMinimax_ICML}. However, all these works consider specialized federated settings, either assuming full-client participation, or system-wise identical clients, each carrying out equal number of local updates. As we see in this paper, partial client participation is the most source of error in simple FL algorithms. Also, system-level heterogeneity can have crucial implications on the algorithm performance.

\paragraph{Differences from Related Existing Work.}
\cite{joshi20fednova_neurips} was the first work to consider the problem of system heterogeneity in simple minimization problems, and proposed a normalized averaging scheme to avoid optimizing an inconsistent objective. Compared to \cite{joshi20fednova_neurips}, we consider a more challenging problem and achieve higher communication savings (\cref{table:comparison}). \cite{yang21partial_client_iclr} analyzed partial client participation in FL and demonstrated the theoretical benefit of using separate client/server learning rates. \cite{mahdavi21localSGDA_aistats, sharma22FedMinimax_ICML} studied minimax problems in the federated setting but assumed a homogeneous number of local updates, with full client participation. The very recent work \cite{yang22sagda_neurips} considers NC-SC problem and achieves similar communication savings as ours. However, our work considers a more general minimax FL framework with system-level client heterogeneity, and partial client participation. We consider several classes of functions and improve the communication and computation complexity of existing minimax algorithms.

%% file: sections/3_Preliminaries.tex
\section{Preliminaries}
\label{sec:prelim}

\paragraph{Notations.} We let $\| \cdot \|$ denote the Euclidean norm $\| \cdot \|_2$.
Given a positive integer $m$, the set $\{ 1, 2, \hdots, m \}$ is denoted by $[m]$. 
Vectors at client $i$ are denoted with subscript $i$, e.g., $\bx_i$, while iteration indices are denoted using superscripts, e.g., $\by^{(t)}$ or $\by^{(t,k)}$.
Given a function $g$, we define its gradient vector as $\lb \Gx g(\bx, \by)^{\top}, \Gy g(\bx, \by)^{\top} \rb^{\top}$,
and its stochastic gradient as $\G g(\bx, \by; \xi)$, where $\xi$ denotes the randomness.

\paragraph{Convergence Metrics.}
\iftwocol

\else
    In presence of nonconvexity, we can only prove convergence to an \textit{approximate} stationary point, which is defined next.
\fi

\begin{definition}[$\epsilon$-Stationarity]
\label{defn:stationarity}
A point $\bx$ is an $\epsilon$-stationary point of a differentiable function $g$ if $\norm{\G g (\bx)} \leq \epsilon$.
\end{definition}

\iftwocol

\else
\begin{definition}
Stochastic Gradient (SG) complexity is the total number of gradients computed by a single client during the course of the algorithm.
\end{definition}
\fi

\begin{definition}[Communication Rounds]
During a single communication round, the server sends its \textit{global} model to a set of clients, which carry out multiple local updates starting from the same model, and return their \textit{local} vectors to the server. The server then aggregates these local vectors to arrive at a new global model. Throughout this paper, we denote the number of communication rounds by $T$.
\end{definition}
Next, we discuss some assumptions used in the paper.

\begin{assump}[Smoothness]
\label{assum:smoothness}
Each local function $f_i$ is differentiable and has Lipschitz continuous gradients.
That is, there exists a constant $\Lf > 0$ such that at each client $i \in [n]$, for all $\bx, \bx' \in \mbb R^{d_1}$ and $\by, \by' \in \mbb R^{d_2}$,
\iftwocol
\else
    \begin{align*}
        \lnr \G f_i(\bx, \by) - \G f_i(\bx', \by') \rnr \leq \Lf \lnr (\bx, \by) - (\bx', \by') \rnr.
    \end{align*}
\fi
\end{assump}

\begin{assump}[\textit{Local} Variance]
\label{assum:bdd_var}
The stochastic gradient oracle at each client is \textit{unbiased}. Also, there exist constants $\localvar, \varscale \geq 0$ such that at each client $i \in [n]$, for all $\bx, \by$,
\iftwocol
\else
    \begin{align*}
        \mbe_{\xi_i} [ \G f_i(\bx, \by; \xi_i) ] &= \G f_i(\bx, \by), \\
        \mbe_{\xi_i} \| \G f_i(\bx, \by; \xi_i) - \G f_i(\bx, \by) \|^2 & \leq \varscale^2 \norm{\G f_i(\bx, \by)}^2 + \localvar^2.
    \end{align*}
\fi

\end{assump}

\begin{assump}[\textit{Global} Heterogeneity]
\label{assum:bdd_hetero}
For any set of non-negative weights $\{ w_i \}_{i=1}^n$ such that $\textstyle \sum_{i=1}^n w_i=1$, there exist constants $\heteroscale \geq 1, \hetero \geq 0$ such that for all $\bx, \by$,
\iftwocol
\else
    \begin{align*}
        \sumin w_i \lnr \Gx f_i \lp \bx, \by \rp \rnr^2 & \leq \heteroscale^2 \lnr \sumin w_i \Gx f_i \lp \bx, \by \rp \rnr^2 + \hetero^2, \\
        \sumin w_i \lnr \Gy f_i \lp \bx, \by \rp \rnr^2 & \leq \heteroscale^2 \lnr \sumin w_i \Gy f_i \lp \bx, \by \rp \rnr^2 + \hetero^2.
    \end{align*}
\fi
If all $f_i$'s are identical, we have $\heteroscale = 1$, and $\hetero = 0$.
\end{assump}
Most existing work uses simplified versions of Assumptions \ref{assum:bdd_var}, \ref{assum:bdd_hetero}, assuming $\varscale=0$ and/or $\heteroscale=0$.

%% file: sections/4_Convergence.tex
\section{Algorithm for Heterogeneous Federated Minimax Optimization}
\label{sec:algo_theory}
In this section, we propose a federated minimax algorithm to handle system heterogeneity across clients.

\subsection{Limitations of Local SGDA}
Following the success of FedAvg \cite{fedavg17aistats} in FL, \cite{mahdavi21localSGDA_aistats} was the first to explore a simple extension Local stochastic gradient descent-ascent (SGDA) in minimax problems. Between successive communication rounds, clients take multiple simultaneous descent/ascent steps to respectively update the min-variable $\bx$ and max-variable $\by$. Subsequent work in \cite{sharma22FedMinimax_ICML} improved the convergence results and showed that LocalSGDA achieves optimal gradient complexity for several classes of nonconvex minimax problems. However, existing work on LocalSGDA also assumes the participation of all $n$ clients in every communication round. More crucially, as observed with simple minimization problems \cite{joshi20fednova_neurips}, if clients carry out an unequal number of local updates, or if their local optimizers are not all the same, LocalSGDA (like FedAvg) might converge to the stationary point of a different objective. This is further discussed in \cref{sec:NC_SC}, \ref{sec:NC_C}, and illustrated in \cref{fig:localSGD_hetero},
where the learning process gets disproportionately skewed towards the clients carrying out more local updates.
\iftwocol
\else
    \begin{figure}[h]
         \centering
         \includegraphics[width=0.65\textwidth]{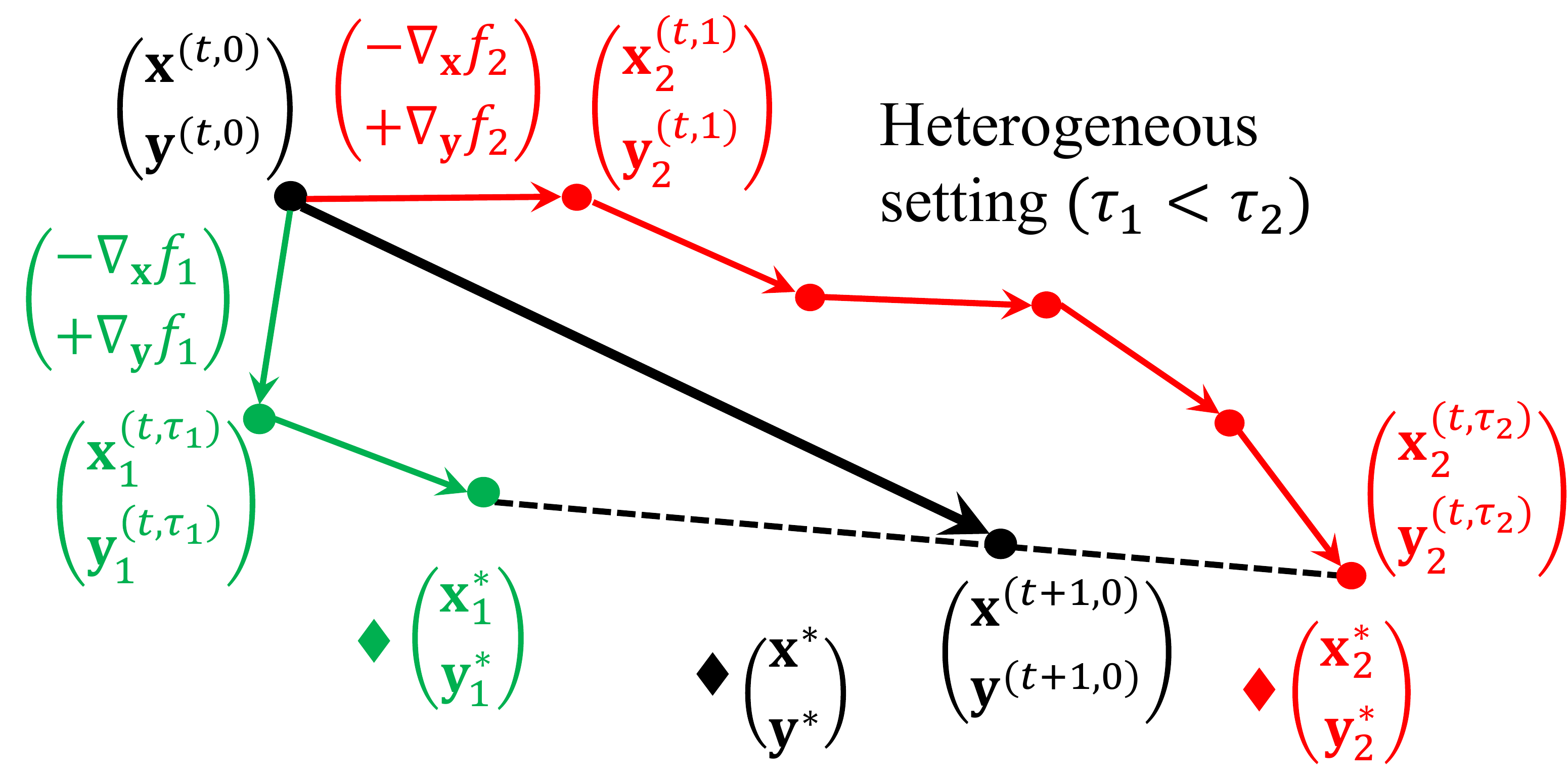}
         \caption{FedAvg with heterogeneous local updates. The green (red) triangle represents the local optimizer of $f_1 (f_2)$, while $(\bx^*, \by^*)$ is the global optimizer. The number of local updates at client $i$ is $\sync_i$, where $\sync_1 = 2$, $\sync_2 = 5$).
         \label{fig:localSGD_hetero}
         }
    \end{figure}
\fi
\paragraph{Generalized Local SGDA Update Rule.} To understand this mismatched convergence phenomenon with naive aggregation in local SGDA, recall that Local SGDA updates are of the form 
\iftwocol
\else
    $$\bxtp = \bxt + \lrsx \sumin p_i \Delta_{\bx, i}^{(t)}, \qquad \bytp = \byt + \lrsy \sumin p_i \Delta_{\by, i}^{(t)},$$
    where $\lrsx, \lrsy$ are the server learning rates, $\Delta_{\bx, i}^{(t)} = \frac{1}{\lrcx} \big( \bx_i^{(t,\syncit)} - \bxt \big)$, $\Delta_{\by, i}^{(t)} = \frac{1}{\lrcy} \big( \by_i^{(t,\syncit)} - \byt \big)$ are the scaled local updates. $\bx_i^{(t,\syncit)}$ is the iterate at client $i$ after taking $\syncit$ local steps, and $\lrcx, \lrcy$ are the client learning rates. Let us consider a generalized version of this update rule where $\Delta_{\bx, i}^{(t)}, \Delta_{\by, i}^{(t)}$ are linear combinations of local stochastic gradients computed by client $i$, as $\Delta_{\by, i}^{(t)} = \sumiktt \aitk \Gy f_i(\bxitk, \byitk; \xiitk)$, where $\aitk \geq 0$. 
    Commonly used client optimizers, such as, SGD, local momentum, variable local learning rates can be accommodated in this general form (see \cref{app:grad_agg} for some examples). For this more general form, we can rewrite the $\bx, \by$ updates at the server as follows
    \begin{equation}
        \begin{aligned}
            & \bxtp = \bxt - \lrsx \textstyle\sumin p_i \mbf G_{\bx,i}^{(t)} \frac{\bait}{\nait_1} \nait_1 \\
            &= \bxt - \underbrace{\Big( \sum_{j=1}^{\numclients} p_j \najt_1 \Big)}_{\sefft} \lrsx \sumin \underbrace{\mfrac{p_i \nait_1}{\textstyle \sum_{j=1}^\numclients p_j \najt_1}}_{\wi} \underbrace{\mfrac{\mbf G_i^{(t)} \bait}{\nait_1}}_{\bdxit}, \\
            & \bytp = \byt + \sefft \lrsy \textstyle\sumin \wi \bdyit,
        \end{aligned}
        \label{eq:wtd_server_updates}
    \end{equation}
\fi
where $G_{\bx,i}^{(t)} = [\Gy f_i(\bxitk, \byitk; \xiitk)]_{k=0}^{\syncit} \in \mbb R^{d_x \times \syncit}$ contains the $\syncit$ stochastic gradients stacked column-wise, $\bait = [a_i^{t,0}, a_i^{t,1}, \dots, a_i^{t,\syncit-1}]^\top$, $\bdxit, \bdyit$ are the normalized aggregates of the stochastic gradients and $\sefft$ is the \textit{effective} number of local steps.
\iftwocol
\else
    \begin{figure}[thb]
         \centering
         \includegraphics[width=0.65\textwidth]{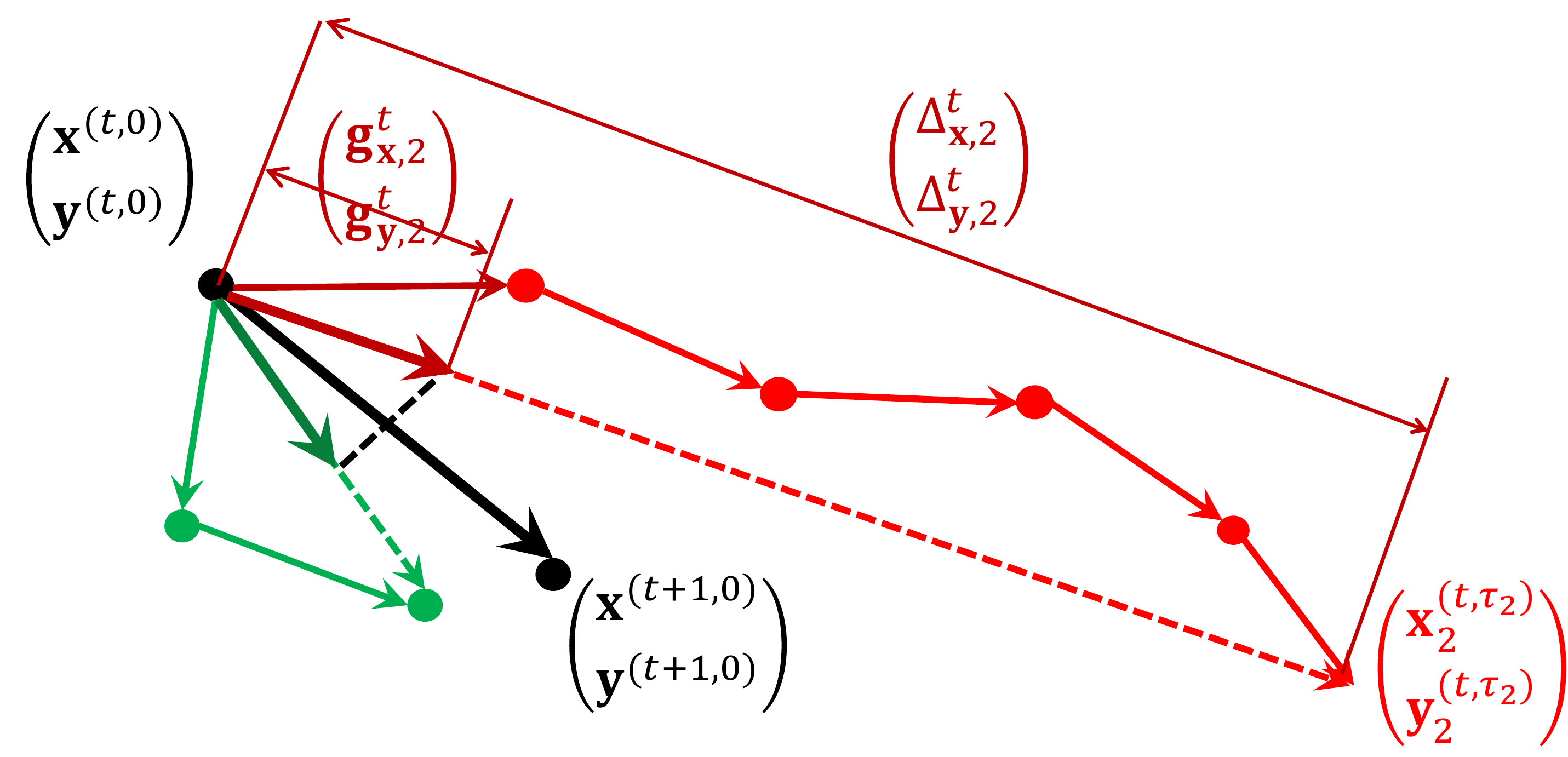}
         \caption{Generalized update rule in \eqref{eq:wtd_server_updates}. Note that $(\bdxit, \bdyit) = \frac{1}{\sync_i} (\Delta_{\bx, i}^{(t)}, \Delta_{\by, i}^{(t)})$. Also, at the server, the weighted sum $\sumin \wi \bdxit$ gets scaled by $\sefft$.
         }
    \end{figure}
\fi
Similar to the observation for simple minimization problems in \cite{joshi20fednova_neurips}, we see in Theorems \ref{thm:NC_SC}, \ref{thm:NC_C} that the resulting iterates of this general algorithm end up converging to the stationary point of a different objective $\TF = \sumin \wi f_i$. Further, in \cref{cor:obj_inconsistent}, we observe that this mismatch is a result of using weights $\wi$ in \eqref{eq:wtd_server_updates} to weigh the clients' contribution.

\iftwocol
\else
    \begin{algorithm}[t!]
    \caption{
    \colorbox{blue!20}{\fedsgda} and \colorbox{red!20}{\fedsgdaplus}
    }
    \label{alg_NC_minimax}
    \begin{algorithmic}[1]
        \STATE{\textbf{Input:} initialization $\bx^{(0)}, \by^{(0)}$, Number of communication rounds $T$, learning rates: client $\{ \lrcx, \lrcy \}$, server $\{ \lrsx, \lrsy \}$, \#local-updates $\{ \syncit \}_{i,t}$, \colorbox{red!20}{$S$, $s=-1$}}
    	\FOR{$t=0$ to $T-1$}
                \STATE{Server selects client set $\clientset$; sends them $(\bxt, \byt)$}    
                \IF{\colorbox{red!20}{$t$ mod $S = 0$}}
                    \STATE{\colorbox{red!20}{$s \gets s+1$}}
                    \STATE{\colorbox{red!20}{Server sends $\Hbxs = \bxt$ to clients in $\clientset$}}
                \ENDIF
    	        \STATE{$\bx^{(t,0)}_i = \bxt$, $\by^{(t,0)}_i = \byt$ for $i \in \clientset$}
    	        \FOR{$k = 0, \hdots, \syncit-1$} \label{line:alg_NC_minimax:client_update_start}
    	            \STATE{$\bxitkp = \bxitk - \lrcx \aitk \Gx f_i (\bxitk, \byitk; \xiitk)$}
    	            \STATE{\colorbox{red!20}{$\byitkp = \byitk + \lrcy \aitk \Gy f_i ({\color{red}\Hbxs}, \byitk; \xiitk)$}\hfill {\color{red}\# $\by$-update for \fedsgdaplus}}
    	            \STATE{\colorbox{blue!20}{$\byitkp = \byitk + \lrcy \aitk \Gy f_i ({\color{blue}\bxitk}, \byitk; \xiitk)$}\hfill {\color{blue}\# $\by$-update for \fedsgda}}
    	        \ENDFOR \label{line:alg_NC_minimax:client_update_end}
        	   \STATE{Client $i$ aggregates its gradients to compute $\bdxit, \bdyit$}\label{line:alg_NC_minimax:client_agg_start}
        	   \STATE{$\bdxit = \sumiktt \mfrac{\aitk}{\nait_1} \Gx f_i ( \bxitk, \byitk; \xiitk )$}
    	       \STATE{\colorbox{red!20}{$\bdyit = \sumiktt \mfrac{\aitk}{\nait_1} \Gy f_i ({\color{red}\Hbxs}, \byitk; \xiitk )$}}
    	       \STATE{\colorbox{blue!20}{$\bdyit = \sumiktt \mfrac{\aitk}{\nait_1} \Gy f_i ({\color{blue}\bxitk}, \byitk; \xiitk )$}}
    	       \label{line:alg_NC_minimax:client_agg_end}
    	   \STATE{Clients $i \in \clientset$ communicate $\{ \bdxit, \bdyit \}$ to the server}
            \STATE{ Server computes aggregate vectors $\{ \bdxt, \bdyt\}$ using \eqref{eq:server_agg_WOR}}
            \STATE{ Server step: 
            $\left\{\begin{matrix}
            \bxtp = \bxt - \sefft \lrsx \bdxt, \quad 
            \bytp = \byt + \sefft \lrsy \bdyt
            \end{matrix}\right.$}
    	\ENDFOR
    	\STATE{\textbf{Return: }$\bbxT$ drawn uniformly at random from $\{ \bxt \}_{t=1}^T$}
    \end{algorithmic}
    \end{algorithm}
\fi

\subsection{Proposed Normalized Federated Minimax Algorithm} 
From the generalized update rule, we can see that setting the weights $w_i$ equal to $p_i$ will ensure that the surrogate objective $\tilde{F}$ matches with the original global objective $F$. Setting $w_i = p_i$ results in normalization of the local progress at each client before their aggregation at the server. As a result, we can preserve convergence to a stationary point of the original objective function $F$, even with heterogeneous $\{ \syncit \}$, as we see in \cref{thm:NC_SC} and \cref{thm:NC_C}. 

The algorithm follows the steps given in \Cref{alg_NC_minimax}. In each communication round $t$, the server selects a client set $\clientset$ and communicates its model parameters $(\bxt, \byt)$ to these clients. The selected clients then run multiple local stochastic gradient steps.
The number of local steps $\{ \syncit \}$ can vary across clients and across rounds. At the end of $\syncit$ local steps, client $i$ aggregates its local stochastic gradients into $\{ \bdxit, \bdyit \}$,
which are then sent to the server. 
\iftwocol
\else
    Note that the gradients at client $i$, $\{ \G f_i (\cdot, \cdot; \xiitk) \}_{k=0}^{\syncit}$, are normalized by $\nait_1$, where $\bait = [a_i^{t,0}, a_i^{t,1}, \dots, a_i^{t,\syncit-1}]^\top$ is the vector of weights assigned to individual stochastic gradients in the local updates.\footnote{For LocalSGDA \cite{mahdavi21localSGDA_aistats, sharma22FedMinimax_ICML}, $\aitk = 1$ for all $i \in [n], t \in [T], k \in [\syncit]$ and $\nait_1 = \syncit$. Therefore, $\bdxit, \bdyit$ are simply the average of the stochastic gradients computed in the $t$-th round.} 
\fi
The server aggregates these local vectors to compute global direction estimates $\bdxt, \bdyt$,
which are then used to update the server model parameters $(\bxt, \byt)$.

\paragraph{Client Selection.} In each round $t$, the server samples $|\clientset|$ clients uniformly at random \textit{without replacement} (WOR). While aggregating client updates at the server, client $i$ update is weighed by $\twi = \wi \numclients/|\clientset|$, i.e.,
\iftwocol
\else
    \begin{align}
        \bdxt = \sumiS \twi \bdxit, \qquad \bdyt = \sumiS \twi \bdyit. \label{eq:server_agg_WOR}
    \end{align}
\fi
Note that $\mbe_{\clientset} [ \bdxt ] = \sumin \wi \bdxit, \mbe_{\clientset} [ \bdyt ] = \sumin \wi \bdyit$.

\section{Convergence Results}
\label{sec:conv_results}
Next, we present the convergence results for different classes of nonconvex minimax problems. For simplicity, throughout this section we assume the parameters utilized in \cref{alg_NC_minimax} to be fixed across $t$. Therefore, $\aitk \equiv \aik$, $\bait \equiv \bai$, $\syncit \equiv \sync_i$, $\sefft \equiv \seff$ and $|\clientset| = \selclients$, for all $t$. 

\subsection{Non-convex-Strongly-Concave (NC-SC) Case}
\label{sec:NC_SC}

\begin{assump}[$\mu$-Strong-concavity (SC) in $\by$]
\label{assum:SC_y}
A function $f$ is $\mu$-strong concave ($\mu > 0$) in $\by$ if for all $\bx, \bar{\by}, \by_2$,
\iftwocol
\else
    $$-f(\bx, \by_2) \geq -f(\bx, \bar{\by}) - \lan \Gy f(\bx, \bar{\by}), \by_2-\bar{\by} \ran + \frac{\mu}{2} \normb{\by_2-\bar{\by}}^2.$$
\fi
\end{assump}

\paragraph{General Convergence Result.}
We first show that using the local updates of \cref{alg_NC_minimax}, the iterates converge to the stationary point of a surrogate objective $\TF$, $\TF (\bx, \by) \triangleq \textstyle \sum_{i=1}^n w_i f_i(\bx, \by)$. 
See \cref{app:NC_SC} for the full statement.

\begin{theorem}
\label{thm:NC_SC}
Suppose the local loss functions $\{ f_i \}_i$ satisfy Assumptions \ref{assum:smoothness}, \ref{assum:bdd_var}, \ref{assum:bdd_hetero}, \ref{assum:SC_y}. Suppose the server selects $|\clientset| = \selclients$ clients in each round $t$. Given appropriate choices of client and server learning rates, $(\lrcx, \lrcy)$ and $(\lrsx, \lrsy)$ respectively, the iterates generated by \fedsgda \ satisfy
\iftwocol
\else
    \begin{align}
        & \min_{t \in [T]} \mbe \big\| \G \TPhi(\bxt) \big\|^2 \leq \mco \Big( \underbrace{\kappa^2 \mfrac{(\bar{\sync}/\seff) + \Aw \localvar^2 + \Bw \varscale^2 \hetero^2}{\sqrt{\selclients \bar{\sync} T}}}_{\text{Error with full synchronization}} \Big) + \underbrace{\mco \Big( \kappa^2 \mfrac{\Cw \localvar^2 + \Dw \hetero^2}{\bar{\sync}^2 T} \Big)}_{\substack{\text{Error due to local updates}}} + \underbrace{\mco \Big( \mfrac{\numclients - \selclients}{\numclients - 1} \cdot \mfrac{\kappa^2 \Ew \seff \hetero^2}{\sqrt{\selclients \bar{\sync} T}} \Big)}_{\substack{\text{Partial Participation Error}}}. \label{eq:thm:NC_SC}
    \end{align}
    where, 
    $\kappa = \Lf/\mu$ is the condition number,
    $\TPhi(\bx) \triangleq \max_\by \TF (\bx, \by)$ is the envelope function, 
    $\bar{\sync} = \frac{1}{\numclients} \sumin \sync_i$,
    $\Aw \triangleq n \seff \sumin \frac{w_i^2 \nai_2^2}{\nai_1^2}$, 
    $\Bw \triangleq n \seff \max_i \frac{w_i \nai_2^2}{\nai_1^2}$, 
     $\Cw \triangleq \sumin \wi ( \nai_2^2 - [\alpha^{(t,\sync_i-1)}_{i}]^2 )$, 
     $\Dw \triangleq \max_i (\varscale^2\norm{\mbf a_{i,-1}}_2^2 + \norm{\mbf a_{i,-1}}_1^2)$, where $\mbf a_{i,-1}  \triangleq [a_i^{(0)}, a_i^{(1)}, \dots, a_i^{(\sync_i-2)}]^\top$ for all $i$ 
    and
    $\Ew \triangleq \numclients \max_i w_i$.
\fi
\end{theorem}

See \cref{app:NC_SC} for the proof. The first term in the bound in \eqref{eq:thm:NC_SC} represents the optimization error for a centralized algorithm (see Appendix C.3 in \cite{lin_GDA_icml20}). The second term represents the error if at least one of the clients carries out multiple $(\sync_i > 1)$ local updates. The last term results from client subsampling. This also explains its dependence on the data heterogeneity $\hetero$.

\cref{thm:NC_SC} states convergence for a surrogate objective $\TF$. Next, we see convergence for the true objective $F$.

\begin{cor}[Convergence in terms of $F$]
\label{cor:obj_inconsistent}
Given $\Phi(\bx) \triangleq \max_\by F (\bx, \by)$, under the conditions of \cref{thm:NC_SC},
\iftwocol
\else
    \begin{align}
        \min_{t\in[T]} \normb{\nabla \Phi(\bx^{(t)})}^2
        & \leq 2 \lp 2 \csqdist \beta_H^2 + 1 \rp \epsilon_{\text{opt}} + 4 \csqdist \hetero^2 \nn \\
        & \qquad \quad + \mfrac{4 \Lf^2}{T} \sumtTtext \normb{\by^*(\bxt) - \Tby^*(\bxt)}^2. \label{eq:thm:NC_SC_cor_true_obj}
    \end{align}
    where $\csqdist \triangleq \sumin \frac{(p_i-w_i)^2}{w_i}$, $\epsilon_{\text{opt}} \triangleq \frac{1}{T} \sum_{t=0}^{T-1} \norm{\nabla \TPhi(\bx^{(t)})}^2$ denotes the optimization error in \eqref{eq:thm:NC_SC}. If $p_i = w_i$ for all $i \in [n]$, then $\csqdist = 0$. Also, then $\TF(\bx, \by) \equiv F(\bx, \by)$. Therefore, $\by^*(\bx) = \argmax_\by F(\bx, \by)$ and $\Tby^*(\bx) = \argmax_\by \TF(\bx, \by)$ are identical, for all $\bx$. 
    Hence, \eqref{eq:thm:NC_SC_cor_true_obj} yields
    $\min_{t\in[T]} \normb{\nabla \Phi(\bx^{(t)})}^2 \leq 2 \epsilon_{\text{opt}}$.
\fi
\end{cor}

It follows from \cref{cor:obj_inconsistent} that if we replace $\{\wi\}$ with $\{p_i\}$ in the server updates in \cref{alg_NC_minimax}, we get convergence in terms of the true objective $F$.

\begin{remark}
If clients are weighted equally ($w_i = 1/n$ for all $i$), with each carrying out $\sync$ steps of local SGDA, we get {\small$\bar{\sync} = \sync, \Aw = \Bw = 1, \Cw = \sync-1$}, and {\small$D = (\sync-1)(\sync-1+\varscale^2)$}. Therefore, the bound in \eqref{eq:thm:NC_SC} can be greatly simplified to
\iftwocol
\else
    \begin{align}
        & \mco \Big( \mfrac{\localvar^2 + \varscale^2 \hetero^2}{\sqrt{\selclients \bar{\sync} T}} + \mfrac{\localvar^2 + \sync \hetero^2}{\sync T} + \lp \mfrac{\numclients - \selclients}{\numclients - 1} \rp \hetero^2 \sqrt{\mfrac{\sync}{\selclients T}} \Big).
        \label{eq:thm:NC_SC_simplified}
    \end{align}
\fi
Several key insights can be derived from \eqref{eq:thm:NC_SC_simplified}.
\begin{itemize}[leftmargin=*]
    \setlength\itemsep{-0.5em}
    \item Partial Client Participation (PCP) error $\mco \big( \frac{\numclients - \selclients}{\numclients - 1} \cdot \hetero^2 \sqrt{\frac{\sync}{\selclients T}} \big)$ is the \textit{most significant} component of convergence error. Further, unlike the other two errors, the error due to PCP actually increases with local updates $\sync$. Consequently, we do not observe communication savings by performing multiple local updates at the clients, except in the special case when $\hetero = 0$ (see \cref{table:comparison}). Similar observations have been made for minimization \cite{yang21partial_client_iclr, jhunjhunwala22fedvarp_uai} and very recently for minimax problems \cite{yang22sagda_neurips}.
    \item In the absence of multiple local updates (i.e., $\sync_i = 1$ for all $i$) and with full participation ($\selclients = \numclients$), the resulting error $\mco \big( \frac{\localvar^2 + \varscale^2 \hetero^2}{\sqrt{\numclients \bar{\sync} T}} \big)$ depends on the global heterogeneity $\hetero$ despite full synchronization. This is owing to the more general local variance bound (\cref{assum:bdd_var}). For $\varscale > 0$, this dependence on $\hetero$ is unavoidable. 
    This observation holds for all the results in this paper. 
    See \cref{rem:impact_varscale} (\cref{sec:aux}) for a justification. 
\end{itemize}
\end{remark}

\begin{cor}[Improved Communication Savings]
\label{cor:NC_SC_comm_cost}
Suppose all the clients are weighted equally ($p_i = 1/n$ for all $i$), with each carrying out $\sync$ steps of local SGDA.
To reach an $\epsilon$-stationary point, i.e., $\bx$ such that $\mbe \| \G \Phi (\bx) \| \leq \epsilon$, 
\begin{itemize}[leftmargin=*]
    \setlength\itemsep{-0.5em}
    \item Under full participation, the per-client gradient complexity of \fedsgda \ is 
    \iftwocol
\else
        $T \sync = \mco \lp \frac{\kappa^4}{\numclients \epsilon^4} \rp$.
    \fi
    The number of communication rounds required is
    \iftwocol
\else
        $T = \mco \lp \frac{\kappa^2}{\epsilon^{2}} \rp$.
    \fi
    \item Under partial participation, in the special case when inter-client data heterogeneity $\hetero = 0$, the per-client gradient complexity of \fedsgda \ is $\mco (\kappa^4/(\selclients \epsilon^4))$, while the communication cost is $\mco ( \kappa^2/\epsilon^{2} )$.
\end{itemize}
\end{cor}

\begin{remark}
The gradient complexity in \cref{cor:NC_SC_comm_cost} is optimal in $\epsilon$, and achieves linear speedup in the number of participating clients.
The communication complexity is also optimal in $\epsilon$ and improves the corresponding results in \cite{mahdavi21localSGDA_aistats, sharma22FedMinimax_ICML}. We match the communication cost in the very recent work \cite{yang22sagda_neurips}. However, our work addresses a more realistic FL setting with disparate clients.
\end{remark}

\paragraph{Extending the Results to Nonconvex-PL Functions}
\begin{assump}
\label{assum:PL_y}
A function $f$ satisfies $\mu$-PL condition in $\by$ ($\mu > 0$), if for any fixed $\bx$: 
\iftwocol
\else
    \begin{enumerate}
        \item $\max_{\by'} f(\bx, \by')$ has a nonempty solution set;
        \item $\norm{\Gy f(\bx, \by)}^2 \geq 2 \mu ( \max_{\by'} f(\bx, \by') - f(\bx, \by) )$, for all $\by$.
    \end{enumerate}
\fi
\end{assump}

\begin{remark}
If the local functions $\{ f_i \}$ satisfy Assumptions \ref{assum:smoothness}, \ref{assum:bdd_var}, \ref{assum:bdd_hetero}, and the global function $F$ satisfies \cref{assum:PL_y}, then for appropriately chosen learning rates (see \cref{app:NC_PL_PCP_result}), the bounds in \cref{thm:NC_SC} hold for $\mu$-PL functions as well.
\end{remark}

\subsection{Non-convex-Concave (NC-C) Case}
\label{sec:NC_C}

In this subsection, we consider smooth nonconvex functions which satisfy the following assumptions.

\begin{assump}[Concavity]
\label{assum:concavity}
The function $f$ is concave in $\by$ if for a fixed $\bx \in \mbb R^{d_1}$, for all  $\by, \by' \in \mbb R^{d_2}$,
\iftwocol
\else
    \begin{align*}
        f(\bx, \by) \leq f(\bx, \by') + \lan \Gy f(\bx, \by'), \by - \by' \ran.
    \end{align*}
\fi
\end{assump}

\begin{assump}[Lipschitz continuity in $\bx$]
\label{assum:Lips_cont_x}
Given a  function $f$, there exists a constant $G_{\bx}$, such that for each $\by \in \mbb R^{d_2}$, and all $\bx, \bx' \in \mbb R^{d_1}$,
\iftwocol
\else
    \begin{align*}
        \norm{f(\bx, \by) - f(\bx', \by)} \leq G_{\bx} \norm{\bx - \bx'}.
    \end{align*}
\fi
\end{assump}

The envelope function $\Phi(\bx) = \max_\by f(\bx, \by)$ used so far, may no longer be smooth in the absence of a unique maximizer.
\iftwocol
\else
    Instead, we use the alternate definition of stationarity, proposed in \cite{davis19wc_siam}, utilizing the Moreau envelope of $\Phi$, which is defined next.
\fi

\begin{definition}[Moreau Envelope]
The function $\phi_{\lambda}$ is the $\lambda$-Moreau envelope of $\phi$, for $\lambda > 0$, if for all $\bx \in \mbb R^{d_x}$,
\iftwocol
\else
    \begin{align*}
        \phi_\lambda(\bx) = \min_{\bx'} \phi (\bx') + \frac{1}{2 \lambda} \norm{\bx' - \bx}^2.   
    \end{align*}
\fi
\end{definition}

\iftwocol
    
\else
    \cite{drusvyatskiy19wc_mathprog} showed that a small {\small$\norm{\G \phi_\lambda(\bx)}$} indicates the existence of some point $\Tbx$ in the vicinity of $\bx$, that is \textit{nearly stationary} for $\phi$. Hence, in our case, we focus on minimizing $\norm{\G \Phi_\lambda(\bx)}$.
\fi

\paragraph{Proposed Algorithm.}
For nonconvex-concave functions, we use \fedsgdaplus. The $\bx$-updates are identical to \fedsgda. For the $\by$ updates however, the clients compute stochastic gradients $\Gy f_i (\Hbxs, \byitk; \xiitk)$ keeping the $x$-component fixed at $\Hbxs$ for $S$ communication rounds. This \textit{trick}, originally proposed in \cite{mahdavi21localSGDA_aistats}, gives the analytical benefit of a double-loop algorithm (which update $\by$ several times before updating $\bx$ once) while also updating $\bx$ simultaneously.

\begin{theorem}
\label{thm:NC_C}
Suppose the local loss functions $\{ f_i \}$ satisfy Assumptions \ref{assum:smoothness}, \ref{assum:bdd_var}, \ref{assum:bdd_hetero}, \ref{assum:concavity}, \ref{assum:Lips_cont_x}, $\normb{\byt}^2 \leq R$ for all $t$, and the server selects $|\clientset| = \selclients$ clients for all $t$. With appropriate client and server learning rates, $(\lrcx, \lrcy)$ and $(\lrsx, \lrsy)$ respectively, the iterates of \fedsgdaplus \ satisfy
\iftwocol
\else
    \begin{align}
        & \min_{t \in [T]} 
        \mbe \normb{\G \TPhi_{1/2\Lf} (\bxt)}^2 \leq 
        \underbrace{\mco \Big( \mfrac{(\bar{\sync}/\seff)^{1/4}}{(\bar{\sync} \selclients T)^{1/4}} + \mfrac{(\seff \selclients)^{1/4}}{T^{3/4}} \Big)}_{\text{Error with full synchronization}} + \underbrace{\mco \lp \mfrac{\Cw \localvar^2 + \Dp (G_{\bx}^2 + \hetero^2)}{\bar{\sync}^2 T^{3/4}} \rp}_{\text{Local updates error}} + \underbrace{\mco \Big( \Big( \mfrac{\numclients - \selclients}{\numclients-1} \cdot \mfrac{\Ew}{\selclients T } \Big)^{1/4} \Big)}_{\text{Partial participation error}}, \label{eq:thm:NC_C}
    \end{align}
\fi
where $\Phi_{1/2\Lf}$ is the Moreau envelope of $\Phi$. The constants $\Cw, \Dw, \bar{\sync}$ are defined in \cref{thm:NC_SC}.
\end{theorem}

See \cref{app:NC_C} for the proof. \cref{thm:NC_C} states convergence for a surrogate objective $\TF$. Next, we see convergence for the true objective $F$.

\begin{cor}[Convergence in terms of $F$]
\label{cor:NC_C_obj_inconsistent}
Given envelope functions $\Phi(\bx) \triangleq \max_\by F (\bx, \by)$, $\TPhi(\bx) \triangleq \max_\by \TF (\bx, \by)$, under the conditions of \cref{thm:NC_C},
\iftwocol
\else
    \begin{align*}
        \min_{t\in[T]} \normb{\nabla \Phi_{1/2\Lf} (\bx^{(t)})}^2 \leq \epsilon'_{\text{opt}} + \frac{8 \Lf^2}{T} \textstyle \sum_{t=0}^{T-1} \normb{\widetilde{\bx}^{(t)} - \bar{\bx}^{(t)}}^2,
    \end{align*}
\fi
where $\Phi_{1/2\Lf}$ is the Moreau envelope of $\Phi$, $\widetilde{\bx}^{(t)} \triangleq \argmin_{\bx'} \{ \TPhi (\bx') + \Lf \normb{\bx' - \bxt}^2 \}$, $\bar{\bx}^{(t)} \triangleq \argmin_{\bx'} \{ \Phi (\bx') + \Lf \normb{\bx' - \bxt}^2 \}$, for all $t$, $\epsilon'_{\text{opt}}$ is the error bound in \eqref{eq:thm:NC_C}.
\end{cor}

Similar to \cref{cor:obj_inconsistent}, if we replace $\{\wi\}$ with $\{p_i\}$ for all $i \in [n]$ in the server updates in \cref{alg_NC_minimax}, then $\TF \equiv F$, and $\widetilde{\bx}^{(t)}$ and $\bar{\bx}^{(t)}$ are identical for all $t$. Consequently, \cref{thm:NC_C} gives us $\min_{t\in[T]} \normb{\nabla \Phi_{1/2\Lf} (\bx^{(t)})}^2 \leq \epsilon'_{\text{opt}}$.

\begin{remark}
\label{rem:NC_C_local_SGDAplus_1}
Some existing works do not require \cref{assum:Lips_cont_x} for NC-C functions, and also improve the convergence rate. However, these methods either have a double-loop structure \cite{rafique18WCC_oms, zhang2022sapd_NC_C_arxiv}, or work with deterministic problems \cite{lan_unified_ncc_arxiv20, zhang_1_loop_ncc_neurips20}. Proposing a single-loop method for stochastic NC-C problems with the same advantages is an open problem. 
\end{remark}

\begin{cor}[Improved Communication Savings]
\label{cor:NC_C_comm_cost}
Suppose all the clients are weighted equally ($p_i = 1/n$ for all $i$), with each carrying out $\sync$ steps of local SGDA.
To reach an $\epsilon$-stationary point, i.e., $\bx$ such that $\mbe \| \G \Phi_{1/2\Lf} (\bx) \| \leq \epsilon$, 
\begin{itemize}[leftmargin=*]
    \setlength\itemsep{-0.5em}
    \item Under full participation, the per-client gradient complexity of \fedsgdaplus \ is 
    \iftwocol
\else
        $T \sync = \mco \lp \frac{1}{\numclients \epsilon^8} \rp$.
    \fi
    The number of communication rounds required is
    \iftwocol
\else
        $T = \mco \lp \frac{1}{\epsilon^{4}} \rp$.
    \fi
    \item Under partial participation, in the special case when inter-client data heterogeneity $\hetero = 0$, the per-client gradient complexity of \fedsgda \ is $\mco (1/(\selclients \epsilon^8))$, while the communication cost is $\mco ( 1/\epsilon^{4} )$.
\end{itemize}
\end{cor}
In terms of communication requirements, we achieve massive savings (compared to $\mco ( 1/\epsilon^{7} )$ in \cite{sharma22FedMinimax_ICML}). Our gradient complexity results achieve linear speedup in the number of participating clients.

\subsection{Nonconvex-1-Point-Concave (NC-1PC) Case}
One-point-convexity has been observed in SGD dynamics during neural network training.
\begin{assump}[One-point-Concavity in $\by$]
\label{assum:1pc_y}
The function $f$ is said to be one-point-concave in $\by$ if fixing $\bx \in \mbb R^{d_1}$, for all  $\by \in \mbb R^{d_2}$,
\iftwocol
\else
    \begin{align*}
        \lan \Gy f(\bx, \by'), \by - \by^*(\bx) \ran \leq f(\bx, \by) - f(\bx, \by^*(\bx)),   
    \end{align*}
\fi
where $\by^*(\bx) \in \argmax_\by f(\bx, \by)$.
\end{assump}
Owing to space limitations, we only state the per-client gradient complexity, and communication complexity results under the special case when all the clients are weighted equally ($p_i = 1/n$ for all $i$), with each carrying out $\sync$ steps of local SGDA. See \cref{app:NC_1PC} for more details.

\begin{theorem}
\label{thm:NC_1PC}
Suppose the local loss functions $\{ f_i \}$ satisfy Assumptions \ref{assum:smoothness}, \ref{assum:bdd_var}, \ref{assum:bdd_hetero}, \ref{assum:Lips_cont_x}. Suppose for all $\bx$, all the $f_i$'s satisfy \cref{assum:1pc_y} at a common global minimizer $\by^* (\bx)$, and that $\normb{\byt}^2 \leq R$ for all $t$. Then, to reach an $\epsilon$-accurate point, the stochastic gradient complexity of \fedsgdaplus \ (\cref{alg_NC_minimax}) is $\mco (1/(\numclients \epsilon^8))$, and the number of communication rounds required is $T/\sync = \mco ( 1/\epsilon^{4} )$.
\end{theorem}

\begin{remark}
\cref{thm:NC_1PC} proves the conjecture posed in \cite{sharma22FedMinimax_ICML} that linear speedup should be achievable for NC-1PC functions. Further, we improve their communication complexity from $\mco (1/\epsilon^7)$ to $\mco (1/\epsilon^4)$. As an intermediate result in our proof, we show convergence of Local SGD for one-point-convex functions, extending convex minimization bounds to a much larger class of functions.
\end{remark}

%% file: sections/5_Simulations.tex
\section{Experiments}
\label{sec:exp}

In this section, we evaluate the empirical performance of the proposed algorithms. We consider a robust neural training problem \cite{sinha17certifiable_robust_iclr, madry18adversarial_iclr}, and a fair classification problem \cite{mohri19agnosticFL_icml, mahdavi20dist_robustfl_neurips}. Due to space constraints, additional details of our experiments, and some additional results are included in \cref{app:add_exp}.
Our experiments were run on a network of $n=15$ clients, each equipped with an NVIDIA TitanX GPU. We model data heterogeneity across clients using Dirichlet distribution \cite{wang19FL_iclr} with parameter $\alpha$, $\text{Dir}_{\numclients}(\alpha)$. Small $\alpha \Rightarrow$ higher heterogeneity across clients.

\paragraph{Robust NN training.}
We consider the following robust neural network (NN) training problem.
\iftwocol
\else
    \begin{align}
        \min_\bx \max_{\norm{\by}^2 \leq 1} \sum_{j=1}^N \ell \lp h_\bx (\mbf a_i + \by), b_i \rp, \label{eq:robustNN}
    \end{align}
\fi
where $\bx$ denotes the NN parameters, $(a_i, b_i)$ denote the feature and label of the $i$-th sample, $\by$ denotes the adversarially added feature perturbation, and $h_\bx$ denotes the NN output.
\iftwocol
\else
    \begin{figure}[h!]
        \centering
        \includegraphics[width=0.55\textwidth]{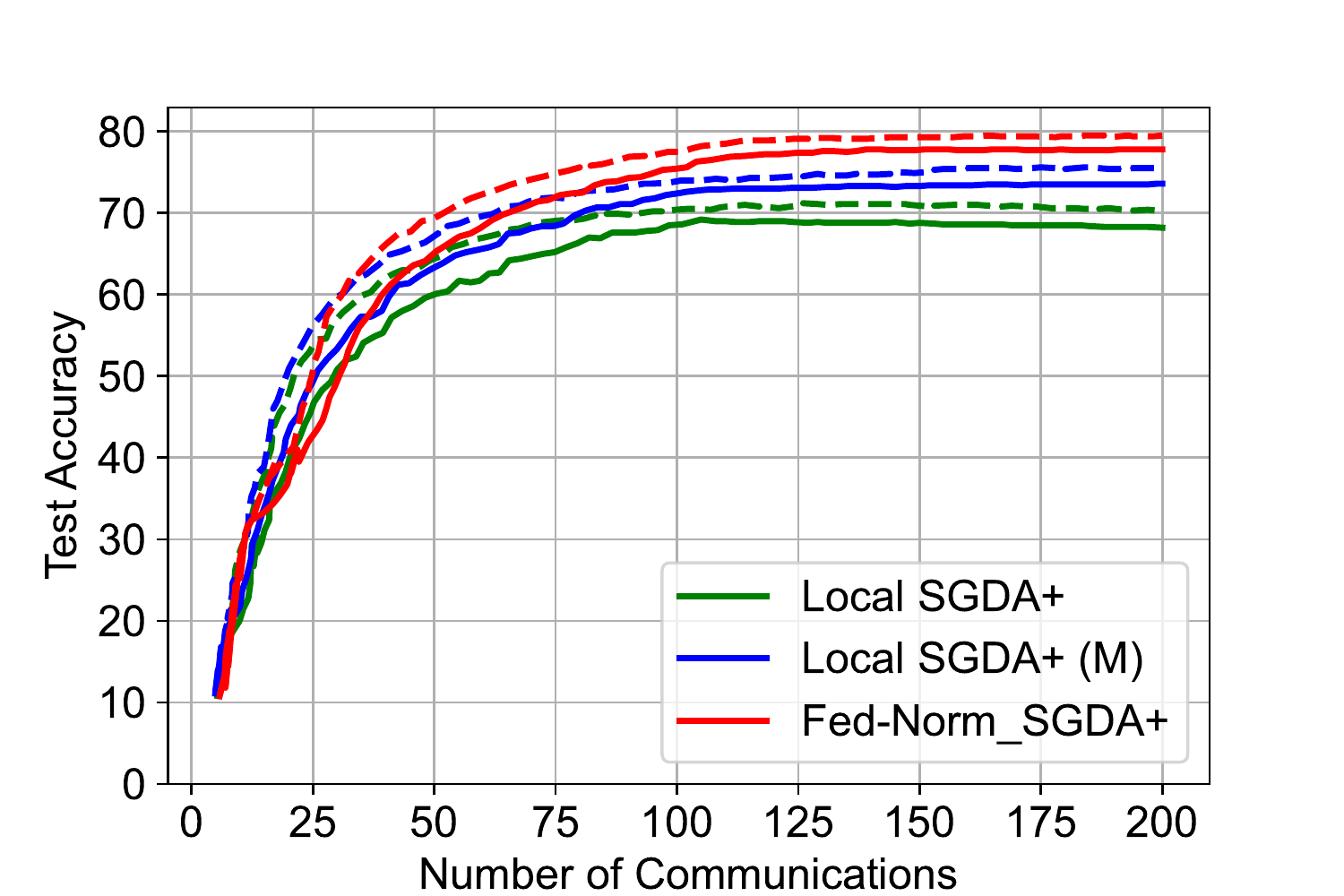}
        \caption{Comparison of the effect of heterogeneous number of local updates $\{ \sync_i \}$ on the performance of \fedsgdaplus \ (\cref{alg_NC_minimax}), Local SGDA+, and Local SGDA+ with momentum, while solving \eqref{eq:robustNN} on CIFAR10 dataset, with VGG11 model. The solid (dashed) curves are for $E = 5$ ($E=7$), and $\alpha = 0.1$. \label{fig:robustNN_hetero_epoch}}
    \end{figure}
\fi

\paragraph{Impact of system heterogeneity.}
In \cref{fig:robustNN_hetero_epoch}, we compare the effect of heterogeneous  number of local updates across clients, on the performance of our proposed \fedsgdaplus. We compare with Local SGDA+ \cite{mahdavi21localSGDA_aistats}, and Local SGDA+ with momentum \cite{sharma22FedMinimax_ICML}. Clients sample the number of epochs they run locally via $\sync_i \sim Unif[2,E]$. We observe that \fedsgdaplus \ adapts well to system heterogeneity and outperforms both existing methods.
\iftwocol
\else
    \begin{figure}[h!]
        \centering
        \includegraphics[width=0.55\textwidth]{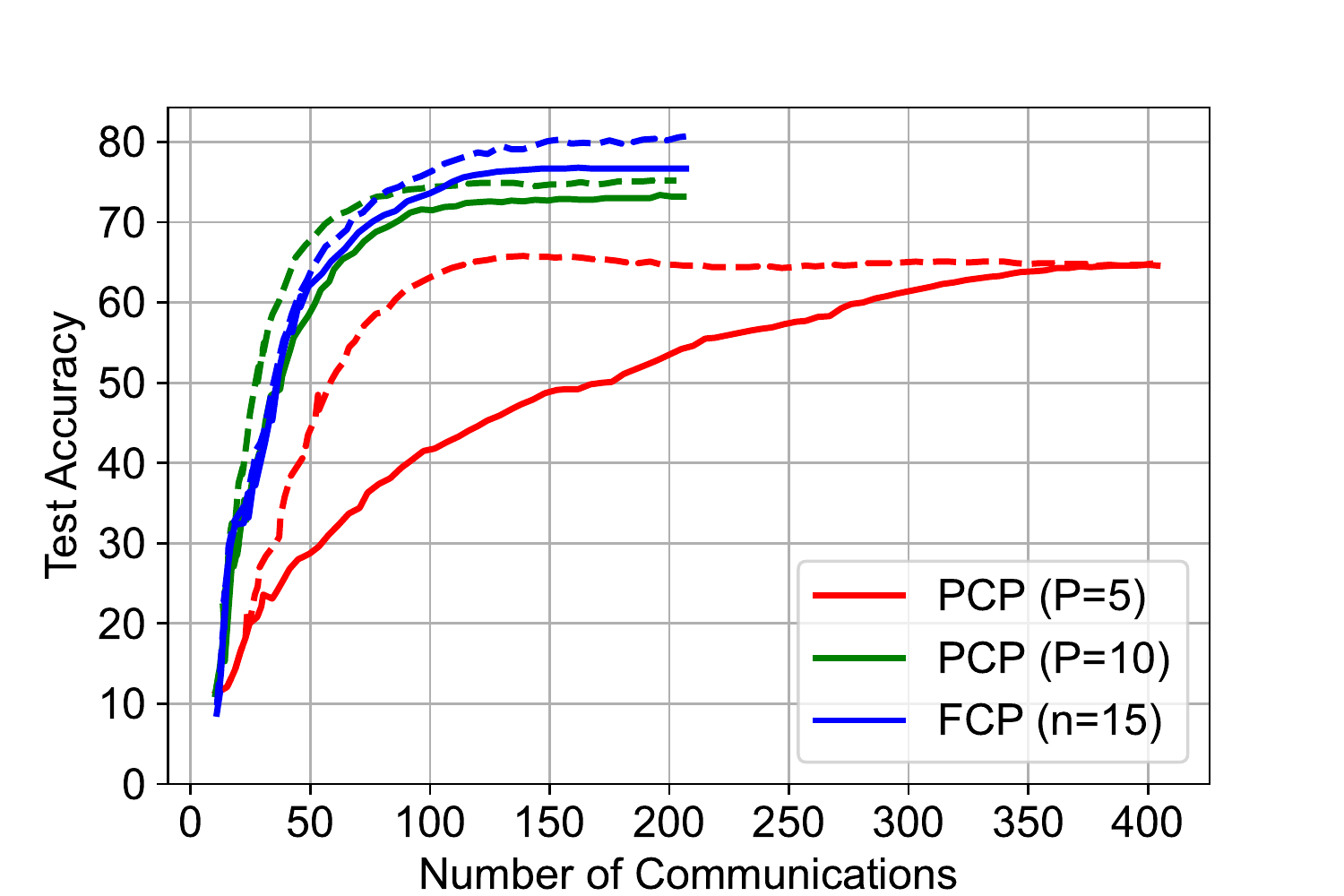}
        \caption{Comparison of the effects of partial client participation (PCP) on the performance of \fedsgdaplus, for the robust NN training problem on the CIFAR10 dataset, with the VGG11 model. The figure shows the robust test accuracy. The solid (dashed) curves are for $\alpha = 0.1$ ($\alpha=1.0$). \label{fig:robustnn_partial}}
    \end{figure}
\fi

\paragraph{Impact of partial participation and heterogeneity.} Next, we compare the impact of different levels of partial client participation on performance. We compare the full participation setting ($\numclients=15$) with $\selclients = 5, 10$. Clients sample the number of epochs they run locally via $\sync_i \sim Unif[2,5]$. We plot the results for two different values of the data heterogeneity parameter $\alpha = 0.1, 1.0$. As seen in all our theoretical results where partial participation was the most significant component of convergence error, smaller values of $\selclients$ result in performance loss. Further, higher inter-client heterogeneity (modeled by smaller values of $\alpha$) results in worse performance. We further explore the impact of $\alpha$ on performance in \cref{app:add_exp}.

\iftwocol
\else
    \begin{figure}[h!]
        \centering
        \includegraphics[width=0.55\textwidth]{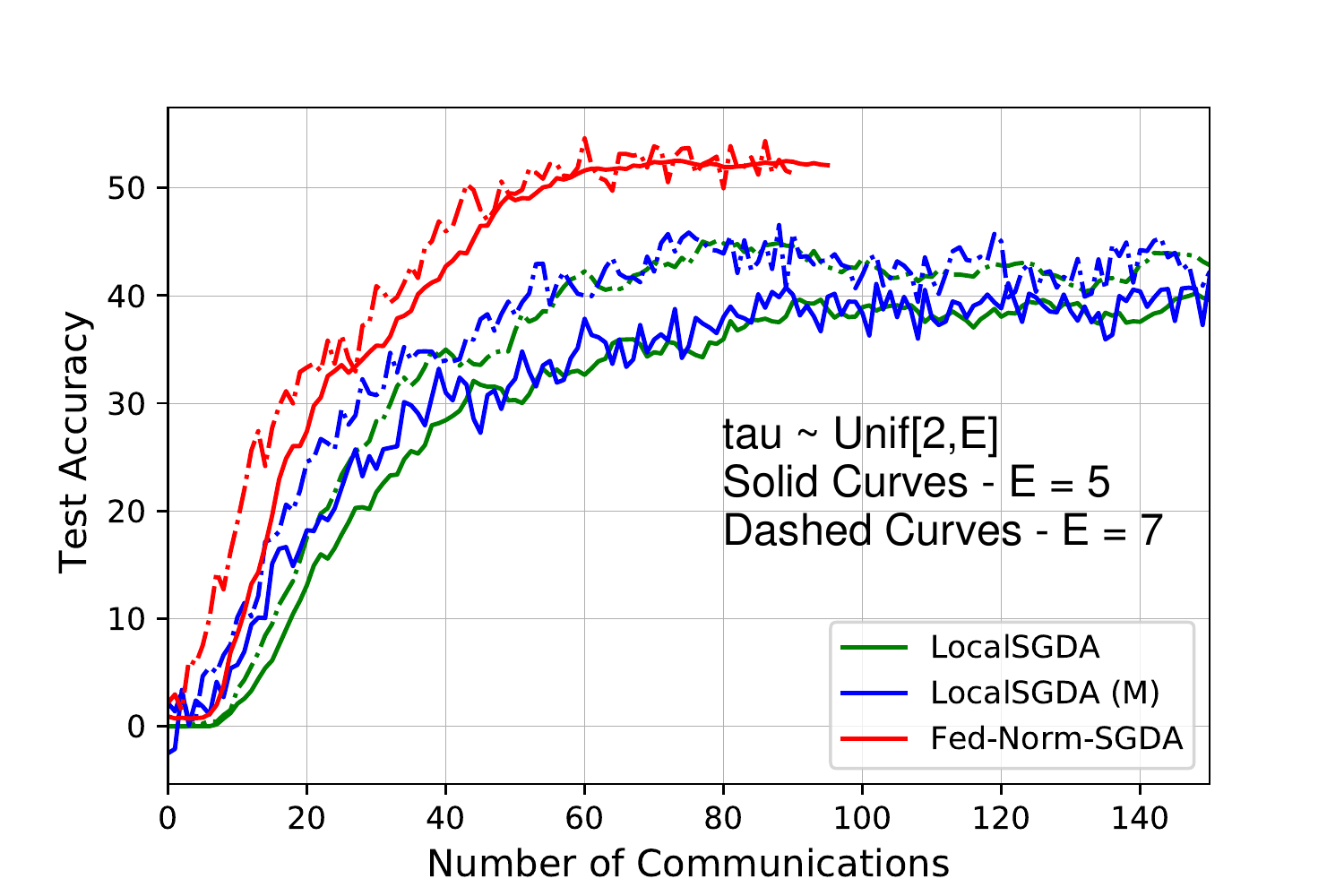}
        \caption{Comparison of Local SGDA, Local SGDA with momentum, and \fedsgda, for the fair classification task on the CIFAR10 dataset, with the VGG11 model. The solid (dashed) curves are for $E = 5$ ($E=7$), $\alpha = 0.1$. \label{fig:fairclass}}
    \end{figure}
\fi

\paragraph{Fair Classification.}
We consider minimax formulation of the fair classification problem \cite{mohri19agnosticFL_icml, nouiehed19minimax_neurips19}.
\iftwocol
\else
    \begin{align}
        \min_\bx \max_{\by \in \mc Y} \sum_{c=1}^C y_c F_c(\bx) -\frac{\lambda}{2} \norm{\by}^2,
        \label{eq:exp_fair_2}
    \end{align}
\fi
where $\bx$ denotes the parameters of the NN, $\{F_c\}_{c=1}^C$ denote the loss corresponding to class $c$, and {\small$\Delta_C$} is the $C$-dimensional probability simplex. In \cref{fig:fairclass}, we plot the worst distribution test accuracy achieved by \fedsgda, Local SGDA \cite{mahdavi21localSGDA_aistats} and Local SGDA with momentum \cite{sharma22FedMinimax_ICML}. As in \cref{fig:robustNN_hetero_epoch}, clients sample $\sync_i \sim Unif[2,E]$. We plot the test accuracy on the worst distribution in each case. Again, \fedsgda \ outperforms existing methods.

%% file: sections/6_Conclusion.tex
\section{Conclusion}
\label{sec_conclude}
In this work, we considered nonconvex minimax problems in the federated setting, where in addition to inter-client data heterogeneity and partial client participation, there is system heterogeneity as well. We observed that existing methods, such as Local SGDA, might converge to the stationary point of an objective quite different from the original intended objective. We show that normalizing individual client contributions solves this problem. We analyze several classes of nonconvex minimax functions, and significantly improve existing computation and communication complexity results. Potential future directions include analyzing federated systems with unpredictable client presence \cite{yang22anarchic_icml}.

\nocite{lin20near_opt_det_colt, nesterov18book, yoon21acc_ncc_icml, ouyang21lower_cc_bilinear_mathprog, wang20improved_cc_neurips, li21lower_bd_NCSC_neurips, lei21stability_Minimax_icml, zhang21NCSC_uai, kiyavash20catalyst_neurips, lee21NCNC_structured_neurips, lu20HiBSA_NC_C_ieee, tran20hybrid_NCLin_neurips, jin20local_opt_NCNC_icml, liang20proxGDA_KL_iclr, luo21near_opt_FS_cc_arxiv, xie20lower_FS_cc_icml, gasnikov21decen_deter_cc_icoa, ozdaglar19dec_prox_sp_arxiv, richtarik21dist_VI_comm_arxiv, gasnikov21dec_person_FL_arxiv, jhunjhunwala22fedvarp_uai}

%% file: Supplementary.tex
\newpage

\begin{center}
    {\LARGE \textbf{Appendix}}
\end{center}

The appendix are organized as follows. 
In Section \ref{app:bkgrnd} we mention some basic mathematical results and inequalities which are used throughout the paper.
In \cref{app:NC_SC} we prove the non-asymptotic convergence of \fedsgda \ (\cref{alg_NC_minimax}) for smooth nonconvex-strongly-concave (and nonconvex-P{\L}) functions, and derive gradient complexity and communication cost of the algorithm to achieve an $\epsilon$-stationary point.
In \cref{app:NC_C}, we prove the non-asymptotic convergence of \fedsgdaplus \ (\cref{alg_NC_minimax}) for smooth nonconvex-concave and nonconvex-one-point-concave functions.
Finally, in \cref{app:add_exp} we provide the details of the additional experiments we performed.
\input{Appendix/1_Basic}
\input{Appendix/3_NC_SC_PCP}
\input{Appendix/5_NC_C_PCP}
\input{Appendix/6_Experiments}

%% file: Appendix/1_Basic.tex
\section{Background}
\label{app:bkgrnd}

\subsection{Gradient Aggregation with Different Solvers at Clients}
\label{app:grad_agg}

\paragraph{Local SGDA.}
Suppose $\syncit = \sefft = \sync$ for all $i \in [n]$, $t \in [T]$. Also, $\aitk = 1$ for all $k \in [\sync], t$. Then, the local iterate updates in \cref{alg_NC_minimax}-\fedsgda \ reduce to (the updates for \fedsgdaplus \ are analogous) 
\begin{align*}
    \bxitkp &= \bxitk - \lrcx \Gx f_i (\bxitk, \byitk; \xiitk), \\
    \byitkp &= \byitk + \lrcy \Gy f_i (\bxitk, \byitk; \xiitk),
\end{align*}
for $k \in \{0, \dots, \sync-1 \}$
and the gradient aggregate vectors $(\bdxit, \bdyit)$ are simply the average of individual gradients
\begin{align*}
    \bdxit = \frac{1}{\sync} \sum_{k=0}^{\sync-1} \Gx f_i (\bxitk, \byitk; \xiitk), \quad \bdyit = \frac{1}{\sync} \sum_{k=0}^{\sync-1} \Gy f_i (\bxitk, \byitk; \xiitk)
\end{align*}
Note that these are precisely the iterates of LocalSGDA proposed in \cite{mahdavi21localSGDA_aistats, sharma22FedMinimax_ICML}, with the only difference that in LocalSGDA, the clients communicate the iterates $\{ \bx_i^{(t,\sync)}, \by_i^{(t,\sync)} \}$ to the server, which averages them to compute $\{ \bx^{(t+1)}, \by^{(t+1)} \}$. While here, the clients communicate $\{ \bdxit, \bdyit \}$. Also, in \fedsgda, the clients and server use separate learning rates, which results in tighter bounds on the local-updates error.

\paragraph{With Momentum in Local Updates.}
Suppose each local client uses a momentum buffer with momentum scale $\rho$. Then, for $k \in \{0, \dots, \syncit-1 \}$
\begin{align*}
    \begin{matrix}
    \mbf d_{\bx,i}^{t,k+1} = \rho \mbf d_{\bx,i}^{t,k} + \Gx f_i (\bxitk, \byitk; \xiitk), \qquad & \bxitkp = \bxitk - \lrcx \mbf d_{\bx,i}^{t,k+1} \\ 
    \mbf d_{\by,i}^{t,k+1} = \rho \mbf d_{\by,i}^{t,k} + \Gy f_i (\bxitk, \byitk; \xiitk), \qquad & \byitkp = \byitk + \lrcy \mbf d_{\by,i}^{t,k+1},
    \end{matrix}
\end{align*}
Simple calculations show that the coefficient of $\Gx f_i (\bxitk, \byitk; \xiitk)$ and $\Gy f_i (\bxitk, \byitk; \xiitk)$ in the gradient aggregate vectors $(\bdxit, \bdyit)$ is
\begin{align*}
    \sum_{j \geq k}^{\syncit - 1} = 1 + \rho + \dots + \rho^{\syncit - 1 - k} = \frac{1 - \rho^{\syncit-k}}{1-\rho}.
\end{align*}
Therefore, the aggregation vector is $\bait = \frac{1}{1-\rho} [1-\rho^{\syncit}, 1-\rho^{\syncit-1}, \dots, 1-\rho]$, and
\begin{align*}
    \nait_1 = \sum_{k=0}^{\syncit - 1} \frac{1 - \rho^{\syncit-k}}{1-\rho} = \frac{1}{1-\rho} \lb \syncit - \frac{\rho(1-\rho^{\syncit})}{1-\rho} \rb.
\end{align*}

\subsection{Auxiliary Results}
\label{sec:aux}

\begin{remark}[Impact of heterogeneity $\hetero$ even with $\sync=1$]
\label{rem:impact_varscale}
Consider two simple minimization problems:
\begin{align*}
    \textbf{(P1):} \quad \min_\bx \frac{1}{\numclients} \sumin f_i(\bx) \qquad \text{ and } \qquad \textbf{(P2):} \quad \min_\bx f(x).
\end{align*}
\textbf{(P1)} is a simple distributed minimization problem, with $\numclients$ clients, which we solve using synchronous distributed SGD.
At iteration $t$, each client $i$ computes stochastic gradient $\G f_i(\bxt; \xiit)$, and sends it to the server, which averages these, and takes a step in the direction $\frac{1}{\numclients} \sumin \G f_i(\bxt; \xiit)$.
On the other hand, \textbf{(P1)} is a centralized minimization problem, where at each iteration $t$, the agent computes a stochastic gradient estimator with batch-size $n$, $\frac{1}{\numclients} \sumin \G f (\bxt; \xiit)$.
We compare the variance of the two global gradient estimators as follows.
\begin{equation}
    \begin{aligned}
        & \qquad \qquad \qquad \quad \textbf{(P1)} \nn \\
        & \mbe \Big\| \mfrac{1}{\numclients} \sumintext \G f_i (\bxt; \xiit) - \G f(\bxt) \Big\|^2 \\
        & \leq \mfrac{1}{\numclients^2} \sumintext \Big[ \localvar^2 + \varscale^2\mbe \big\| \G f_i(\bxt) \big\|^2 \Big] \\
        & \leq \mfrac{\localvar^2}{\numclients} + \mfrac{\varscale^2}{\numclients} \Big[ \beta_H^2 \mbe \big\| \G f(\bxt) \big\|^2 + \hetero^2 \Big].
    \end{aligned}
    \qquad \text{ vs } \qquad
    \begin{aligned}
        & \qquad \qquad \qquad \quad \textbf{(P2)} \nn \\
        & \mbe \Big\| \mfrac{1}{\numclients} \sumintext \G f (\bxt; \xiit) - \G f(\bxt) \Big\|^2 \\
        &= \mfrac{1}{\numclients^2} \sumintext \mbe \Big\| \G f (\bxt; \xiit) - \G f(\bxt) \Big\|^2 \\
        & \leq \mfrac{\localvar^2}{\numclients} + \mfrac{\varscale^2}{\numclients} \mbe \big\| \G f(\bxt) \big\|^2. 
    \end{aligned}
\end{equation}
Since almost all the existing works consider the local variance bound (\cref{assum:bdd_var}) with $\varscale = 0$, the global gradient estimator in both synchronous distributed SGD \textbf{(P1)} and single-agent minibatch SGD \textbf{(P2)} have the same $\frac{\localvar^2}{\numclients}$ variance bound.
Therefore, in most existing federated works on minimization \cite{joshi20fednova_neurips, yang21partial_client_iclr} and minimax problems \cite{sharma22FedMinimax_ICML}, the full synchronization error only depends on the local variance $\localvar^2$.
However, as seen above, for $\varscale > 0$, this \textit{apparent equivalence} breaks down.
\cite{koloskova20unified_localSGD_icml}, which considers similar local variance assumption as ours for minimization problems, also show similar dependence on heterogeneity $\hetero$.
\end{remark}

\subsection{Useful Inequalities}

\begin{lemma}[Young's inequality]
\label{lem:Young}
Given two same-dimensional vectors $\mbf u, \mbf v \in \mbb R^d$, the Euclidean inner product can be bounded as follows:
$$\lan \mbf u, \mbf v \ran \leq \frac{\norm{\mbf u}^2}{2 \gamma} + \frac{\gamma \norm{\mbf v}^2}{2}$$
for every constant $\gamma > 0$.
\end{lemma}

\begin{lemma}[Strong Concavity]
A function $g: \mc X \times \mc Y$ is strongly concave in $\by$, if there exists a constant $\mu > 0$, such that for all $\bx \in \mc X$, and for all $\by, \by' \in \mc Y$, the following inequality holds.
$$g(\bx, \by) \leq g(\bx, \by') + \lan \Gy g(\bx, \by'), \by' - \by \ran - \frac{\mu}{2} \norm{\by - \by'}^2.$$
\end{lemma}

\begin{lemma}[Jensen's inequality]
\label{lem:jensens}
Given a convex function $f$ and a random variable $X$, the following holds.
$$f \lp \mbe [X] \rp \leq \mbe \lb f(X) \rb.$$
\end{lemma}

\begin{lemma}[Sum of squares]
\label{lem:sum_of_squares}
For a positive integer $K$, and a set of vectors $x_1, \hdots, x_K$, the following holds:
\begin{align*}
    \norm{\sum_{k=1}^K x_k}^2 \leq K \sum_{k=1}^K \norm{x_k}^2.
\end{align*}
\end{lemma}

\begin{lemma}[Quadratic growth condition \cite{schmidt16lin_conv_PL_kdd}]
\label{lem:quad_growth}
If function $g$ satisfies Assumptions \ref{assum:smoothness}, \ref{assum:SC_y}, then for all $x$, the following conditions holds
\begin{align*}
    g(x) - \min_{z} g(z) & \geq \frac{\mu}{2} \norm{x_p - x}^2, \\
    \norm{\G g(x)}^2 & \geq 2 \mu \lp g(x) - \min_z g(z) \rp.
\end{align*}
\end{lemma}

\begin{lemma}
\label{lem:smooth_convex}
For $L$-smooth, convex function $g$, the following inequality holds
\begin{align}
    \mbe \lnr \G g (\by) - \G g (\bx) \rnr^2 \leq 2 L \lb g (\by) - g (\bx) - \G g(\bx)^\top (\by - \bx) \rb.
\end{align}
\end{lemma}

\begin{lemma}[Proposition 6 in \cite{cho2022sgda_shuffle}]
\label{lem:smooth}
For $L$-smooth function $g$ which is bounded below by $g^*$, the following inequality holds for all $\bx$
\begin{align}
    \mbe \lnr \G g (\bx) \rnr^2 \leq 2 L \lb g (\bx) - g^* \rb.
\end{align}
\end{lemma}

\subsection{Comparison of Convergence Rates}
\begin{table}[h]
\begin{center}
\begin{threeparttable}
\caption{Comparison of the convergence rates for different classes of nonconvex minimax problems. $n$ is the total number of clients, while $\selclients$ is the number of clients sampled in each round under partial client participation. $T$ is the number of communication rounds, $\tau$ is the number of local updates per client, $\sigma_L^2, \hetero^2$ are the stochastic gradient variance and global heterogeneity, respectively (\cref{assum:bdd_var}, \ref{assum:bdd_hetero}). $\kappa = \Lf/\mu$ is the condition number. For a fair comparison with existing works, our results in this table are specialized to the case when all clients (i) have equal weights ($p_i=1/\numclients$), (ii) perform equal number of local updates ($\sync_i = \sync$), and (iii) use the same local update algorithm SGDA. However, our results (\cref{sec:algo_theory}) apply under more general settings when (i)-(iii) do not hold. 
}
\label{table:compare_rates}
\vskip 0.15in
\begin{small}
\begin{tabular}{|c|c|c|c|}
\hline
Work & \makecell{Partial \\ Participation} & \makecell{System \\ Heterogeneity} & Convergence Rate \\
\hline
\multicolumn{4}{|c|}{Nonconvex-Strongly-concave (NC-SC)/Nonconvex-Polyak-{\L}ojasiewicz (NC-PL)} \\
\hline
\cite{sharma22FedMinimax_ICML} & \xmark & \xmark & $\mco \lp \mfrac{1}{\sqrt{\numclients \sync T}} + \mfrac{\numclients \sync (\localvar^2 + \hetero^2)}{T} \rp$ \\
\cite{yang22sagda_neurips} & \checkmark & \xmark & $\lp \mfrac{(\numclients - \selclients) (\hetero^2)}{\numclients} \sqrt{\mfrac{\sync}{\selclients T}} + \mfrac{1}{\sqrt{\selclients \sync T}} + \mfrac{(\localvar^2 + \sync \hetero^2)}{\sync T} \rp$ \\
\rowcolor{Gainsboro!60} \textbf{Our Work:} \cref{thm:NC_SC} & \checkmark & \checkmark & $\lp \mfrac{(\numclients - \selclients) \hetero^2}{(\numclients-1)} \sqrt{\mfrac{\sync}{\selclients T}} + \mfrac{1}{\sqrt{\selclients \sync T}} + \mfrac{\localvar^2 + \sync \hetero^2}{\sync T} \rp$ \\
\hline
\multicolumn{4}{|c|}{Nonconvex-Concave (NC-C)} \\
\hline
\cite{sharma22FedMinimax_ICML} & \xmark & \xmark & $\mco \lp \mfrac{1}{(\sync \numclients T)^{1/4}} \rp + \mco \lp \mfrac{(n \sync)^{3/2}}{\sqrt{T}} \rp$ \\
\rowcolor{Gainsboro!60} \textbf{Our Work:} \cref{thm:NC_C} & \checkmark & \checkmark & $\mco \lp \sqrt[4]{\mfrac{\numclients - \selclients}{(\numclients-1) \selclients T}} + \mfrac{1}{(\sync \selclients T)^{1/4}} + \mfrac{(\selclients \sync)^{1/4}}{T^{3/4}} \rp$ \\
\hline
\multicolumn{4}{|c|}{Nonconvex-One-point-concave (NC-1PC)} \\
\hline
\cite{sharma22FedMinimax_ICML} & \xmark & \xmark & $\mco \lp \mfrac{1}{(\sync T)^{1/4}} \rp + \mco \lp \mfrac{\sync^{3/2}}{\sqrt{T}} \rp$ \\
\rowcolor{Gainsboro!60} \textbf{Our Work:} \cref{thm:NC_1PC} & \checkmark & \checkmark & $\mco \lp \sqrt[4]{\mfrac{\numclients - \selclients}{(\numclients-1) \selclients T}} + \mfrac{1}{(\sync \selclients T)^{1/4}} + \mfrac{(\selclients \sync)^{1/4}}{T^{3/4}} \rp$ \\
\hline
\end{tabular}
\end{small}
\vskip -0.1in
\end{threeparttable}
\end{center}
\end{table}

%% file: Appendix/3_NC_SC_PCP.tex
\section{Convergence of \fedsgda \ for Nonconvex-Strongly-Concave Functions (\texorpdfstring{\cref{thm:NC_SC}}{Theorem 1})}
\label{app:NC_SC}
We organize this section as follows. First, in \cref{app:NC_SC_PCP_int_results} we present some intermediate results, which we use to prove the main theorem. Next, in \cref{app:NC_SC_PCP_thm_proof}, we present the proof of \cref{thm:NC_SC}, which is followed by the proofs of the intermediate results in \cref{app:NC_SC_PCP_int_results_proofs}. \cref{app:NC_SC_PCP_aux_lemma} contains some auxiliary results. Finally, in \cref{app:NC_PL_PCP_result} we discuss the convergence result for nonconvex-PL functions.


The problem we solve is
\begin{align*}
    \min_{\bx} \max_{\by} \lcb \TF(\bx, \by) \triangleq \sumin \wi f_i(\bx, \by) \rcb.
\end{align*}
We define $\TPhi (\bx) \triangleq \max_{\by} \TF(\bx, \by)$ and $\Tby^* (\bx) \in \argmax_{\by} \TF(\bx, \by)$. Since $\TF(\bx, \cdot)$ is $\mu$-strongly concave, $\Tby^*(\bx)$ is unique.
In \fedsgda \ (\cref{alg_NC_minimax}), the client updates are given by
\begin{equation}
    \begin{aligned}
        & \bxitk = \bxt - \lrcx \sumijk \aijk \Gx f_i (\bxitj, \byitj; \xiitj), \\
        & \byitk = \byt + \lrcy \sumijk \aijk \Gy f_i (\bxitj, \byitj; \xiitj),
    \end{aligned}
    \label{eq:client_update_alg_NC_SC}
\end{equation}
where $1 \leq k \leq \sync_i$. These client updates are then aggregated to compute $\{ \bdxit, \bdyit \}$
\begin{equation}
    \begin{aligned}
        & \bdxit = \frac{1}{\nai} \sumikt \aikt \Gx f_i \lp \bxitk, \byitk; \xiitk \rp; \quad \bhxit = \frac{1}{\nai} \sumikt \aikt \Gx f_i \lp \bxitk, \byitk \rp \\
        & \bdyit = \frac{1}{\nai} \sumikt \aikt \Gy f_i \lp \bxitk, \byitk; \xiitk \rp; \quad \bhyit = \frac{1}{\nai} \sumikt \aikt \Gy f_i \lp \bxitk, \byitk \rp. \nn
    \end{aligned}
\end{equation}

\begin{remark}
Note that we have made explicit, the dependence on $k$ in $\aijk$ above. This was omitted in the main paper to avoid tedious notation. However, for some local optimizers, such as local momentum at the clients (\cref{app:grad_agg}), the coefficients $\aijk$ change with $k$. We assume in our subsequent analysis that $\aijk \leq \alpha$ for all $j \in \{ 0,1,\dots, k-1\}$ and for all $k \in \{ 1,2, \dots, \sync_i \}$. Further, we assume that $\normb{\boldsymbol a_i (k)}_1 \leq \normb{\boldsymbol a_i (k)}_1$ and $\normb{\boldsymbol a_i (k+1)}_2 \leq \normb{\boldsymbol a_i (k+1)}_2$ for feasible $k$. We also use the notation $\nai \triangleq \normb{\boldsymbol a_i (\sync_i)}$.
\end{remark}

At iteration $t$, the server samples  $|\clientset|$ clients \textit{without} replacement \textbf{(WOR)} uniformly at random. While aggregating at the server, client $i$ update is weighed by $\twi = w_i \numclients/|\clientset|$. The aggregates $(\bdxt, \bdyt)$ computed at the server are of the form
\begin{equation}
    \begin{aligned}
        \bdxt &= \sumiS \twi \bdxit, \quad \text{such that} \quad \mbe_{\clientset} [ \bdxt ] = \mbe_{\clientset} \Big[ \sumin \mathbb{I} (i \in \clientset) \twi \bdxit \Big] = \sumin w_i \bdxit \\
        \bdyt &= \sumiS \twi \bdyit, \quad \text{such that} \quad \mbe_{\clientset} [ \bdyt ] = \mbe_{\clientset} \Big[ \sumin \mathbb{I} (i \in \clientset) \twi \bdyit \Big] = \sumin w_i \bdyit
    \end{aligned}
\end{equation}
For simplicity of analysis, unless stated otherwise, we assume that $|\clientset| = P$ for all $t$.
Finally, server updates the $\bx, \by$ variables as
\begin{equation}
    \begin{aligned}
        \bxtp = \bxt - \seff \lrsx \bdxt, \qquad \bytp = \byt + \seff \lrsy \bdyt. \nn
    \end{aligned}
\end{equation}

\subsection{Intermediate Lemmas} \label{app:NC_SC_PCP_int_results}

We begin with the following result from \cite{nouiehed19minimax_neurips19} about the smoothness of $\TPhi(\cdot)$.

\begin{lemma}
\label{lem:Phi_smooth_nouiehed}
If a function $f(\cdot, \cdot)$ satisfies Assumptions \ref{assum:smoothness}, \ref{assum:SC_y} ($\Lf$-smoothness and $\mu$-strong concavity in $\by$), then $\phi (\cdot) \triangleq \max_\by f(\cdot, \by)$ is $\Lp$-smooth with $\Lp = \kappa \Lf/2 + \Lf$, where $\kappa = \Lf/\mu$ is the condition number.
\end{lemma}

\begin{lemma}
\label{lem:NC_SC_PCP_wtd_dir_norm_sq}
Suppose the local client loss functions $\{ f_i \}$ satisfy Assumptions \ref{assum:smoothness}, \ref{assum:bdd_hetero},
and the stochastic oracles for the local functions satisfy \cref{assum:bdd_var}.
Suppose the server selects $\selclients$ clients in each round. Then the iterates generated by \fedsgda \ (\cref{alg_NC_minimax}) satisfy
\begin{equation}
    \begin{aligned}
        & \mbe \norm{\bdxt}^2 = \mbe \norm{\sumiS \twi \bdxit}^2 \\
        & \leq \frac{\numclients}{\selclients} \lp \frac{\selclients-1}{\numclients-1} \rp \mbe \lnr \sumin w_i \bhxit \rnr^2 + \frac{\numclients}{\selclients} \sumin \frac{w_i^2}{\nai_1^2} \sumikt [ \aikt ]^2 \lb \localvar^2 + \varscale^2 \mbe \lnr \Gx f_i ( \bxitk, \byitk ) \rnr^2 \rb \\
        & \quad + \frac{\numclients (\numclients - \selclients)}{\numclients-1} \lb \frac{2 \Lf^2}{\selclients} \sumin \frac{w_i^2}{\nai_1} \sumikt \aikt \CExytk (i) + (\max_i w_i) \frac{2}{\selclients} \lp \heteroscale^2 \norm{\Gx \TF (\bxt, \byt)}^2 + \hetero^2 \rp \rb,
    \end{aligned}
    \label{eq:lem:NC_SC_PCP_WOR_wtd_dir_norm_sq}
\end{equation}
where, $\CExytk (i) \triangleq \mbe \lb \| \bxitk - \bxt \|^2 + \| \byitk - \byt \|^2 \rb$ is the iterate drift for client $i$, at local iteration $k$ in the $t$-th communication round.
\end{lemma}


\begin{lemma}
\label{lem:NC_SC_PCP_Phi_decay_one_iter}
Suppose the local client loss functions $\{ f_i \}$ satisfy Assumptions \ref{assum:smoothness}, \ref{assum:bdd_hetero}, \ref{assum:SC_y},
and the stochastic oracles for the local functions satisfy \cref{assum:bdd_var}. 
Also, the server learning rate $\lrsx$ satisfies $64 \seff \lrsx \Lp \varscale^2 \heteroscale^2 \frac{\numclients}{P} ( \max_i w_i \nai_2^2/\nai_1^2 ) \leq 1$, $8 \seff \lrsx \Lp (\max_i w_i) \frac{\numclients}{\selclients} \lp \frac{\numclients - \selclients}{\numclients-1} \rp \max \{ 8 \heteroscale^2, 1 \} \leq 1$, and $8 \seff \lrsx \Lp \varscale^2 \frac{\numclients}{\selclients} ( \max_{i,k} w_i \aikt/\nai_1 ) \leq 1$. Then the iterates generated by \cref{alg_NC_minimax} satisfy
\begin{equation}
    \begin{aligned}
        & \mbe \lb \TPhi(\bxtp) - \TPhi(\bxt) \rb \leq - \frac{7 \seff \lrsx}{16} \mbe \norm{\G \TPhi(\bxt)}^2 - \frac{\seff \lrsx}{2} \lp 1 - \frac{\numclients (\selclients-1)}{\selclients (\numclients-1)} \seff \lrsx \Lp \rp \mbe \norm{\sumin w_i \bhxit}^2 \\
        & \quad + \frac{5}{4} \seff \lrsx \Lf^2 \sumin \frac{w_i}{\nai_1} \sumikt \aikt \CExytk (i) + \frac{9 \seff \lrsx \Lf^2}{4 \mu} \mbe \lb \TPhi (\bxt) - \TF (\bxt, \byt) \rb \\
        & \quad + \frac{\seff^2 \lrsxsq \Lp}{2} \frac{\numclients}{\selclients} \lb \localvar^2 \sumin \frac{w_i^2 \nai_2^2}{\nai_1^2} + \hetero^2 \lp 2 (\max_i w_i) \frac{\numclients - \selclients}{\numclients-1} + 2 \varscale^2 \max_i \frac{w_i \nai_2^2}{\nai_1^2} \rp \rb.
    \end{aligned}
    \label{eq:lem:NC_SC_PCP_WOR_Phi_decay_one_iter}
\end{equation}
\end{lemma}

\begin{rem}
The bound in \cref{eq:lem:NC_SC_PCP_WOR_Phi_decay_one_iter} looks very similar to the corresponding one-step decay bound for simple smooth minimization problems. The major difference is the presence of $\mbe \lb \TPhi (\bxt) - \TF (\bxt, \byt) \rb$, which quantifies the inaccuracy of $\byt$ in solving the \textit{max} problem $\max_\by \TF(\bxt, \by)$.
The term $\sumin \frac{w_i}{\nai_1} \sumikt \aikt \CExytk (i)$ is the client drift and is bounded in \cref{lem:NC_SC_PCP_consensus_error}.
\end{rem}


\begin{lemma}
\label{lem:NC_SC_PCP_consensus_error}
Suppose the local loss functions $\{ f_i \}$ satisfy Assumptions \ref{assum:smoothness}, \ref{assum:bdd_hetero}, \ref{assum:SC_y},
and the stochastic oracles for the local functions satisfy \cref{assum:bdd_var}.
Further, in \cref{alg_NC_minimax}, we choose learning rates $\lrcx, \lrcy$ such that $\max \{ \lrcx, \lrcy \} \leq \frac{1}{2 \Lf ( \max_i \nai_1 ) \sqrt{2 (1 + \varscale^2)}}$.
Then, the iterates $\{ \bxit, \byit \}$ generated by \fedsgda \ (\cref{alg_NC_minimax}) satisfy
\begin{equation}
    \begin{aligned}
        & \Lf^2 \sumin \frac{w_i}{\nai_1} \sumikt \aikt \CExytk (i) \leq 2 \lp \lrcxsq + \lrcysq \rp \Lf^2 \localvar^2 \sumin w_i \norm{\boldsymbol a_{i,-1}}_2^2 + 4 \Lf^2 \TMa \lp \lrcxsq + \lrcysq \rp \hetero^2 \nn \\
        & \qquad + 8 \Lf^2 \TMa \heteroscale^2 \lrcxsq \mbe \lnr \G \TPhi (\bxt) \rnr^2 + 8 \Lf^3 \TMa \heteroscale^2 \lp 2 \kappa \lrcxsq + \lrcysq \rp \mbe \lb \TPhi (\bxt) - \TF (\bxt, \byt) \rb,
    \end{aligned}
\end{equation}
where $\TMa \triangleq \max_i \lp \norm{\boldsymbol a_{i,-1}}_1^2 + \varscale^2 \norm{\boldsymbol a_{i,-1}}_2^2 \rp$.
\end{lemma}

\begin{lemma}
\label{lem:NC_SC_PCP_phi_f_diff}
Suppose the local loss functions $\{ f_i \}$ satisfy Assumptions \ref{assum:smoothness}, \ref{assum:bdd_hetero}, \ref{assum:SC_y}, and the stochastic oracles for the local functions satisfy \cref{assum:bdd_var}.
The server learning rates $\lrsx, \lrsy$ satisfy the following conditions:
{\small
$$2 \seff \lrsy \Lf \leq 1, \seff \lrsy \kappa \Lf \heteroscale^2 \frac{\numclients}{\selclients} \max \lcb \varscale^2 \max_i \frac{w_i \nai_2^2}{\nai_1^2}, \frac{\numclients - \selclients}{\numclients-1} \max_i w_i \rcb \leq \frac{1}{64}, \lrsx \leq \frac{\lrsy}{156 \kappa^2},$$
$$8 \seff \lrsx \Lp \varscale^2 \frac{\numclients}{\selclients} \max \lcb \max_{i,k} \frac{w_i \aikt}{\nai_1}, \frac{\numclients - \selclients}{\numclients-1} \max_i w_i \rcb \leq 1, \seff \lrsx \Lf \heteroscale \sqrt{ \frac{\numclients}{\selclients}} \max \lcb \frac{\numclients - \selclients}{\numclients-1} \max_i w_i, \varscale \sqrt{\max_i \frac{w_i \nai_2^2}{\nai_1^2}} \rcb \leq \frac{1}{40 \kappa}$$
}%
The client learning rates $\lrcx, \lrcy$ satisfy $\lrcy \Lf \heteroscale \leq \frac{1}{16 \sqrt{\kappa \TMa}}$ and $\lrcx$: $\lrcx \Lf \heteroscale \leq \frac{1}{64 \kappa \sqrt{\TMa}}$, respectively.
Then the iterates generated by \fedsgda \ (\cref{alg_NC_minimax}) satisfy
\begin{equation}
    \begin{aligned}
        & \frac{1}{T} \sumtT \mbe \lb \TPhi (\bxt) - \TF (\bxt, \byt) \rb \\
        & \leq \frac{4 \lb \TPhi (\bx^{(0)}) - \TF(\bx^{(0)}, \by^{(0)}) \rb}{\seff \lrsy \mu T} + \frac{1}{12 \mu \kappa^2} \frac{1}{T} \sumtT \mbe \norm{\G \TPhi(\bxt)}^2 + \frac{8 \seff \lrsxsq \Lp}{\lrsy \mu} \frac{\numclients (\selclients-1)}{\selclients (\numclients-1)} \mbe \lnr \sumin w_i \bhxit \rnr^2 \\
        & \quad + 6 \seff \lrsy \kappa \frac{\numclients}{\selclients} \lb \localvar^2 \sumin \frac{w_i^2 \nai_2^2}{\nai_1^2} + 2 \hetero^2 \lp \frac{\numclients - \selclients}{\numclients-1} \max_i w_i + \varscale^2 \max_i \frac{w_i \nai_2^2}{\nai_1^2} \rp \rb \\
        & \quad + 18 \kappa \Lf \lp \lrcxsq + \lrcysq \rp \lb \localvar^2 \sumin w_i \norm{\boldsymbol a_{i,-1}}_2^2 + 2 \hetero^2 \TMa \rb \\
        & \quad + \frac{8 \seff \lrsxsq \kappa}{\lrsy} \frac{\numclients}{\selclients} \lp \frac{\numclients - \selclients}{\numclients-1} \max_i w_i + \varscale^2 \max_{i,k} \frac{w_i \aikt}{\nai_1} \rp \lp \lrcxsq + \lrcysq \rp \Lf^2 \lb \localvar^2 \sumin w_i \norm{\boldsymbol a_{i,-1}}_2^2 + 2 \hetero^2 \TMa \rb.
    \end{aligned}
    \label{eq:lem:NC_SC_PCP_WOR_phi_f_diff}
\end{equation}
\end{lemma}

\subsection{Proof of \texorpdfstring{\cref{thm:NC_SC}}{Theorem 2}}
\label{app:NC_SC_PCP_thm_proof}

For the sake of completeness, we first state the full statement of \Cref{thm:NC_SC} here.

\begin{theorem*}
Suppose the local loss functions $\{ f_i \}_i$ satisfy Assumptions \ref{assum:smoothness}, \ref{assum:bdd_var}, \ref{assum:bdd_hetero}, \ref{assum:SC_y}.
Suppose the server selects clients using without-replacement sampling scheme \textbf{(WOR)}.
Also, the server learning rates $\lrsx, \lrsy$ and the client learning rates $\lrcx, \lrcy$ satisfy the conditions specified in \cref{lem:NC_SC_PCP_phi_f_diff}.
Then the iterates generated by \fedsgda \ (\cref{alg_NC_minimax}) satisfy
\begin{equation}
    \begin{aligned}
        \min_{t \in [0:T-1]} \mbe \norm{\G \TPhi(\bxt)}^2 & \leq \frac{1}{T} \sumtT \mbe \norm{\G \TPhi (\bxt)}^2 \leq \underbrace{\mco \lp \kappa^2 \lb \frac{\Delta_{\TPhi}}{\seff \lrsy T} + \frac{\lrsy \Lf}{P} \lp \Aw \localvar^2 + \Bw \varscale^2 \hetero^2 \rp \rb \rp}_{\text{Error with full synchronization}}\nn \\
        & \qquad + \underbrace{\mco \lp \kappa^2 \lp \lrcxsq + \lrcysq \rp \Lf^2 \lb \Cw \localvar^2 + D \hetero^2 \rb \rp}_{\text{Error due to local updates}} + \underbrace{\mco \lp \kappa^2 \frac{\numclients - \selclients}{\numclients-1} \frac{\lrsy \Lf \Ew \seff \hetero^2}{P} \rp}_{\text{Partial Participation Error}},
    \end{aligned}
\end{equation}
where $\kappa = \Lf/\mu$ is the condition number, $\TPhi(\bx) \triangleq \max_\by \TF (\bx, \by)$ is the envelope function, 
$\Delta_{\TPhi} \triangleq \TPhi (\bx^{(0)}) - \min_\bx \TPhi (\bx)$,
$\Aw \triangleq n \seff \sumin \frac{w_i^2 \nai_2^2}{\nai_1^2}$, 
$\Bw \triangleq n \seff \max_i \frac{w_i \nai_2^2}{\nai_1^2}$, 
$\Cw \triangleq \sumin \wi ( \nai_2^2 - [\alpha^{(t,\sync_i-1)}_{i}]^2 )$, 
$\Dw \triangleq \max_i (\varscale^2\norm{\mbf a_{i,-1}}_2^2 + \norm{\mbf a_{i,-1}}_1^2)$, where $\mbf a_{i,-1}  \triangleq [a_i^{(0)}, a_i^{(1)}, \dots, a_i^{(\sync_i-2)}]^\top$ for all $i$ and
$\Ew \triangleq \numclients \max_i w_i$.

Using $\lrsy = \Theta \big( \frac{1}{\Lf} \sqrt{\frac{\selclients}{\bar{\sync} T}} \big)$ and $\lrcx \leq \lrcy = \Theta \big( \frac{1}{\Lf \bar{\sync} \sqrt{T}} \big)$, where $\bar{\sync} = \frac{1}{\numclients} \sumin \sync_i$ in the bounds above, we get
\begin{align*}
    \min_{t \in [T]} \mbe \norm{\G \TPhi(\bxt)}^2 \leq \underbrace{\mco \lp \kappa^2 \frac{(\bar{\sync}/\seff) + \Aw \localvar^2 + \Bw \varscale^2 \hetero^2}{\sqrt{\selclients \bar{\sync} T}}\rp}_{\text{Error with full synchronization}} + \underbrace{\mco \lp \kappa^2 \frac{\Cw \localvar^2 + \Dp \hetero^2}{\bar{\sync}^2 T}\rp}_{\substack{\text{Local updates error}}} + \underbrace{\mco \lp \frac{\numclients - \selclients}{\numclients - 1} \cdot \frac{\kappa^2 \Ew \seff \hetero^2}{\sqrt{\selclients \bar{\sync} T}}\rp}_{\substack{\text{Partial participation} \\ \text{error}}}.
\end{align*}
\end{theorem*}

\begin{proof}
Using \cref{lem:NC_SC_PCP_Phi_decay_one_iter}, and substituting in the bound on iterates' drift from \cref{lem:NC_SC_PCP_consensus_error}, we can bound
\begin{align}
    & \mbe \lb \TPhi(\bxtp) - \TPhi(\bxt) \rb \leq - \frac{7 \seff \lrsx}{16} \mbe \norm{\G \TPhi(\bxt)}^2 - \frac{\seff \lrsx}{2} \lp 1 - \frac{\numclients (\selclients-1)}{\selclients (\numclients-1)} \seff \lrsx \Lp \rp \mbe \norm{\sumin w_i \bhxit}^2 \nn \\
    & \quad + \frac{9 \seff \lrsx\Lf^2}{4 \mu} \mbe \lb \TPhi (\bxt) - \TF (\bxt, \byt) \rb \nn \\
    & \quad + \frac{\seff^2 \lrsxsq \Lp}{2} \frac{\numclients}{\selclients} \lb \localvar^2 \sumin \frac{w_i^2 \nai_2^2}{\nai_1^2} + \hetero^2 \lp 2 (\max_i w_i) \frac{\numclients - \selclients}{\numclients-1} + 2 \varscale^2 \max_i \frac{w_i \nai_2^2}{\nai_1^2} \rp \rb \nn \\
    & \quad + \frac{5}{2} \seff \lrsx \lp \lrcxsq + \lrcysq \rp \Lf^2 \lb \localvar^2 \sumin w_i \norm{\boldsymbol a_{i,-1}}_2^2 + 2 \hetero^2 \TMa \rb \nn \\
    & \quad + 10 \seff \lrsx \Lf^2 \TMa \heteroscale^2 \lb \lrcxsq \mbe \lnr \G \TPhi (\bxt) \rnr^2 + \Lf \lp 2 \kappa \lrcxsq + \lrcysq \rp \mbe \lb \TPhi (\bxt) - \TF (\bxt, \byt) \rb \rb. \label{eq_proof:thm_NC_SC_PCP_WOR_1}
\end{align}
Summing \eqref{eq_proof:thm_NC_SC_PCP_WOR_1} over $t = 0, \hdots, T-1$, substituting the bound on $\mbe \lb \TPhi (\bxt) - \TF (\bxt, \byt) \rb$ from \cref{lem:NC_SC_PCP_phi_f_diff}, and rearranging the terms, we get
\begin{align}
    & \frac{1}{T} \sumtT \mbe \norm{\G \TPhi(\bxt)}^2 \nn \\
    &= \mco \lp \frac{\kappa^2 \Delta_{\TPhi}}{\seff \lrsy T} + \seff \lrsy \Lf \kappa^2 \frac{\numclients}{\selclients} \lb \localvar^2 \sumin \frac{w_i^2 \nai_2^2}{\nai_1^2} + \hetero^2 \lp \frac{\numclients - \selclients}{\numclients-1} \max_i w_i + \varscale^2 \max_i \frac{w_i \nai_2^2}{\nai_1^2} \rp \rb \rp \nn \\
    & \quad + \mco \lp \kappa^2 \lp \lrcxsq + \lrcysq \rp \Lf^2 \lb \localvar^2 \sumin w_i \lp \nai_2^2 - [a^{(t,\sync_i-1)}_{i}]^2 \rp + \hetero^2 \max_i \lp \norm{\boldsymbol a_{i,-1}}_1^2 + \varscale^2 \norm{\boldsymbol a_{i,-1}}_2^2 \rp \rb \rp 
    \label{eq_proof:thm_NC_SC_PCP_WOR_3}
\end{align}
Consequently, using constants $\Aw, \Bw, \Cw, D, \Ew$, \eqref{eq_proof:thm_NC_SC_PCP_WOR_3} can be simplified to
{\small
\begin{align}
    & \frac{1}{T} \sumtT \mbe \norm{\G \TPhi(\bxt)}^2 \leq \mco \lp \kappa^2 \lb \frac{\Delta_{\TPhi}}{\seff \lrsy T} + \frac{\lrsy \Lf}{P} \lp \Aw \localvar^2 + (\Bw \varscale^2 + \Ew \seff) \hetero^2 \rp \rb + \kappa^2 \lp \lrcxsq + \lrcysq \rp \Lf^2 \lb \Cw \localvar^2 + D \hetero^2 \rb \rp. \nn
\end{align}
}%
which completes the proof.
\end{proof}

\paragraph{Convergence in terms of $F$}
\label{app:NC_SC_bias_conv}
\begin{proof}[Proof of \cref{cor:obj_inconsistent}]
According to the definition of $\obj(\bx)$ and $\surloss(\bx)$, we have
\begin{align}
	& \G \Phi(\bx) - \G \TPhi(\bx)
	= \sumin \lb p_i \Gx f_i(\bx, \by^*(\bx)) - w_i \Gx f_i(\bx, \Tby^*(\bx)) \rb \tag{$\by^* (\bx) \in \argmax_{\by} F(\bx, \by)$} \\
	&= \sumin p_i \lb \Gx f_i(\bx, \by^*(\bx)) - \Gx f_i(\bx, \Tby^*(\bx)) \rb + \sumin \lp p_i - w_i \rp \Gx f_i(\bx, \Tby^*(\bx)) \nn \\
	&= \lb \Gx F(\bx, \by^*(\bx)) - \Gx F(\bx, \Tby^*(\bx)) \rb + \sumin \frac{p_i-w_i}{\sqrt{w}_i} \cdot \sqrt{w}_i \Gx f_i(\bx, \Tby^*(\bx)). \nn
\end{align}
Taking norm, using $\Lf$-smoothness and applying Cauchy–Schwarz inequality, we get
\begin{align}
    \norm{\G \Phi(\bx) - \G \TPhi(\bx)}^2 & \leq 2 \Lf^2 \norm{\by^*(\bx) - \Tby^*(\bx)}^2 + 2 \lb \sumin \frac{(p_i-w_i)^2}{w_i} \rb \lb \sumin w_i \norm{\Gx f_i(\bx, \Tby^*(\bx))}^2 \rb \nn \\
    & \leq 2 \Lf^2 \norm{\by^*(\bx) - \Tby^*(\bx)}^2 + 2 \csqdist \lb \heteroscale^2 \norm{\G \TPhi(\bx)}^2 + \hetero^2 \rb, \nn
\end{align}
where the last inequality uses \Cref{assum:bdd_hetero}. Next, note that
\begin{align*}
    \norm{\G \Phi(\bx)}^2 \leq & 2 \norm{\G \Phi(\bx) - \G \TPhi(\bx)}^2 + 2 \norm{\nabla \TPhi(\bx)}^2.
\end{align*}
Therefore, we obtain
\begin{align}
    & \min_{t\in[T]} \norm{\nabla \Phi(\bx^{(t)})}^2
    \leq \frac{1}{T} \sum_{t=0}^{T-1} \norm{\nabla \Phi(\bx^{(t)})}^2 \nn \\
    & \leq 2 \lb 2 \csqdist \heteroscale^2 + 1 \rb \frac{1}{T} \sum_{t=0}^{T-1} \norm{\nabla \TPhi(\bx^{(t)})}^2 + 4 \lb \csqdist \hetero^2 + \Lf^2 \avgtT \norm{\by^*(\bxt) - \Tby^*(\bxt)}^2 \rb \nn \\
    &= 2 \lb 2 \csqdist \heteroscale^2 + 1 \rb \epsilon_{\text{opt}} + 4 \lb \csqdist \hetero^2 + \Lf^2 \avgtT \norm{\by^*(\bxt) - \Tby^*(\bxt)}^2 \rb. \nn
\end{align}
where $\epsilon_{\text{opt}}$ denotes the optimization error in the right hand side of \eqref{eq:thm:NC_SC} in \cref{thm:NC_SC}.
\end{proof}

\begin{proof}[Proof of \cref{cor:NC_SC_comm_cost}]
If clients are weighted equally ($w_i = p_i = 1/n$ for all $i$), with each carrying out $\sync$ steps of local SGDA, as seen in \eqref{eq:thm:NC_SC_simplified} we get
\begin{align}
    \min_{t\in[T]} \norm{\nabla \Phi(\bx^{(t)})}^2 \leq & \mco \Big( \mfrac{\localvar^2 + \varscale^2 \hetero^2}{\sqrt{\selclients \sync T}} + \mfrac{\localvar^2 + \sync \hetero^2}{\sync T} + \lp \mfrac{\numclients - \selclients}{\numclients - 1} \rp \hetero^2 \sqrt{\mfrac{\sync}{\selclients T}} \Big). \nn
\end{align}
\begin{itemize}
    \item For full client participation, this reduces to
    \begin{align*}
        \min_{t \in [T]} \mbe \norm{\G \TPhi(\bxt)}^2 \leq \mco \lp \frac{1}{\sqrt{\numclients \sync T}} + \frac{1}{T} \rp.
    \end{align*}
    To reach an $\epsilon$-stationary point, assuming $\numclients \sync \leq T$, the per-client gradient complexity is $T \sync = \mco \lp \frac{1}{\numclients \epsilon^4} \rp$. Since $\sync \leq T/n$, the minimum number of communication rounds required is $T = \mco \lp \frac{1}{\epsilon^2} \rp$.
    \item For partial participation, $\mco \Big( \lp \mfrac{\numclients - \selclients}{\numclients - 1} \rp \hetero^2 \sqrt{\mfrac{\sync}{\selclients T}} \Big)$ is the dominant term, and we do not get any convergence benefit of multiple local updates. Consequently, per-gradient client complexity and number of communication rounds are both $T \sync = \mco \lp \frac{1}{\selclients \epsilon^4} \rp$, for $\sync = \mco (1)$. However, if the data across clients comes from identical distributions ($\hetero = 0$), then we recover 
    per-client gradient complexity of $\mco \lp \frac{1}{\selclients \epsilon^4} \rp$, and number of communication rounds $= \mco \lp \frac{1}{\epsilon^2} \rp$.
\end{itemize}
\end{proof}

\subsection{Proofs of the Intermediate Lemmas}
\label{app:NC_SC_PCP_int_results_proofs}

\begin{proof}[Proof of \cref{lem:NC_SC_PCP_wtd_dir_norm_sq}]
\begin{align}
    \mbe \lnr \sumiS \twi \bdxit \rnr^2 &= \mbe \lnr \sumiS \twi \lp \bdxit - \bhxit + \bhxit \rp \rnr^2 \nn \\
    & \quad = \mbe \lnr \sumiS \twi \lp \bdxit - \bhxit \rp \rnr^2 + \mbe \lnr \sumiS \twi \bhxit \rnr^2 \nn \\
    & \quad = \mbe \lb \sumiS \twi^2 \lnr \bdxit - \bhxit \rnr^2 \rb + \mbe \lnr \sumiS \twi \bhxit \rnr^2 \tag{sampling scheme} \\
    & \quad = \frac{n}{P} \sumin w_i^2 \mbe \lnr \bdxit - \bhxit \rnr^2 + \mbe \lnr \sumiS \twi \bhxit \rnr^2 \tag{$\because \twi = w_i \numclients/\selclients$} \\
    & \quad \leq \frac{n}{P} \sumin \frac{w_i^2}{\nai_1^2} \sumikt [ \aikt ]^2 \lb \localvar^2 + \varscale^2 \mbe \lnr \Gx f_i ( \bxitk, \byitk ) \rnr^2 \rb + \mbe \lnr \sumiS \twi \bhxit \rnr^2. \label{eq_proof:lem:NC_SC_PCP_WOR_wtd_dir_norm_sq_1}
\end{align}
where the last inequality follows from \cref{assum:smoothness} and \ref{assum:bdd_var}.
Further, we can bound the second term as follows.
\begin{align}
    & \mbe \lnr \sumiS \twi \bhxit - \sumin w_i \bhxit + \sumin w_i \bhxit \rnr^2 \nn \\
    & = \mbe \lnr \sumin w_i \bhxit \rnr^2 + \mbe \lnr \sumin \mathbb{I} (i \in \clientset) \twi \bhxit - \sumin w_i \bhxit \rnr^2 \tag{(WOR) sampling} \\
    & = \mbe \lnr \sumin w_i \bhxit \rnr^2 + \sumin \mbe \lb \lp (\mathbb{I} (i \in \clientset))^2 \twi^2 + w_i^2 - 2 \mathbb{I} (i \in \clientset) \twi w_i \rp \lnr \bhxit \rnr^2 \rb \nn \\
    & \qquad \qquad + \sum_{i \neq j} \mbe \lan (\mathbb{I} (i \in \clientset) \twi - w_i) \bhxit, (\mathbb{I} (j \in \clientset) \twj - \wj) \bhxjt \ran \nn \\
    & = \mbe \lnr \sumin w_i \bhxit \rnr^2 + \sumin \mbe \lb w_i^2 \lp \frac{\numclients}{\selclients} - 1 \rp \lnr \bhxit \rnr^2 \rb \nn \\
    & \qquad \qquad + \sum_{i \neq j} \mbe \lb \lp \mathbb{I} (i \in \clientset) \cdot \mathbb{I} (j \in \clientset) \twi \twj - \mathbb{I} (j \in \clientset) \twj w_i - \mathbb{I} (i \in \clientset) \twi \wj + w_i \wj \rp \lan \bhxit, \bhxjt \ran \rb \nn \\
    & = \mbe \lnr \sumin w_i \bhxit \rnr^2 + \lp \frac{\numclients}{\selclients} - 1 \rp \sumin \mbe \lb w_i^2 \lnr \bhxit \rnr^2 \rb + \sum_{i \neq j} \mbe \lb w_i \wj \lp \frac{\numclients}{\selclients} \lp \frac{\selclients-1}{\numclients-1} \rp - 1 \rp \lan \bhxit, \bhxjt \ran \rb \nn \\
    & = \frac{\numclients}{\selclients} \lp \frac{\selclients-1}{\numclients-1} \rp \mbe \lnr \sumin w_i \bhxit \rnr^2 + \frac{\numclients}{\selclients} \frac{\numclients - \selclients}{\numclients-1} \sumin w_i^2 \mbe \lnr \bhxit \rnr^2, \label{eq_proof:lem:NC_SC_PCP_WOR_wtd_dir_norm_sq_2}
\end{align}
Next, we bound the second term in \eqref{eq_proof:lem:NC_SC_PCP_WOR_wtd_dir_norm_sq_2}.
\begin{align}
    & \sumin w_i^2 \mbe \lnr \bhxit - \Gx f_i(\bxt, \byt) + \Gx f_i(\bxt, \byt) \rnr^2 \nn \\
    & \leq 2 \Lf^2 \sumin \frac{w_i^2}{\nai_1} \sumikt \aikt \CExytk (i) + 2 (\max_i w_i) \lp \heteroscale^2 \norm{\Gx \TF (\bxt, \byt)}^2 + \hetero^2 \rp, \label{eq_proof:lem:NC_SC_PCP_WOR_wtd_dir_norm_sq_3}
\end{align}
using \cref{assum:bdd_hetero}. $\CExytk (i) \triangleq \mbe \lb \| \bxitk - \bxt \|^2 + \| \byitk - \byt \|^2 \rb$. Substituting \eqref{eq_proof:lem:NC_SC_PCP_WOR_wtd_dir_norm_sq_3} in \eqref{eq_proof:lem:NC_SC_PCP_WOR_wtd_dir_norm_sq_2}, and using the resulting bound in \eqref{eq_proof:lem:NC_SC_PCP_WOR_wtd_dir_norm_sq_1} we get the bound in \eqref{eq:lem:NC_SC_PCP_WOR_wtd_dir_norm_sq}.
\end{proof}

\begin{proof}[Proof of \Cref{lem:NC_SC_PCP_Phi_decay_one_iter}]
\label{proof:lem:NC_SC_PCP_Phi_decay_one_iter}
Since the local functions $\{ f_i \}$ satisfy \cref{assum:SC_y}, $F(\bx, \cdot)$ is $\mu$-SC for any $\bx$.
In the proof, we use the quadratic growth property of $\mu$-SC function $F(\bx, \cdot)$, i.e., for any given $\bx$
\begin{align}
    \frac{\mu}{2} \norm{\by - \by^*(\bx)}^2 \leq F(\bx, \by^*(\bx)) - F(\bx, \by), \quad \text{ for all } \by, \label{eq:quad_growth_PL}
\end{align}
where $\by^*(\bx) = \argmax_{\by'} F(\bx, \by')$. Using $\Lp$-smoothness of $\TPhi(\cdot)$,
\begin{align}
    & \mbe \TPhi(\bxtp) \leq \mbe \TPhi(\bxt) - \mbe \left\langle \G \TPhi(\bxt), \seff \lrsx \frac{1}{|\clientset|} \sumiS \bdxit \right\rangle + \frac{\seff^2 \lrsxsq \Lp}{2} \mbe \lnr \frac{1}{|\clientset|} \sumiS \bdxit \rnr^2 \nn \\
    &= \mbe \TPhi(\bxt) - \seff \lrsx \mbe \lan \G \TPhi(\bxt), \sumin w_i \bhxit \ran + \frac{\seff^2 \lrsxsq \Lp}{2} \mbe \lnr \bdxt \rnr^2 \tag{using \cref{assum:bdd_var} and \eqref{eq:server_agg_WOR}} \\
    &= \mbe \TPhi(\bxt) - \frac{\seff \lrsx}{2} \mbe \lb \norm{\G \TPhi(\bxt)}^2 + \norm{\sumin w_i \bhxit}^2 \rb + \frac{\seff \lrsx}{2} \mbe \norm{\G \TPhi(\bxt) - \sumin w_i \bhxit}^2 \nn \\
    & \quad + \frac{\seff^2 \lrsxsq \Lp}{2} \mbe \lnr \bdxt \rnr^2. \label{eq_proof:lem:NC_SC_PCP_Phi_decay_1}
\end{align}
Next, 
\begin{align}
    & \mbe \norm{\G \TPhi(\bxt) - \sumin w_i \bhxit}^2 \nn \\
    & = \mbe \norm{\sumin w_i \lp \Gx f_i(\bxt, \by^* (\bxt)) - \Gx f_i(\bxt, \byt) + \Gx f_i(\bxt, \byt) - \bhxit \rp}^2 \tag{since $\by^*(\bx) = \argmax_{\by'} F(\bx, \by')$} \\
    & \leq 2 \Lf^2 \mbe \norm{\by^* (\bxt) - \byt}^2 + 2 \mbe \norm{\sumin w_i \lp \Gx f_i(\bxt, \byt) - \frac{1}{\nai_1} \sumikt \aikt \Gx f_i \lp \bxitk, \byitk \rp \rp}^2 \tag{Jensen's inequality; $\Lf$-smoothness; Young's inequality} \\
    & \leq \frac{4 \Lf^2}{\mu} \mbe \lb \TPhi (\bxt) - \TF (\bxt, \byt) \rb + 2 \sumin w_i \frac{1}{\nai_1} \sumikt \aikt \mbe \norm{\Gx f_i(\bxt, \byt) - \Gx f_i (\bxitk, \byitk)}^2 \tag{Quadratic growth of $\mu$-SC functions \eqref{eq:quad_growth_PL}; Jensen's inequality} \\
    & \leq \frac{4 \Lf^2}{\mu} \mbe \lb \TPhi (\bxt) - \TF (\bxt, \byt) \rb + 2 \Lf^2 \sumin \frac{w_i}{\nai_1} \sumikt \aikt \mbe \lb \lnr \bxitk - \bxt \rnr^2 + \lnr \byitk - \byt \rnr^2 \rb \tag{$\Lf$-smoothness} \\
    & = \frac{4 \Lf^2}{\mu} \mbe \lb \TPhi (\bxt) - \TF (\bxt, \byt) \rb + 2 \Lf^2 \sumin \frac{w_i}{\nai_1} \sumikt \aikt \CExytk (i). \label{eq_proof:lem:NC_SC_Phi_decay_2}
\end{align}
where $\CExytk (i) \triangleq \mbe \lb \lnr \bxitk - \bxt \rnr^2 + \lnr \byitk - \byt \rnr^2 \rb$. 
Further, the term containing $\lnr \Gx f_i ( \bxitk, \byitk ) \rnr^2$ in \eqref{eq:lem:NC_SC_PCP_WOR_wtd_dir_norm_sq} is bounded in \cref{lem:NC_SC_avg_grad_x}.
substituting the bounds from \eqref{eq_proof:lem:NC_SC_Phi_decay_2},  \eqref{eq:lem:NC_SC_PCP_WOR_wtd_dir_norm_sq} and \cref{lem:NC_SC_avg_grad_x} into \eqref{eq_proof:lem:NC_SC_PCP_Phi_decay_1}, we get
\begin{align}
    &\mbe \TPhi(\bxtp)
    \leq \mbe \TPhi(\bxt) - \frac{\seff \lrsx}{2} \mbe \lb \norm{\G \TPhi(\bxt)}^2 + \norm{\sumin w_i \bhxit}^2 \rb \nn \\
    & \quad + \frac{\seff \lrsx}{2} \lb \frac{4 \Lf^2}{\mu} \mbe \lb \TPhi (\bxt) - \TF (\bxt, \byt) \rb + 2 \Lf^2 \sumin \frac{w_i}{\nai_1} \sumikt \aikt \CExytk (i) \rb \nn \\
    & \quad + \frac{\seff^2 \lrsxsq \Lp}{2} \lb \frac{\numclients}{\selclients} \lp \frac{\selclients-1}{\numclients-1} \rp \mbe \lnr \sumin w_i \bhxit \rnr^2 + \frac{\numclients}{\selclients} \lp \frac{\numclients - \selclients}{\numclients-1} \rp 2 \Lf^2 \sumin \frac{w_i^2}{\nai_1} \sumikt \aikt \CExytk (i) \rb \nn \\
    & \quad + \frac{\seff^2 \lrsxsq \Lp}{2} \frac{\numclients \localvar^2}{\selclients} \sumin \frac{w_i^2 \nai_2^2}{\nai_1^2} + \frac{\seff^2 \lrsxsq \Lp}{2} \frac{\numclients}{\selclients} \lp \frac{\numclients - \selclients}{\numclients-1} \rp 2 (\max_i w_i) \lp \heteroscale^2 \norm{\Gx \TF (\bxt, \byt)}^2 + \hetero^2 \rp \nn \\
    & \quad + \frac{\seff^2 \lrsxsq \Lp}{2} \frac{\numclients}{\selclients} \varscale^2 \lb 2 \Lf^2 \sumin \frac{w_i^2}{\nai_1^2} \sumikt [ \aikt ]^2 \CExytk (i) + 2 \hetero^2 \lp \max_i \frac{w_i \nai_2^2}{\nai_1^2} \rp \rb \nn \\
    & \quad + \frac{\seff^2 \lrsxsq \Lp}{2} \frac{\numclients}{\selclients} \varscale^2 4 \heteroscale^2 \lp \max_i \frac{w_i \nai_2^2}{\nai_1^2} \rp \lb \frac{2 \Lf^2}{\mu} \mbe \lp \TPhi (\bxt) - \TF (\bxt, \byt) \rp + \lnr \Gx \TPhi ( \bxt ) \rnr^2 \rb \nn \\
    & \leq \mbe \TPhi(\bxt) - \frac{7 \seff \lrsx}{16} \mbe \norm{\G \TPhi(\bxt)}^2 - \frac{\seff \lrsx}{2} \lp 1 - \frac{\numclients}{\selclients} \lp \frac{\selclients-1}{\numclients-1} \rp \seff \lrsx \Lp \rp \mbe \norm{\sumin w_i \bhxit}^2 \nn \\
    & \quad + \frac{9 \seff \lrsx \Lf^2}{4 \mu} \mbe \lb \TPhi (\bxt) - \TF (\bxt, \byt) \rb + \frac{5}{4} \seff \lrsx \Lf^2 \sumin \frac{w_i}{\nai_1} \sumikt \aikt \CExytk (i) \nn \\
    & \quad + \frac{\seff^2 \lrsxsq \Lp}{2} \frac{\numclients}{\selclients} \lb \localvar^2 \sumin \frac{w_i^2 \nai_2^2}{\nai_1^2} + \hetero^2 \lp 2 (\max_i w_i) \frac{\numclients - \selclients}{\numclients-1} + 2 \varscale^2 \max_i \frac{w_i \nai_2^2}{\nai_1^2} \rp \rb, \nn
\end{align}
where, 
the coefficients are simplified, using assumptions on the learning rate $\lrsx$
\begin{align*}
    \seff \lrsx \Lp \lb (\max_i w_i) \lp \frac{\numclients - \selclients}{\numclients-1} \rp + \varscale^2 \max_{i,k} \frac{w_i \aikt}{\nai_1} \rb & \leq \frac{\selclients}{4 \numclients} \\
    \seff \lrsx \Lp \lb (\max_i w_i) \lp \frac{\numclients - \selclients}{\numclients-1} \rp + \varscale^2 \max_i \frac{w_i \nai_2^2}{\nai_1^2} \rb & \leq \frac{\selclients}{32 \heteroscale^2 \numclients}.
\end{align*}
This finishes the proof of.
\end{proof}

\begin{proof}[Proof of \cref{lem:NC_SC_PCP_consensus_error}]
We use the client update equations for individual iterates \eqref{eq:client_update_alg_NC_SC}. To bound $\CExytk (i)$, first we bound a single component term $\mbe \lnr \bxitk - \bxt \rnr^2$. For $1 \leq k \leq \sync_i$,
Using modified variance assumption
\begin{align}
    & \mbe \lnr \bxitk - \bxt \rnr^2 = \lrcxsq \mbe \lnr \sumijk \aijk \lp \Gx f_i \lp \bxitj, \byitj; \xiitj \rp - \Gx f_i \lp \bxitj, \byitj \rp + \Gx f_i \lp \bxitj, \byitj \rp \rp \rnr^2 \nn \\
    &= \lrcxsq \lb \mbe \lnr \sumijk \aijk \lp \Gx f_i \lp \bxitj, \byitj; \xiitj \rp - \Gx f_i \lp \bxitj, \byitj \rp \rp \rnr^2 + \mbe \lnr \sumijk \aijk \Gx f_i \lp \bxitj, \byitj \rp \rnr^2 \rb \tag{using unbiasedness in \cref{assum:bdd_var}} \\
    &= \lrcxsq \lb \sumijk [\aijk]^2 \mbe \lnr \Gx f_i \lp \bxitj, \byitj; \xiitj \rp - \Gx f_i \lp \bxitj, \byitj \rp \rnr^2 + \mbe \lnr \sumijk \aijk \Gx f_i \lp \bxitj, \byitj \rp \rnr^2 \rb \nn \\
    & \leq \lrcxsq \lb \sumijk [\aijk]^2 \lp \localvar^2 + \varscale^2 \norm{\Gx f_i ( \bxitj, \byitj )}^2 \rp + \lp \sumijk \aijk \rp \sumijk \aijk \mbe \lnr \Gx f_i \lp \bxitj, \byitj \rp \rnr^2 \rb. \label{eq_proof:lem:NC_SC_PCP_consensus_error_1}
\end{align}
where the last inequality follows from Jensen's inequality (\cref{lem:jensens}). Next, note that
\begin{equation}
    \begin{aligned}
        \frac{1}{\nai_1} \sumikt \aikt \sumijk [\aijk]^2 & \leq \frac{1}{\nai_1} \sumikt \aikt \sum_{k=0}^{\sync_i-2} [\aijk]^2 = \sum_{k=0}^{\sync_i-2} [\aijk]^2 \leq \nai_2^2 - [a^{(t,\sync_i-1)}_{i} (\sync_i)]^2, \\
        \frac{1}{\nai_1} \sumikt \aikt \sumijk \aijk & \leq \frac{1}{\nai_1} \sumikt \aikt \sum_{k=0}^{\sync_i-2} \aijk = \sum_{k=0}^{\sync_i-2} \aijk \leq \nai_1 - [a^{(t,\sync_i-1)}_{i} (\sync_i)],
        \\
        \frac{1}{\nai_1} \sumikt \aikt \sumijk [\aijk]^2 & \leq \frac{1}{\nai_1} \sumikt \aikt \sum_{j=0}^{\sync_i-2} [\aijk]^2 \leq \alpha \cdot \sum_{k=0}^{\sync_i-2} \aikt,
    \end{aligned}
    \label{eq_proof:lem:NC_SC_PCP_consensus_error_2}
\end{equation}
We define $\norm{\boldsymbol a_{i,-1}}_2^2 \triangleq \nai_2^2 - [a^{(t,\sync_i-1)}_{i} (\sync_i)]^2$, $\norm{\boldsymbol a_{i,-1}}_1 \triangleq \nai_1 - [a^{(t,\sync_i-1)}_{i} (\sync_i)]$
for the sake of brevity. Using \eqref{eq_proof:lem:NC_SC_PCP_consensus_error_2}, we bound the individual terms in \eqref{eq_proof:lem:NC_SC_PCP_consensus_error_1}.
\begin{align}
    & \frac{1}{\nai_1} \sumikt \aikt \sumijk [\aijk]^2 \varscale^2 \norm{\Gx f_i ( \bxitj, \byitj )}^2 \nn \\
    & \leq 2 \varscale^2 \alpha \Lf^2 \sum_{j=0}^{\sync_i-2} [\aijk] \CExytj + 2 \varscale^2 \norm{\boldsymbol a_{i,-1}}_2^2 \norm{\Gx f_i ( \bxt, \byt )}^2.
    \label{eq_proof:lem:NC_SC_PCP_consensus_error_2a}
\end{align}
Similarly, 
\begin{align}
    & \frac{1}{\nai_1} \sumikt \aikt \lp \sumijk \aijk \rp \sumijk \aijk \mbe \lnr \Gx f_i \lp \bxitj, \byitj \rp \rnr^2 \nn \\
    & \leq \frac{2}{\nai_1} \sumikt \aikt \lp \sum_{j=0}^{\sync_i-2} \aijk \rp \sum_{j=0}^{\sync_i-2} \aijk \lb \Lf^2 \CExytj + \norm{\Gx f_i ( \bxt, \byt )}^2 \rb \nn \\
    & \leq 2 \norm{\boldsymbol a_{i,-1}}_1 \Lf^2 \sum_{j=0}^{\sync_i-2} \aijk \CExytj + 2 \norm{\boldsymbol a_{i,-1}}_1^2 \norm{\Gx f_i ( \bxt, \byt )}^2.
    \label{eq_proof:lem:NC_SC_PCP_consensus_error_2b}
\end{align}
Substituting \eqref{eq_proof:lem:NC_SC_PCP_consensus_error_2a}, \eqref{eq_proof:lem:NC_SC_PCP_consensus_error_2b} in \eqref{eq_proof:lem:NC_SC_PCP_consensus_error_1}, we get
\begin{align}
    \frac{1}{\nai_1} \sumikt \aikt \mbe \lnr \bxitk - \bxt \rnr^2 & \leq \lrcxsq \localvar^2 \norm{\boldsymbol a_{i,-1}}_2^2 + 2 \lrcxsq \Lf^2 \lp \norm{\boldsymbol a_{i,-1}}_1 + \varscale^2 \alpha \rp \sumikt \aikt \CExytk(i) \nn \\
    & \quad + 2 \lrcxsq \lp \norm{\boldsymbol a_{i,-1}}_1^2 + \varscale^2 \norm{\boldsymbol a_{i,-1}}_2^2 \rp \mbe \lnr \Gx f_i ( \bxt, \byt ) \rnr^2.
    \label{eq_proof:lem:NC_SC_PCP_consensus_error_3}
\end{align}
Similarly, we can bound the $\by$ error
\begin{align}
    \frac{1}{\nai_1} \sumikt \aikt \mbe \lnr \byitk - \byt \rnr^2 & \leq \lrcysq \localvar^2 \norm{\boldsymbol a_{i,-1}}_2^2 + 2 \lrcysq \Lf^2 \lp \norm{\boldsymbol a_{i,-1}}_1 + \varscale^2 \alpha \rp \sumikt \aikt \CExytk(i) \nn \\
    & \quad + 2 \lrcysq \lp \norm{\boldsymbol a_{i,-1}}_1^2 + \varscale^2 \norm{\boldsymbol a_{i,-1}}_2^2 \rp \mbe \lnr \Gy f_i ( \bxt, \byt ) \rnr^2.
    \label{eq_proof:lem:NC_SC_PCP_consensus_error_4}
\end{align}
Combining the two bounds in \eqref{eq_proof:lem:NC_SC_PCP_consensus_error_3} and \eqref{eq_proof:lem:NC_SC_PCP_consensus_error_4}, we get
\begin{align}
    & \frac{1}{\nai_1} \sumikt \aikt \mbe \lb \lnr \bxitk - \bxt \rnr^2 + \lnr \byitk - \byt \rnr^2 \rb \leq \lp \lrcxsq + \lrcysq \rp \localvar^2 \norm{\boldsymbol a_{i,-1}}_2^2 \nn \\
    & \quad + 2 \lp \lrcxsq + \lrcysq \rp \Lf^2 \nai_1 \lp \norm{\boldsymbol a_{i,-1}}_1 + \varscale^2 \alpha \rp \frac{1}{\nai_1} \sumikt \aikt \CExytk (i) \nn \\
    & \quad + 2 \lp \norm{\boldsymbol a_{i,-1}}_1^2 + \varscale^2 \norm{\boldsymbol a_{i,-1}}_2^2 \rp \lb \lrcxsq \mbe \lnr \Gx f_i \lp \bxt, \byt \rp \rnr^2 + \lrcysq \mbe \lnr \Gy f_i \lp \bxt, \byt \rp \rnr^2 \rb. \label{eq_proof:lem:NC_SC_PCP_consensus_error_5}
\end{align}
Define $A_m \triangleq 2 \Lf^2 \lp \lrcxsq + \lrcysq \rp \max_i \nai_1 \lp \norm{\boldsymbol a_{i,-1}}_1 + \varscale^2 \alpha \rp$. Rearranging the terms in \eqref{eq_proof:lem:NC_SC_PCP_consensus_error_5}, and taking the weighted sum over agents, we get
\begin{align}
    & \Lf^2 \sumin \frac{w_i}{\nai_1} \sumikt \aikt \CExytk (i) \nn \\
    & \leq \frac{\lp \lrcxsq + \lrcysq \rp \Lf^2 \localvar^2}{1-A_m} \sumin w_i \norm{\boldsymbol a_{i,-1}}_2^2 \nn \\
    & \quad + \frac{2 \Lf^2}{1-A_m} \sumin w_i \lp \norm{\boldsymbol a_{i,-1}}_1^2 + \varscale^2 \norm{\boldsymbol a_{i,-1}}_2^2 \rp \lb \lrcxsq \mbe \lnr \Gx f_i \lp \bxt, \byt \rp \rnr^2 + \lrcysq \mbe \lnr \Gy f_i \lp \bxt, \byt \rp \rnr^2 \rb \nn \\
    & \leq \frac{\lp \lrcxsq + \lrcysq \rp \Lf^2}{1-A_m} \localvar^2 \sumin w_i \norm{\boldsymbol a_{i,-1}}_2^2 + \frac{2 \Lf^2 \TMa}{1-A_m} \lrcxsq \lp \heteroscale^2 \mbe \lnr \sumin w_i \Gx f_i \lp \bxt, \byt \rp \rnr^2 + \hetero^2 \rp \nn \\
    & \quad + \frac{2 \Lf^2 \TMa}{1-A_m} \lrcysq \lp \heteroscale^2 \mbe \lnr \sumin w_i \Gy f_i \lp \bxt, \byt \rp \rnr^2 + \hetero^2 \rp. \label{eq_proof:lem:NC_SC_PCP_consensus_error_6}
\end{align}
where \eqref{eq_proof:lem:NC_SC_PCP_consensus_error_6} follows from \cref{assum:bdd_hetero}, and we define $\TMa \triangleq \max_i \lp \norm{\boldsymbol a_{i,-1}}_1^2 + \varscale^2 \norm{\boldsymbol a_{i,-1}}_2^2 \rp$. 
We bounded $\mbe \lnr \Gx F \lp \bxt, \byt \rp \rnr^2$ in \cref{lem:grad_TF_wrt_x}. Similarly, we can bound $\mbe \lnr \Gy F \lp \bxt, \byt \rp \rnr^2$ as follows.
\begin{align}
    \mbe \lnr \Gy F \lp \bxt, \byt \rp \rnr^2 & = \mbe \lnr \Gy F \lp \bxt, \byt \rp - \Gy F \lp \bxt, \by^* (\bxt) \rp \rnr^2 \tag{$\because \by^*(\bx) = \argmax_{\by'} F(\bx, \by')$} \\
    & \leq 2 \Lf \mbe \lb \TPhi (\bxt) - \TF (\bxt, \byt) \rb. \label{eq:grad_TF_wrt_y}
\end{align}
using $\Lf$-smoothness and concavity of $F(\bx, \cdot)$ (\cref{lem:smooth_convex}). Also, for the choice of $\lrcx, \lrcy$, we get $A_m \leq 1/2$. 
Consequently, substituting the two bounds in \eqref{eq_proof:lem:NC_SC_PCP_consensus_error_6}, we complete the proof.
\end{proof}

\begin{proof}[Proof of \cref{lem:NC_SC_PCP_phi_f_diff}]
Using $\Lf$-smoothness (\cref{assum:smoothness}) of $F(\bx, \cdot)$,
\begin{align}
    & \mbe \TF (\bxtp, \byt) \leq \mbe \TF (\bxtp, \bytp) - \mbe \lan \Gy \TF (\bxtp, \byt), \bytp - \byt \ran + \frac{\Lf}{2} \norm{\bytp - \byt}^2 \nn \\
    &= \mbe \TF (\bxtp, \bytp) - \seff \lrsy \mbe \lan \Gy \TF (\bxtp, \byt), \sumin w_i \bhyit \ran + \frac{\seff^2 \lrsysq \Lf}{2} \mbe \norm{\bdyt}^2, \nn \\
    & \leq \mbe \TF (\bxtp, \bytp) - \frac{\seff \lrsy}{2} \mbe \lb \norm{\Gy \TF (\bxtp, \byt)}^2 + \norm{\sumin w_i \bhyit}^2 - \norm{\Gy \TF (\bxtp, \byt) - \sumin w_i \bhyit}^2 \rb \nn \\
    & \qquad + \frac{\seff^2 \lrsysq \Lf}{2} \mbe \norm{\bdyt}^2.
    \label{eq_proof:lem:NC_SC_PCP_phi_f_diff_1}
\end{align}
Next, we bound the individual terms in \eqref{eq_proof:lem:NC_SC_PCP_phi_f_diff_1}. Using quadratic growth of $\mu$-SC functions (\cref{lem:quad_growth})
\begin{align*}
    \mbe \norm{\Gy \TF (\bxtp, \byt)}^2 \geq 2 \mu \mbe \lb \TPhi (\bxtp) - \TF (\bxtp, \byt) \rb.
\end{align*}
Next, we bound $\norm{\Gy \TF (\bxtp, \byt) - \sumin w_i \bhyit}^2$, using similar reasoning as in \eqref{eq_proof:lem:NC_SC_Phi_decay_2}.
\begin{align}
    & \mbe \norm{\Gy \TF (\bxtp, \byt) - \Gy \TF (\bxt, \byt) + \Gy \TF (\bxt, \byt) - \sumin w_i \bhyit}^2 \nn \\
    & \leq 2 \Lf^2 \norm{\bxtp - \bxt}^2 + 2 \norm{\sumin w_i \lp \Gy f_i (\bxt, \byt) - \bhyit \rp}^2 \nn \\
    & \leq 2 \Lf^2 \seff^2 \lrsxsq \mbe \lnr \bdxt \rnr^2 + 2 \Lf^2 \sumin \frac{w_i}{\nai_1} \sumikt \aikt \CExytk (i), \label{eq_proof:lem:NC_SC_PCP_phi_f_diff_2}
\end{align}

We can bound $\mbe \normb{\bdxt}^2$ using \eqref{eq:lem:NC_SC_PCP_WOR_wtd_dir_norm_sq} in \cref{lem:NC_SC_PCP_wtd_dir_norm_sq} to get
\begin{align}
    & \mbe \norm{\bdxt}^2 
    \leq \frac{\numclients (\selclients-1)}{\selclients (\numclients-1)} \mbe \lnr \sumin w_i \bhxit \rnr^2 + \frac{\localvar^2 \numclients}{\selclients} \sumin \frac{w_i^{2} \nai_2^2}{\nai_1^2} +  \frac{2 \hetero^2}{\selclients} \lp \frac{\numclients (\numclients - \selclients)}{\numclients-1} \max_i w_i \rp \nn \\
    & \quad + \frac{\varscale^2 \numclients}{\selclients} \lb 2 \Lf^2 \sumin \frac{w_i^{2}}{\nai_1^2} \sumikt [ \aikt ]^2 \CExytk (i) + 4 \heteroscale^2 \lp \max_i \frac{w_i \nai_2^2}{\nai_1^2} \rp \mbe \lnr \Gx \TPhi ( \bxt ) \rnr^2 \rb \nn \\
    & \quad + \frac{\varscale^2 \numclients}{\selclients} 2 \max_i \frac{w_i \nai_2^2}{\nai_1^2} \lb \hetero^2 + 4 \heteroscale^2 \Lf \kappa \mbe \lb \TPhi (\bxt) - \TF (\bxt, \byt) \rb \rb \nn \\
    & \quad + \frac{\numclients (\numclients - \selclients)}{\numclients-1} \frac{2 \Lf^2}{\selclients} \sumin \frac{w_i^{{2}}}{\nai_1} \sumikt \aikt \CExytk (i) \nn \\
    & \quad + \lp \frac{\numclients (\numclients - \selclients)}{\numclients-1} \max_i w_i \rp \frac{2 \heteroscale^2}{\selclients} \lb 4 \Lf \kappa \mbe \lb \TPhi (\bxt) - \TF (\bxt, \byt) \rb + 2 \lnr \Gx \TPhi ( \bxt ) \rnr^2 \rb \nn \\
    & \leq \frac{\numclients (\selclients-1)}{\selclients (\numclients-1)} \mbe \lnr \sumin w_i \bhxit \rnr^2 + \frac{\localvar^2 \numclients}{\selclients} \sumin \frac{w_i^{{2}} \nai_2^2}{\nai_1^2} + \frac{2 \numclients \hetero^2}{\selclients} \lp \frac{\numclients - \selclients}{\numclients-1} \max_i w_i + \varscale^2 \max_i \frac{w_i \nai_2^2}{\nai_1^2} \rp \nn \\
    & \quad + \lp \frac{\numclients - \selclients}{\numclients-1} \max_i w_i + \varscale^2 \max_{i,k} \frac{w_i \aikt}{\nai_1} \rp \frac{2 \numclients \Lf^2}{\selclients} \sumin \frac{w_i}{\nai_1} \sumikt \aikt \CExytk (i) \nn \\
    & \quad + \frac{4 \numclients \heteroscale^2}{\selclients} \lp \frac{\numclients - \selclients}{\numclients-1} \max_i w_i + \varscale^2 \max_i \frac{w_i \nai_2^2}{\nai_1^2} \rp \lnr \Gx \TPhi ( \bxt ) \rnr^2 \nn \\
    & \quad + \frac{8 \numclients \heteroscale^2 \Lf \kappa}{\selclients} \lp \frac{\numclients - \selclients}{\numclients-1} \max_i w_i + \varscale^2 \max_i \frac{w_i \nai_2^2}{\nai_1^2} \rp \mbe \lb \TPhi (\bxt) - \TF (\bxt, \byt) \rb. \label{eq_proof:lem:NC_SC_PCP_WOR_phi_f_diff_3a}
\end{align}
Similarly, we can bound $\mbe \normb{\bdyt}^2$ to get
\begin{align}
    \mbe \norm{\bdyt}^2 
    & \leq \frac{\numclients (\selclients-1)}{\selclients (\numclients-1)} \mbe \lnr \sumin w_i \bhyit \rnr^2 + \frac{\localvar^2 \numclients}{\selclients} \sumin \frac{w_i^{{2}} \nai_2^2}{\nai_1^2} + \frac{2 \hetero^2 \numclients}{\selclients} \lp \frac{\numclients - \selclients}{\numclients-1} \max_i w_i + \varscale^2 \max_i \frac{w_i \nai_2^2}{\nai_1^2} \rp \nn \\
    & \quad + \lp \frac{\numclients - \selclients}{\numclients-1} \max_i w_i + \varscale^2 \max_{i,k} \frac{w_i \aikt}{\nai_1} \rp \frac{2 \numclients \Lf^2}{\selclients} \sumin \frac{w_i}{\nai_1} \sumikt \aikt \CExytk (i) \nn \\
    & \quad + \frac{4 \heteroscale^2 \Lf \numclients}{\selclients} \lp \frac{\numclients - \selclients}{\numclients-1} \max_i w_i + \varscale^2 \max_i \frac{w_i \nai_2^2}{\nai_1^2} \rp \mbe \lb \TPhi (\bxt) - \TF (\bxt, \byt) \rb. \label{eq_proof:lem:NC_SC_PCP_WOR_phi_f_diff_3b}
\end{align}
Substituting \eqref{eq_proof:lem:NC_SC_PCP_phi_f_diff_2}, \eqref{eq_proof:lem:NC_SC_PCP_WOR_phi_f_diff_3a}, \eqref{eq_proof:lem:NC_SC_PCP_WOR_phi_f_diff_3b} and \cref{lem:NC_SC_PCP_consensus_error} in \eqref{eq_proof:lem:NC_SC_PCP_phi_f_diff_1}, and rearranging the terms, we get
\begin{align}
    & \mbe \TF (\bxtp, \byt) \nn \\
    & \leq \mbe \TF (\bxtp, \bytp) - \frac{\seff \lrsy}{2} \lp 1 - \seff \lrsy \Lf \frac{\numclients (\selclients-1)}{\selclients (\numclients-1)} \rp \mbe \norm{\sumin w_i \bhyit}^2 + \seff^3 \lrsxsq \lrsy \Lf^2 \frac{\numclients (\selclients-1)}{\selclients (\numclients-1)} \mbe \lnr \sumin w_i \bhxit \rnr^2 \nn \\
    & \quad - \seff \lrsy \mu \mbe \lb \TPhi (\bxtp) - \TF (\bxtp, \byt) \rb \nn \\
    & \quad + \seff^3 \lrsxsq \lrsy \Lf^2 \frac{4 \numclients \heteroscale^2}{\selclients} \lp \frac{\numclients - \selclients}{\numclients-1} \max_i w_i + \varscale^2 \max_i \frac{w_i \nai_2^2}{\nai_1^2} \rp \lnr \Gx \TPhi ( \bxt ) \rnr^2 \nn \\
    & \quad + \frac{\seff^2 \lrsysq \Lf}{2} \lp 1 + \frac{2 \seff \lrsxsq \Lf}{\lrsy} \rp \lb \frac{\localvar^2 \numclients}{\selclients} \sumin \frac{w_i^{{2}} \nai_2^2}{\nai_1^2} + \frac{2 \hetero^2 \numclients}{\selclients} \lp \frac{\numclients - \selclients}{\numclients-1} \max_i w_i + \varscale^2 \max_i \frac{w_i \nai_2^2}{\nai_1^2} \rp \rb \nn \\
    & \quad + \lb \seff \lrsy + \frac{\seff^2 \lrsysq \Lf \numclients}{\selclients} \lp \frac{2 \seff \lrsxsq \Lf}{\lrsy} + 1 \rp \lp \frac{\numclients - \selclients}{\numclients-1} \max_i w_i + \varscale^2 \max_{i,k} \frac{w_i \aikt}{\nai_1} \rp \rb \Lf^2 \sumin \frac{w_i}{\nai_1} \sumikt \aikt \CExytk (i) \nn \\
    & \quad + 2 \seff^2 \lrsysq \Lf \lp \frac{4 \seff \lrsxsq \Lf \kappa}{\lrsy} + 1 \rp \frac{\heteroscale^2 \numclients \Lf}{\selclients} \lp \frac{\numclients - \selclients}{\numclients-1} \max_i w_i + \varscale^2 \max_i \frac{w_i \nai_2^2}{\nai_1^2} \rp \mbe \lb \TPhi (\bxt) - \TF (\bxt, \byt) \rb \nn \\
    & \leq \mbe \TF (\bxtp, \bytp) - \frac{\seff \lrsy}{2} \lp 1 - \seff \lrsy \Lf \frac{\numclients (\selclients-1)}{\selclients (\numclients-1)} \rp \mbe \norm{\sumin w_i \bhyit}^2 + \seff^3 \lrsxsq \lrsy \Lf^2 \frac{\numclients (\selclients-1)}{\selclients (\numclients-1)} \mbe \lnr \sumin w_i \bhxit \rnr^2 \nn \\
    & \quad - \seff \lrsy \mu \mbe \lb \TPhi (\bxtp) - \TF (\bxtp, \byt) \rb \nn \\
    & \quad + \seff^3 \lrsy \lrsxsq \Lf^2 \frac{4 \heteroscale^2 \numclients}{\selclients} \lp \frac{\numclients - \selclients}{\numclients-1} \max_i w_i + \varscale^2 \max_i \frac{w_i \nai_2^2}{\nai_1^2} \rp \mbe \lnr \Gx \TPhi ( \bxt ) \rnr^2 \nn \\
    & \quad + \seff^2 \lrsysq \Lf \lb \frac{\localvar^2 \numclients}{\selclients} \sumin \frac{w_i^2 \nai_2^2}{\nai_1^2} + \frac{2 \hetero^2 \numclients}{\selclients} \lp \frac{\numclients - \selclients}{\numclients-1} \max_i w_i + \varscale^2 \max_i \frac{w_i \nai_2^2}{\nai_1^2} \rp \rb \tag{$\because 4 \seff \lrsx \Lf \leq 1, \lrsx \kappa \leq \lrsy$} \\
    & \quad + 2 \seff \lrsy \Lf^2 \sumin \frac{w_i}{\nai_1} \sumikt \aikt \CExytk (i) \tag{$\because 8 \seff \lrsx \Lp \leq \selclients$} \\
    & \quad + 4 \seff^2 \lrsysq \Lf^2 \frac{\heteroscale^2 \numclients}{\selclients} \lp \frac{\numclients - \selclients}{\numclients-1} \max_i w_i + \varscale^2 \max_i \frac{w_i \nai_2^2}{\nai_1^2} \rp \mbe \lb \TPhi (\bxt) - \TF (\bxt, \byt) \rb
    \label{eq_proof:lem:NC_SC_PCP_WOR_phi_f_diff_4a}
\end{align}
where we simplify some coefficients using $4 \seff \lrsx \Lf \leq 1, \lrsx \leq \kappa \lrsy$. We rearrange the terms
and use the bound in \cref{lem:NC_SC_PCP_consensus_error} to get
{\small
\begin{align}
    & \mbe \lb \TPhi (\bxtp) - \TF (\bxtp, \bytp) \rb
    \nn \\
    & \leq (1 - \seff \lrsy \mu) \mbe \lb \TPhi (\bxtp) - \TF (\bxtp, \byt) \rb \nn \\
    & \quad - \frac{\seff \lrsy}{2} \lp 1 - \seff \lrsy \Lf \frac{\numclients (\selclients-1)}{\selclients (\numclients-1)} \rp \mbe \norm{\sumin w_i \bhyit}^2 + \seff^3 \lrsxsq \lrsy \Lf^2 \frac{\numclients (\selclients-1)}{\selclients (\numclients-1)} \mbe \lnr \sumin w_i \bhxit \rnr^2 \nn \\
    & \quad + 4 \seff \lrsy \Lf^2 \heteroscale^2 \lb \frac{\seff^2 \lrsxsq \numclients}{\selclients} \lp \frac{\numclients - \selclients}{\numclients-1} \max_i w_i + \varscale^2 \max_i \frac{w_i \nai_2^2}{\nai_1^2} \rp + 4 \TMa \lrcxsq \rb \mbe \lnr \Gx \TPhi ( \bxt ) \rnr^2 \nn \\
    & \quad + \seff^2 \lrsysq \Lf \lb \frac{\numclients \localvar^2}{P} \sumin \frac{w_i^2 \nai_2^2}{\nai_1^2} + \frac{2 \numclients \hetero^2}{P} \lp \frac{\numclients - \selclients}{\numclients-1} \max_i w_i + \varscale^2 \max_i \frac{w_i \nai_2^2}{\nai_1^2} \rp \rb \nn \\
    & \quad + 4 \seff \lrsy \lp \lrcxsq + \lrcysq \rp \Lf^2 \lb \localvar^2 \sumin w_i \norm{\boldsymbol a_{i,-1}}_2^2 + 2 \TMa \hetero^2 \rb \nn \\
    & \quad + 4 \seff \lrsy \Lf^2 \heteroscale^2 \lb \frac{\seff \lrsy \numclients}{\selclients} \lp \frac{\numclients - \selclients}{\numclients-1} \max_i w_i + \varscale^2 \max_i \frac{w_i \nai_2^2}{\nai_1^2} \rp + 4 \Lf \TMa \lp 2 \kappa \lrcxsq + \lrcysq \rp \rb \mbe \lb \TPhi (\bxt) - \TF (\bxt, \byt) \rb
    \label{eq_proof:lem:NC_SC_PCP_WOR_phi_f_diff_5a}
\end{align}
}%
Next, note that
\begin{align}
    & \mbe \lb \TPhi (\bxtp) - \TF (\bxtp, \byt) \rb \nn \\
    &= \mbe \lb \TPhi (\bxtp) - \TPhi (\bxt) + \TPhi (\bxt) - \TF (\bxt, \byt) + \TF (\bxt, \byt) - \TF (\bxtp, \byt) \rb. \label{eq_proof:lem:NC_SC_PCP_phi_f_diff_5b}
\end{align}
Substituting the bound from \cref{lem:NC_SC_PCP_consensus_error} into \cref{lem:NC_SC_PCP_Phi_decay_one_iter}, $\mbe \lb \TPhi (\bxtp) - \TPhi (\bxt) \rb$ can be bounded as follows.
{\small
\begin{align}
    & \mbe \TPhi(\bxtp) - \mbe \TPhi(\bxt) \nn \\
    & \leq - \frac{3 \seff \lrsx}{8} \mbe \norm{\G \TPhi(\bxt)}^2 - \frac{\seff \lrsx}{2} \lp 1 - \seff \lrsx \Lp \frac{\numclients (\selclients-1)}{\selclients (\numclients-1)} \rp \mbe \norm{\sumin w_i \bhxit}^2 + \frac{5 \seff \lrsx \Lf^2}{2 \mu} \mbe \lb \TPhi (\bxt) - \TF (\bxt, \byt) \rb \nn \\
    & \quad + \frac{\seff^2 \lrsxsq \Lp}{2} \frac{\numclients}{\selclients} \lb \localvar^2 \sumin \frac{w_i^2 \nai_2^2}{\nai_1^2} + \hetero^2 \lp 2 (\max_i w_i) \frac{\numclients - \selclients}{\numclients-1} + 2 \varscale^2 \max_i \frac{w_i \nai_2^2}{\nai_1^2} \rp \rb \nn \\
    & \quad + \frac{5}{2} \seff \lrsx \lp \lrcxsq + \lrcysq \rp \Lf^2 \lb \localvar^2 \sumin w_i \norm{\boldsymbol a_{i,-1}}_2^2 + 2 \TMa \hetero^2 \rb. \label{eq_proof:lem:NC_SC_PCP_WOR_phi_f_diff_6}
\end{align}
}%
where, we use $20 \Lf^2 \TMa \heteroscale^2 \lrcxsq \leq \frac{1}{8}$, $40 \Lf \TMa \heteroscale^2 \lp 2 \kappa \lrcxsq + \lrcysq \rp \leq \frac{1}{\mu}$.
Next, we bound $\mbe \lb \TF (\bxt, \byt) - \TF (\bxtp, \byt) \rb$. Again, using $\Lf$-smoothness of $F(\cdot, \by)$,
{\small
\begin{align}
    & \mbe \lb \TF (\bxt, \byt) - \TF (\bxtp, \byt) \rb \nn \\
    & \leq \mbe \lb \lan -\Gx \TF (\bxt, \byt), \bxtp - \bxt \ran + \frac{\Lf}{2} \norm{\bxtp - \bxt}^2 \rb \nn \\
    & \leq \frac{\seff \lrsx}{2} \mbe \lb \norm{\Gx \TF (\bxt, \byt) - \G \TPhi (\bxt) + \G \TPhi (\bxt)}^2 + \norm{\sumin w_i \bhxit}^2 \rb + \frac{\seff^2 \lrsxsq \Lf}{2} \mbe \norm{\frac{1}{|\clientset|} \sumiS \bdxit}^2 \nn \\
    & \leq \seff \lrsx \mbe \lb \frac{2 \Lf^2}{\mu} \lb \TPhi (\bxt) - \TF (\bxt, \byt) \rb + \norm{\G \TPhi (\bxt)}^2 + \frac{1}{2} \norm{\sumin w_i \bhxit}^2 \rb \tag{$\Lf$-smoothness, \Cref{lem:quad_growth}} \\
    & \qquad + \frac{\seff^2 \lrsxsq \Lf}{2} \mbe \norm{\bdxt}^2 \nn \\
    & \leq \frac{3 \seff \lrsx \Lf^2}{\mu} \mbe \lb \TPhi (\bxt) - \TF (\bxt, \byt) \rb + 2 \seff \lrsx \mbe \norm{\G \TPhi (\bxt)}^2 + \frac{\seff \lrsx}{2} \lp 1 + \seff \lrsx \Lf \frac{\numclients (\selclients-1)}{\selclients (\numclients-1)} \rp \mbe \norm{\sumin w_i \bhxit}^2 \tag{using \eqref{eq_proof:lem:NC_SC_PCP_WOR_phi_f_diff_3a}, \cref{lem:NC_SC_PCP_consensus_error}} \\
    & \quad + \frac{\localvar^2 \seff^2 \lrsxsq \Lf}{2} \frac{\numclients}{\selclients} \lb \sumin \frac{w_i^2 \nai_2^2}{\nai_1^2} + 4 \lp \frac{\numclients - \selclients}{\numclients-1} \max_i w_i + \varscale^2 \max_{i,k} \frac{w_i \aikt}{\nai_1} \rp \lp \lrcxsq + \lrcysq \rp \Lf^2 \sumin w_i \norm{\boldsymbol a_{i,-1}}_2^2 \rb \nn \\
    & \quad + \hetero^2 \seff^2 \lrsxsq \Lf \frac{\numclients}{\selclients} \lp \frac{\numclients - \selclients}{\numclients-1} \max_i w_i + \varscale^2 \max_i \frac{w_i \nai_2^2}{\nai_1^2} \rp \nn \\
    & \quad + \hetero^2 \seff^2 \lrsxsq \Lf \frac{\numclients}{\selclients} 4 \lp \frac{\numclients - \selclients}{\numclients-1} \max_i w_i + \varscale^2 \max_{i,k} \frac{w_i \aikt}{\nai_1} \rp \Lf^2 \TMa \lp \lrcxsq + \lrcysq \rp,
    \label{eq_proof:lem:NC_SC_PCP_WOR_phi_f_diff_7}
\end{align}
}%
where the coefficients are simplified since the choices of learning rates ensures that
\begin{align*}
    \seff \lrsy \Lf \frac{4 \numclients \heteroscale^2}{\selclients} \lp \frac{\numclients - \selclients}{\numclients-1} \max_i w_i + \varscale^2 \max_{i,k} \frac{w_i \aikt}{\nai_1} \rp 4 \Lf^2 \TMa \lrcxsq & \leq 1 \nn \\
    \seff \lrsy \Lf \frac{4 \numclients \heteroscale^2}{\selclients} \lp \frac{\numclients - \selclients}{\numclients-1} \max_i w_i + \varscale^2 \max_i \frac{w_i \nai_2^2}{\nai_1^2} \rp & \leq 1 \nn \\
    \seff \lrsx \frac{4 \numclients \heteroscale^2}{\selclients} \lp \frac{\numclients - \selclients}{\numclients-1} \max_i w_i + \varscale^2 \max_{i,k} \frac{w_i \aikt}{\nai_1} \rp 4 \Lf^2 \TMa \lp 2 \kappa \lrcxsq + \lrcysq \rp & \leq \frac{1}{\mu} \nn \\
    \seff \lrsx \Lf \frac{8 \numclients \heteroscale^2}{\selclients} \lp \frac{\numclients - \selclients}{\numclients-1} \max_i w_i + \varscale^2 \max_i \frac{w_i \nai_2^2}{\nai_1^2} \rp & \leq 1
\end{align*}
We substitute the bounds from \eqref{eq_proof:lem:NC_SC_PCP_WOR_phi_f_diff_6}, \eqref{eq_proof:lem:NC_SC_PCP_WOR_phi_f_diff_7} in \eqref{eq_proof:lem:NC_SC_PCP_phi_f_diff_5b}, and subsequently in \eqref{eq_proof:lem:NC_SC_PCP_WOR_phi_f_diff_5a}, we get
{\small
\begin{align}
    & \mbe \lb \TPhi (\bxtp) - \TF (\bxtp, \bytp) \rb \nn \\
    & \leq (1 - \seff \lrsy \mu) \lb 1 + \frac{3 \seff \lrsx \Lf^2}{\mu} + \frac{5 \seff \lrsx \Lf^2}{2 \mu} \rb \mbe \lb \TPhi (\bxt) - \TF (\bxt, \byt) \rb \nn \\
    & \quad + 4 \seff \lrsy \Lf^2 \heteroscale^2 \lb \frac{\seff \lrsy \numclients}{\selclients} \lp \frac{\numclients - \selclients}{\numclients-1} \max_i w_i + \varscale^2 \max_i \frac{w_i \nai_2^2}{\nai_1^2} \rp + 4 \Lf \TMa \lp 2 \kappa \lrcxsq + \lrcysq \rp \rb \mbe \lb \TPhi (\bxt) - \TF (\bxt, \byt) \rb \nn \\
    & \quad + 4 \seff \lrsy \Lf^2 \heteroscale^2 \lb \frac{\seff^2 \lrsxsq \numclients}{\selclients} \lp \frac{\numclients - \selclients}{\numclients-1} \max_i w_i + \varscale^2 \max_i \frac{w_i \nai_2^2}{\nai_1^2} \rp + 4 \TMa \lrcxsq \rb \mbe \lnr \Gx \TPhi ( \bxt ) \rnr^2 \nn \\
    & \quad - \frac{\seff \lrsy}{2} \lp 1 - \seff \lrsy \Lf \frac{\numclients (\selclients-1)}{\selclients (\numclients-1)} \rp \mbe \norm{\sumin w_i \bhyit}^2 + \seff^3 \lrsxsq \lrsy \Lf^2 \frac{\numclients (\selclients-1)}{\selclients (\numclients-1)} \mbe \lnr \sumin w_i \bhxit \rnr^2 \nn \\
    & \quad + (1 - \seff \lrsy \mu) \lb \frac{13 \seff \lrsx}{8} \mbe \norm{\G \TPhi(\bxt)}^2 + \frac{\seff^2 \lrsxsq}{2} (\Lp + \Lf) \frac{\numclients (\selclients-1)}{\selclients (\numclients-1)} \mbe \norm{\sumin w_i \bhxit}^2 \rb \nn \\
    & \quad + \seff^2 \lrsysq \Lf \lb \frac{\localvar^2 \numclients}{\selclients} \sumin \frac{w_i^2 \nai_2^2}{\nai_1^2} + \frac{2 \hetero^2 \numclients}{\selclients} \lp \frac{\numclients - \selclients}{\numclients-1} \max_i w_i + \varscale^2 \max_i \frac{w_i \nai_2^2}{\nai_1^2} \rp \rb \nn \\
    & \quad + 4 \seff \lrsy \lp \lrcxsq + \lrcysq \rp \Lf^2 \lb \localvar^2 \sumin w_i \norm{\boldsymbol a_{i,-1}}_2^2 + 2 \TMa \hetero^2 \rb \nn \\
    & \quad + (1 - \seff \lrsy \mu) \frac{\localvar^2 \seff^2 \lrsxsq \Lf}{2} \frac{\numclients}{\selclients} \lb \sumin \frac{w_i^2 \nai_2^2}{\nai_1^2} + 4 \lp \frac{\numclients - \selclients}{\numclients-1} \max_i w_i + \varscale^2 \max_{i,k} \frac{w_i \aikt}{\nai_1} \rp \lp \lrcxsq + \lrcysq \rp \Lf^2 \sumin w_i \norm{\boldsymbol a_{i,-1}}_2^2 \rb \nn \\
    & \quad + (1 - \seff \lrsy \mu) \hetero^2 \seff^2 \lrsxsq \Lf \frac{\numclients}{\selclients} \lb \lp \frac{\numclients - \selclients}{\numclients-1} \max_i w_i + \varscale^2 \max_i \frac{w_i \nai_2^2}{\nai_1^2} \rp + 4 \lp \frac{\numclients - \selclients}{\numclients-1} \max_i w_i + \varscale^2 \max_{i,k} \frac{w_i \aikt}{\nai_1} \rp \Lf^2 \TMa \lp \lrcxsq + \lrcysq \rp \rb \nn \\ 
    & \quad + (1 - \seff \lrsy \mu) \frac{5}{2} \seff \lrsx \lp \lrcxsq + \lrcysq \rp \Lf^2 \lb \localvar^2 \sumin w_i \norm{\boldsymbol a_{i,-1}}_2^2 + 2 \TMa \hetero^2 \rb \nn \\
    & \quad + (1 - \seff \lrsy \mu) \frac{\seff^2 \lrsxsq \Lp}{2} \frac{\numclients}{\selclients} \lb \localvar^2 \sumin \frac{w_i^2 \nai_2^2}{\nai_1^2} + \hetero^2 \lp 2 (\max_i w_i) \frac{\numclients - \selclients}{\numclients-1} + 2 \varscale^2 \max_i \frac{w_i \nai_2^2}{\nai_1^2} \rp \rb \label{eq_proof:lem:NC_SC_PCP_WOR_phi_f_diff_8}
\end{align}
}%
Next, we simplify the coefficients of different terms in \eqref{eq_proof:lem:NC_SC_PCP_WOR_phi_f_diff_8}.
\begin{itemize}
    \item Coefficient of $\mbe \lb \TPhi (\bxt) - \TF (\bxt, \byt) \rb$ can be simplified to
    \begin{align*}
        & (1 - \seff \lrsy \mu) \lp 1 + \frac{11 \seff \lrsx \Lf^2}{2 \mu} \rp \\
        & \qquad + 4 \seff \lrsy \Lf^2 \heteroscale^2 \lb \frac{\seff \lrsy \numclients}{\selclients} \lp \frac{\numclients - \selclients}{\numclients-1} \max_i w_i + \varscale^2 \max_i \frac{w_i \nai_2^2}{\nai_1^2} \rp + 4 \Lf \TMa \lp 2 \kappa \lrcxsq + \lrcysq \rp \rb \\
        & \leq 1 - \frac{\seff \lrsy \mu}{4}.
    \end{align*}
    using $\lrsx \leq \frac{\lrsy}{11 \kappa^2}$, $\seff \lrsy \kappa \Lf \heteroscale^2 \frac{\numclients}{\selclients} \max \lcb \varscale^2 \max_i \frac{w_i \nai_2^2}{\nai_1^2}, \frac{\numclients - \selclients}{\numclients-1} \max_i w_i \rcb \leq \frac{1}{64}$, $\kappa \Lf \heteroscale \lrcx \leq \frac{1}{16 \sqrt{2 \TMa}}$ and $\Lf \heteroscale \lrcy \leq \frac{1}{16 \sqrt{\kappa \TMa}}$.
    \item Coefficient of $\mbe \lnr \sumin w_i \bhxit \rnr^2$ can be simplified to
    \begin{align*}
        & \seff^3 \lrsxsq \lrsy \Lf^2 \frac{\numclients (\selclients-1)}{\selclients (\numclients-1)} + \frac{\seff^2 \lrsxsq}{2} (\Lp + \Lf) \frac{\numclients (\selclients-1)}{\selclients (\numclients-1)} \nn \\
        & \leq 2 \seff^2 \lrsxsq \Lp \frac{\numclients (\selclients-1)}{\selclients (\numclients-1)}. \tag{$\because \Lf \leq \Lp, \seff \lrsy \Lf \leq 1$}
    \end{align*}
    \item Coefficient of $\mbe \norm{\G \TPhi(\bxt)}^2$ can be simplified to
    \begin{align*}
        & (1 - \seff \lrsy \mu) \frac{13}{8} \seff \lrsx + 4 \seff \lrsy \Lf^2 \heteroscale^2 \lb \frac{\seff^2 \lrsxsq \numclients}{\selclients} \lp \frac{\numclients - \selclients}{\numclients-1} \max_i w_i + \varscale^2 \max_i \frac{w_i \nai_2^2}{\nai_1^2} \rp + 4 \TMa \lrcxsq \rb \nn \\
        & \leq \frac{\seff \lrsy}{48 \kappa^2}. \tag{$\frac{\lrsx}{\lrsy} \leq \frac{1}{49 \kappa^2}$}
    \end{align*}
    using $\lrsx \leq \frac{\lrsy}{156 \kappa^2}$, $\lrcx \Lf \heteroscale \leq \frac{1}{64 \kappa \sqrt{\TMa}}$ and $\seff \lrsx \Lf \heteroscale \sqrt{ \frac{\numclients}{\selclients}} \max \lcb \frac{\numclients - \selclients}{\numclients-1} \max_i w_i, \varscale \sqrt{\max_i \frac{w_i \nai_2^2}{\nai_1^2}} \rcb \leq \frac{1}{40 \kappa}$.
    \item Coefficient of $\localvar^2$ can be simplified to
    \begin{align*}
        & 2 \seff^2 \lrsxsq \Lf \frac{\numclients}{\selclients} \lp \frac{\numclients - \selclients}{\numclients-1} \max_i w_i + \varscale^2 \max_{i,k} \frac{w_i \aikt}{\nai_1} \rp \lp \lrcxsq + \lrcysq \rp \Lf^2 \sumin w_i \norm{\boldsymbol a_{i,-1}}_2^2 \nn \\
        & \quad + \seff^2 \lrsysq \Lf \frac{\numclients}{\selclients} \sumin \frac{w_i^2 \nai_2^2}{\nai_1^2} + \frac{\seff^2 \lrsxsq (\Lp + \Lf)}{2} \frac{\numclients}{\selclients} \sumin \frac{w_i^2 \nai_2^2}{\nai_1^2} \nn \\
        & \quad + \seff \lp 4 \lrsy + \frac{5}{2} \lrsx \rp \lp \lrcxsq + \lrcysq \rp \Lf^2 \sumin w_i \norm{\boldsymbol a_{i,-1}}_2^2
        \nn \\
        & \leq \frac{3}{2} \seff^2 \lrsysq \Lf \frac{\numclients}{\selclients} \sumin \frac{w_i^2 \nai_2^2}{\nai_1^2} + \frac{9}{2} \seff \lrsy \lp \lrcxsq + \lrcysq \rp \Lf^2 \sumin w_i \norm{\boldsymbol a_{i,-1}}_2^2 \nn \\
        & \quad + 2 \seff^2 \lrsxsq \Lf \frac{\numclients}{\selclients} \lp \frac{\numclients - \selclients}{\numclients-1} \max_i w_i + \varscale^2 \max_{i,k} \frac{w_i \aikt}{\nai_1} \rp \lp \lrcxsq + \lrcysq \rp \Lf^2 \sumin w_i \norm{\boldsymbol a_{i,-1}}_2^2.
    \end{align*}
    \item Coefficient of $\hetero^2$ can be simplified to
    \begin{align*}
        & \frac{2 \seff^2 \numclients}{\selclients} \lp \lrsysq \Lf + \lrsxsq (\Lf + \Lp) \rp \lp \frac{\numclients - \selclients}{\numclients-1} \max_i w_i + \varscale^2 \max_i \frac{w_i \nai_2^2}{\nai_1^2} \rp \nn \\
        & \quad + \seff \Lf^2 \TMa \lp \lrcxsq + \lrcysq \rp \lp 8 \lrsy + 4 \seff \lrsxsq \Lf \frac{\numclients}{\selclients} \lp \frac{\numclients - \selclients}{\numclients-1} \max_i w_i + \varscale^2 \max_{i,k} \frac{w_i \aikt}{\nai_1} \rp + 5 \lrsx  \rp \nn \\
        & \leq \frac{3 \seff^2 \lrsysq \Lf \numclients}{\selclients} \lp \frac{\numclients - \selclients}{\numclients-1} \max_i w_i + \varscale^2 \max_i \frac{w_i \nai_2^2}{\nai_1^2} \rp + 9 \lrsy \seff \Lf^2 \TMa \lp \lrcxsq + \lrcysq \rp \nn \\
        & \quad + 4 \seff^2 \lrsxsq \Lf^3 \TMa \lp \lrcxsq + \lrcysq \rp \frac{\numclients}{\selclients} \lp \frac{\numclients - \selclients}{\numclients-1} \max_i w_i + \varscale^2 \max_{i,k} \frac{w_i \aikt}{\nai_1} \rp. \nn
    \end{align*}
\end{itemize}
Finally, substituting these coefficients in \eqref{eq_proof:lem:NC_SC_PCP_WOR_phi_f_diff_8}, summing over $t = 0, \hdots, T-1$ and rearranging the terms, we get
{\small
\begin{align}
    & \frac{1}{T} \sumtT \mbe \lb \TPhi (\bxt) - \TF (\bxt, \byt) \rb \nn \\
    & \leq \frac{4}{\seff \lrsy \mu} \lb \frac{\TPhi (\bx^{(0)}) - F(\bx^{(0)}, \by^{(0)})}{ T} - \frac{\mbe \lp \TPhi (\bx^{(T)}) - F(\bx^{(T)}, \by^{(T)}) \rp}{T} \rb + \frac{1}{12 \mu \kappa^2} \frac{1}{T} \sumtT \mbe \norm{\G \TPhi(\bxt)}^2 \nn \\
    & \quad + \frac{8 \seff \lrsxsq \Lp}{\lrsy \mu} \frac{\numclients (\selclients-1)}{\selclients (\numclients-1)} \mbe \lnr \sumin w_i \bhxit \rnr^2 + 18 \kappa \Lf \lp \lrcxsq + \lrcysq \rp \lb \localvar^2 \sumin w_i \norm{\boldsymbol a_{i,-1}}_2^2 + 2 \hetero^2 \TMa \rb \nn \\
    & \quad + \frac{8 \seff \lrsxsq \kappa}{\lrsy} \frac{\numclients}{\selclients} \lp \frac{\numclients - \selclients}{\numclients-1} \max_i w_i + \varscale^2 \max_{i,k} \frac{w_i \aikt}{\nai_1} \rp \lp \lrcxsq + \lrcysq \rp \Lf^2 \lb \localvar^2 \sumin w_i \norm{\boldsymbol a_{i,-1}}_2^2 + 2 \hetero^2 \TMa \rb \nn \\
    & \quad + 6 \seff \lrsy \kappa \frac{\numclients}{\selclients} \lb \localvar^2 \sumin \frac{w_i^2 \nai_2^2}{\nai_1^2} + 2 \hetero^2 \lp \frac{\numclients - \selclients}{\numclients-1} \max_i w_i + \varscale^2 \max_i \frac{w_i \nai_2^2}{\nai_1^2} \rp \rb,
\end{align}
}%
which concludes the proof.
\end{proof}

\subsection{Auxiliary Lemmas} \label{app:NC_SC_PCP_aux_lemma}

\begin{lemma}
\label{lem:grad_TF_wrt_x}
If the local client function $f_i(\bx, \cdot)$ satisfy Assumptions \ref{assum:smoothness}, \ref{assum:SC_y} ($\Lf$-smoothness and $\mu$-strong concavity in $\by$), then the function $F$
satisfies
$$\mbe \normb{\Gx \TF (\bx, \by}^2 \leq 2 \mbe \normb{\G \TPhi (\bx)}^2 + \frac{4 \Lf^2}{\mu} \mbe \lb \TPhi (\bx) - \TF (\bx, \by) \rb$$
\end{lemma}

\begin{proof}
\begin{align}
    \mbe \norm{\Gx \TF (\bx, \by}^2 & \leq 2 \mbe \lnr \G \TPhi (\bx) \rnr^2 + 2 \mbe \lnr \Gx F \lp \bx, \by \rp - \G \TPhi (\bx) \rnr^2 \nn \\
    & \leq 2 \mbe \lnr \G \TPhi (\bx) \rnr^2 + 2 \Lf^2 \mbe \lnr \by^*(\bx) - \by \rnr^2 \tag{$\Lf$-smoothness (\cref{assum:smoothness})} \\
    & \leq 2 \mbe \lnr \G \TPhi (\bx) \rnr^2 + \frac{4 \Lf^2}{\mu} \mbe \lb \TPhi (\bx) - \TF (\bx, \by) \rb. \tag{\cref{assum:SC_y}}
\end{align}
\end{proof}

\begin{lemma}
\label{lem:NC_SC_avg_grad_x}
If the local client function $f_i(\bx, \cdot)$ satisfy Assumptions \ref{assum:smoothness}, \ref{assum:bdd_hetero} and \ref{assum:SC_y}, then the iterates $\{ \bxitk, \byitk \}_{i, (t,k)}$ generated by \cref{alg_NC_minimax} satisfy
\begin{align}
    & \sumin \frac{w_i^2}{\nai_1^2} \sumikt [ \aikt ]^2 \mbe \normb{\Gx f_i ( \bxitk, \byitk )}^2 \nn \\
    & \leq 2 \sumin \frac{w_i^2}{\nai_1^2} \sumikt [ \aikt ]^2 \Lf^2 \CExytk (i) + 2 \hetero^2 \lp \max_i \frac{w_i \nai_2^2}{\nai_1^2} \rp \nn \\
    & \quad + 4 \heteroscale^2 \lp \max_i \frac{w_i \nai_2^2}{\nai_1^2} \rp \lb \frac{2 \Lf^2}{\mu} \mbe \lp \TPhi (\bxt) - \TF (\bxt, \byt) \rp + \lnr \Gx \TPhi ( \bxt ) \rnr^2 \rb. \nn
\end{align}
\end{lemma}

\begin{proof}
\begin{align}
    & \sumin \frac{w_i^2}{\nai_1^2} \sumikt [ \aikt ]^2 \mbe \lnr \Gx f_i ( \bxitk, \byitk ) \pm \Gx f_i ( \bxt, \byt ) \rnr^2 \nn \\
    & \leq 2 \sumin \frac{w_i^2}{\nai_1^2} \sumikt [ \aikt ]^2 \Lf^2 \mbe \lb \lnr \bxitk- \bxt \rnr^2 + \lnr \byitk - \byt \rnr^2 \rb + 2 \sumin \frac{w_i^2 \nai_2^2}{\nai_1^2} \mbe \lnr \Gx f_i ( \bxt, \byt ) \rnr^2 \tag{$\Lf$-smoothness} \\
    & \leq 2 \sumin \frac{w_i^2}{\nai_1^2} \sumikt [ \aikt ]^2 \Lf^2 \CExytk (i) + 2 \lp \max_i \frac{w_i \nai_2^2}{\nai_1^2} \rp \lb \heteroscale^2 \lnr \Gx \TF (\bxt, \byt) \rnr^2 + \hetero^2 \rb \tag{\cref{assum:bdd_hetero}} 
\end{align}
Using \cref{lem:grad_TF_wrt_x} gives the result.
\end{proof}

\begin{lemma}
\label{lem:NC_SC_avg_grad_y}
If the local client function $f_i(\bx, \cdot)$ satisfy Assumptions \ref{assum:smoothness}, \ref{assum:bdd_hetero} and \ref{assum:SC_y}, then the iterates $\{ \bxitk, \byitk \}_{i, (t,k)}$ generated by \cref{alg_NC_minimax} satisfy
\begin{align}
    & \sumin \frac{w_i^2}{\nai_1^2} \sumikt [ \aikt ]^2 \mbe \normb{\Gy f_i ( \bxitk, \byitk )}^2 \nn \\
    & \leq 2 \sumin \frac{w_i^2}{\nai_1^2} \sumikt [ \aikt ]^2 \Lf^2 \CExytk (i) + 2 \lp \max_i \frac{w_i \nai_2^2}{\nai_1^2} \rp \lb \hetero^2 + 2 \heteroscale^2 \Lf \mbe \lp \TPhi (\bxt) - \TF (\bxt, \byt) \rp \rb. \nn
\end{align}
\end{lemma}

\begin{proof}
Following closely the proof of \cref{lem:NC_SC_avg_grad_x},
\begin{align}
    & \sumin \frac{w_i^2}{\nai_1^2} \sumikt [ \aikt ]^2 \mbe \lnr \Gy f_i ( \bxitk, \byitk ) \pm \Gy f_i ( \bxt, \byt ) \rnr^2 \nn \\
    & \leq 2 \sumin \frac{w_i^2}{\nai_1^2} \sumikt [ \aikt ]^2 \Lf^2 \mbe \lb \lnr \bxitk- \bxt \rnr^2 + \lnr \byitk - \byt \rnr^2 \rb + 2 \sumin \frac{w_i^2 \nai_2^2}{\nai_1^2} \mbe \lnr \Gy f_i ( \bxt, \byt ) \rnr^2 \tag{$\Lf$-smoothness} \\
    & \leq 2 \sumin \frac{w_i^2}{\nai_1^2} \sumikt [ \aikt ]^2 \Lf^2 \CExytk (i) + 2 \lp \max_i \frac{w_i \nai_2^2}{\nai_1^2} \rp \lb \heteroscale^2 \lnr \Gy \TF (\bxt, \byt) \rnr^2 + \hetero^2 \rb \tag{\cref{assum:bdd_hetero}} \\
    & \leq 2 \sumin \frac{w_i^2}{\nai_1^2} \sumikt [ \aikt ]^2 \Lf^2 \CExytk (i) + 2 \lp \max_i \frac{w_i \nai_2^2}{\nai_1^2} \rp \lb \hetero^2 + 2 \heteroscale^2 \Lf \mbe \lp \TPhi (\bxt) - \TF (\bxt, \byt) \rp \rb. \nn
\end{align}
where the final inequality follows from smoothness and concavity of $F$ in $\by$.
\end{proof}

\subsection{Convergence under Polyak {\L}ojasiewicz (PL) Condition}
\label{app:NC_PL_PCP_result}
In case the global function satisfies \cref{assum:PL_y}, the results in this section follow with minor modifications. The crucial difference is that \cref{lem:smooth_convex} no longer holds. \cref{lem:NC_SC_PCP_wtd_dir_norm_sq} and \cref{lem:NC_SC_PCP_Phi_decay_one_iter} follow exactly. The statement of \cref{lem:NC_SC_PCP_consensus_error} needs some modification, since we use \cref{lem:smooth_convex} in the proof.

\begin{lemma}
\label{lem:NC_PL_PCP_consensus_error}
Suppose the local loss functions $\{ f_i \}$ satisfy Assumptions \ref{assum:smoothness}, \ref{assum:bdd_hetero}, \ref{assum:PL_y}, and the stochastic oracles for the local functions satisfy \cref{assum:bdd_var}. 
Under the conditions of \cref{lem:NC_PL_PCP_consensus_error}, the iterates $\{ \bxit, \byit \}$ generated by \cref{alg_NC_minimax} satisfy
\begin{equation}
    \begin{aligned}
        & \Lf^2 \sumin \frac{p_i}{\nai_1} \sumikt \aikt \CExytk (i) \leq 2 \lp \lrcxsq + \lrcysq \rp \Lf^2 \localvar^2 \sumin p_i \norm{\boldsymbol a_{i,-1}}_2^2 + 4 \Lf^2 \TMa \lp \lrcxsq + \lrcysq \rp \hetero^2 \nn \\
        & \qquad + 8 \Lf^2 \TMa \heteroscale^2 \lrcxsq \mbe \lnr \G \TPhi (\bxt) \rnr^2 + 8 {\color{blue}\kappa} \Lf^3 \TMa \heteroscale^2 \lp 2 \lrcxsq + \lrcysq \rp \mbe \lb \TPhi (\bxt) - \TF (\bxt, \byt) \rb,
    \end{aligned}
\end{equation}
where $\TMa \triangleq \max_i \lp \norm{\boldsymbol a_{i,-1}}_1^2 + \varscale^2 \norm{\boldsymbol a_{i,-1}}_2^2 \rp$.
\end{lemma}

The bound in \cref{lem:NC_SC_avg_grad_y} also changes to
\begin{align}
    & \sumin \frac{w_i^2}{\nai_1^2} \sumikt [ \aikt ]^2 \mbe \normb{\Gy f_i ( \bxitk, \byitk )}^2 \nn \\
    & \leq 2 \sumin \frac{w_i^2}{\nai_1^2} \sumikt [ \aikt ]^2 \Lf^2 \CExytk (i) + 2 \lp \max_i \frac{p_i \nai_2^2}{\nai_1^2} \rp \lb \hetero^2 + 2 \heteroscale^2 {\color{blue}\kappa} \Lf \mbe \lp \TPhi (\bxt) - \TF (\bxt, \byt) \rp \rb. \nn
\end{align}

The same bound in \cref{lem:NC_SC_PCP_phi_f_diff} holds, but with more stringent conditions on learning rates, namely $\lrcy \Lf \heteroscale \leq \frac{1}{16 \kappa \sqrt{\TMa}}$ and $\seff \lrsy \kappa \Lf \heteroscale^2 \frac{1}{P} \max \lcb \kappa \varscale^2 \max_i \frac{\nai_2^2}{\nai_1^2}, 1 \rcb \leq \frac{1}{64}$. Consequently, the bounds in \cref{thm:NC_SC} hold, under slightly more stringent conditions on the learning rates.





%% file: Appendix/5_NC_C_PCP.tex
\section{Convergence of \fedsgdaplus \ for Nonconvex Concave Functions (\texorpdfstring{\cref{thm:NC_C}}{Theorem 2})} \label{app:NC_C}

We organize this section as follows. First, in \cref{sec:NC_C_PCP_int_results} we present some intermediate results, which we use in the proof of \cref{thm:NC_C}. Next, in \cref{sec:NC_C_PCP_thm_proof}, we present the proof of \cref{thm:NC_C}, which is followed by the proofs of the intermediate results in \cref{sec:NC_C_PCP_int_results_proofs}.
Finally, we discuss the extension of our results to nonconvex-one-point-concave functions in \cref{app:NC_1PC}.

The problem we solve is
\begin{align*}
    \min_{\bx} \max_{\by} \lcb \TF(\bx, \by) \triangleq \sumin \wi f_i(\bx, \by) \rcb.
\end{align*}
We define $\TPhi (\bx) \triangleq \max_{\by} \TF(\bx, \by)$ and $\Tby^* (\bx) \in \argmax_{\by} \TF(\bx, \by)$. Since $\TF(\bx, \cdot)$ is no longer strongly concave, $\by^*(\bx)$ need not be unique. In \cref{alg_NC_minimax}-\fedsgdaplus \ , the client updates are given by
\begin{equation}
    \begin{aligned}
        & \bxitk = \bxt - \lrcx \sumijk \aijk \Gx f_i (\bxitj, \byitj; \xiitj), \\
        & \byitk = \byt + \lrcy \sumijk \aijk \Gy f_i (\Hbxs, \byitj; \xiitj),
    \end{aligned}
    \label{eq:client_update_alg_NC_C_minimax_PCP}
\end{equation}
where $1 \leq k \leq \sync_i$. The server updates are given by
\begin{equation}
    \bxtp = \bxt - \seff \lrsx \bdxt, \qquad \bytp = \byt + \seff \lrsy \bdyt, \label{eq:server_update_alg_NC_C_minimax_PCP}
\end{equation}
where $\bdxt, \bdyt$ are defined in \eqref{eq:server_agg_WOR}. The normalized (stochastic) gradient vectors are defined as
\begin{equation}
    \begin{aligned}
        & \bdxit = \frac{1}{\nai_1} \sumikt \aikt \Gx f_i \lp \bxitk, \byitk; \xiitk \rp; \quad \bhxit = \frac{1}{\nai_1} \sumikt \aikt \Gx f_i \lp \bxitk, \byitk \rp, \\
        & \bdyit = \frac{1}{\nai_1} \sumikt \aikt \Gy f_i \lp \Hbxs, \byitk; \xiitk \rp; \quad \bhyit = \frac{1}{\nai_1} \sumikt \aikt \Gy f_i \lp \Hbxs, \byitk \rp.
    \end{aligned}
    \label{eq:dir_alg_NC_minimax_PCP}
\end{equation}

\subsection{Intermediate Lemmas} \label{sec:NC_C_PCP_int_results}
As discussed in \cref{sec:NC_C}, we analyze the convergence of the smoothed envelope function $\TPhi_{1/2\Lf}$. We begin with a bound on the one-step decay of this function.

\begin{lemma}[One-step decay of Smoothed-Envelope]
\label{lem:NC_C_PCP_Phi_smooth_decay_one_iter}
Suppose the local loss functions $\{ f_i \}$ satisfy Assumptions \ref{assum:smoothness}, \ref{assum:bdd_var}, \ref{assum:concavity}, and \ref{assum:Lips_cont_x}.
Then, the iterates generated by \cref{alg_NC_minimax}-\fedsgdaplus \  satisfy
\begin{align}
    & \mbe \lb \TPhi_{1/2\Lf} (\bxtp) \rb \leq \mbe \lb \TPhi_{1/2\Lf} (\bxt) \rb + \seff^2 \lrsxsq \Lf \frac{\numclients}{\selclients} \lb \sumin \frac{w_i^2 \nai_2^2}{\nai_1^2} \lp \localvar^2 + \varscale^2 G_{\bx}^2 \rp + G_{\bx}^2 \lp \frac{\selclients-1}{\numclients-1} + \frac{\numclients - \selclients}{\numclients-1} \sumin w_i^2 \rp \rb \nn \\
    & \quad + 2 \seff \lrsx \lcb \Lf^2 \sumin \frac{w_i}{\nai_1} \sumikt \aikt  \CExytk (i) + \Lf \mbe \lb \TPhi(\bxt) - \TF(\bxt, \byt) \rb \rcb - \frac{\seff \lrsx}{8} \mbe \norm{\G \TPhi_{1/2\Lf} (\bxt)}^2, \nn
\end{align}
where $\CExytk (i) = \mbe \lb \| \bxitk - \bxt \|^2 + \| \byitk - \byt \|^2 \rb$ is the drift of client $i \in [n]$, at the $k$-th local step of epoch $t$.
\end{lemma}

Between two successive synchronization time instants (for example, $t, t+1$), the clients drift apart due to local descent/ascent steps, resulting in the $\{ \CExytk (i) \}_{i, k}$ terms. Also, $\mbe \lb \TPhi(\bxt) - \TF(\bxt, \byt) \rb$ quantifies the error of the inner maximization. In the subsequent lemmas, we bound both these error terms.

\begin{lemma}[Consensus Error]
\label{lem:NC_C_PCP_consensus_error}
Suppose the local loss functions $\{ f_i \}$ satisfy Assumptions \ref{assum:smoothness}, \ref{assum:bdd_hetero}, \ref{assum:concavity}, and \ref{assum:Lips_cont_x}. The stochastic oracles for the local
functions satisfy \cref{assum:bdd_var}.
Further, in \cref{alg_NC_minimax}-\fedsgdaplus, we choose the client learning rate $\lrcy$ such that $\lrcy \leq \frac{1}{2 \Lf ( \max_i \nai_1 ) \sqrt{2 \max \{1, \varscale^2 \}}}$.
Then, the iterates $\{ \bxit, \byit \}$ generated by \cref{alg_NC_minimax}-\fedsgdaplus \  satisfy
\begin{align}
    \Lf^2 \sumin \frac{w_i}{\nai_1} \sumikt \aikt \CExytk (i) & \leq 2 \lp \lrcxsq + \lrcysq \rp \Lf^2 \localvar^2 \sumin w_i \norm{\mbf a_{i,-1}}_2^2 + 4 \Lf^2 \TMa \lp \lrcxsq G_{\bx}^2 + \lrcysq \hetero^2 \rp \nn \\
    & \quad + 8 \lrcysq \Lf^3 \TMa \heteroscale^2 \mbe \lb \TPhi (\Hbxs) - \TF(\Hbxs, \byt) \rb, \nn
\end{align}
where $\TMa \triangleq \max_i ( \norm{\mbf a_{i,-1}}_1^2 + \varscale^2 \norm{\mbf a_{i,-1}}_2^2 )$. 
\end{lemma}

Note that consensus error depends on the difference $\mbe [ \TPhi (\Hbxs) - \TF(\Hbxs, \byt) ]$. This is different from the term $\mbe [ \TPhi(\bxt) - \TF(\bxt, \byt) ]$ in \cref{lem:NC_C_PCP_Phi_smooth_decay_one_iter}. Since in \cref{alg_NC_minimax}-\fedsgdaplus \ , the $\bx$-component stays fixed at $\Hbxs$ for $S$ communication rounds while updating $\byitk$, the difference 
\begin{align*}
    \sum_{t = kS}^{(k+1) S - 1} \mbe \lb \TPhi(\Hbxs) - \TF(\Hbxs, \byt) \rb
\end{align*}
can be interpreted as the optimization error, when maximizing the concave function $F(\Hbxs, \cdot)$ over $S$ communication rounds. Next, we bound this error. The following result essentially extends the analysis of FedNova (\cite{joshi20fednova_neurips}) to concave maximization (analogously, convex minimization) problems. We also generalize the corresponding analyses in \cite{khaled20localSGD_aistats, koloskova20unified_localSGD_icml} to heterogeneous local updates.

\begin{lemma}[Local SG updates for Concave Maximization]
\label{lem:NC_C_PCP_local_SGA_concave}
Suppose the local functions satisfy Assumptions \ref{assum:smoothness}, \ref{assum:bdd_var}, \ref{assum:bdd_hetero} and \ref{assum:concavity}. Further, let $\norm{\byt}^2 \leq R$ for all $t$. We run \cref{alg_NC_minimax}-\fedsgdaplus \  with client step-size $\lrcy$ such that $64 \lrcysq \TMa \Lf^2 \heteroscale^2 \frac{\numclients}{\selclients} \leq 1$. Further, the server step-size $\lrsy$ satisfies
\begin{align*}
    2 \seff \lrsy \Lf \frac{\numclients}{\selclients} \max \{ \heteroscale^2, 1 \} \max \lcb \varscale^2 \max_i \frac{w_i \nai_2^2}{\nai_1^2}, \frac{\numclients - \selclients}{\numclients-1} \max_i w_i \rcb & \leq \frac{1}{8}, \\
    2 \seff \lrsy \Lf \frac{\numclients}{\selclients} \max \lcb \frac{\selclients - 1}{\numclients-1}, \varscale^2 \max_{i,k} \frac{w_i \aikt}{\nai_1} \rcb & \leq \frac{1}{8}.
\end{align*}
Then the iterates generated by \cref{alg_NC_minimax}-\fedsgdaplus \  satisfy
\begin{align}
    & \frac{1}{S} \sum_{t = sS}^{(s+1) S - 1} \mbe \lb \TPhi(\Hbxs) - \TF(\Hbxs, \byt) \rb \nn \\
    & \leq \frac{4 R}{\seff \lrsy S} + \seff \lrsy \frac{\numclients}{\selclients} \lb \localvar^2 \sumin \frac{w_i^{{2}} \nai_2^2}{\nai_1^2} + 2 \hetero^2 \lp \frac{\numclients - \selclients}{\numclients-1} \max_i w_i + \varscale^2 \max_i \frac{w_i \nai_2^2}{\nai_1^2} \rp \rb \nn \\
    & \quad + 4 \Lf ( \lrcxsq + \lrcysq) \lb \localvar^2 \sumin w_i \norm{\mbf a_{i,-1}}_2^2 + 2 \TMa (G_\bx^2 + \hetero^2) \rb, \nn
\end{align}
where $\TMa \triangleq \max_i \lp \norm{\mbf a_{i,-1}}_1^2 + \varscale^2 \norm{\mbf a_{i,-1}}_2^2 \rp$.
\end{lemma}

\begin{remark}
\label{rem:NC_C_localSGD}
It is worth noting that the proof of \cref{lem:NC_C_PCP_local_SGA_concave} does not require global concavity of local functions. Rather, given $\bx$, we only need concavity of local functions $\{ f_i \}$ at some point $\by^* (\bx)$. This is precisely the one-point-concavity assumption (\cref{assum:1pc_y}) discussed earlier in \cite{mahdavi21localSGDA_aistats, sharma22FedMinimax_ICML}. Therefore, \cref{lem:NC_C_PCP_local_SGA_concave} for a much larger class of functions. Further, the bound in \cref{lem:NC_C_PCP_local_SGA_concave} improves the corresponding bounds derived in existing work. As we discuss in \cref{app:NC_1PC}, this helps us achieve improve complexity results for nonconvex-one-point-concave (NC-1PC) functions.
\end{remark}

Next, we bound the difference $\mbe \lb \TPhi(\bxt) - \TF(\bxt, \byt) \rb$.

\begin{lemma}
\label{lem:NC_C_PCP_Phi_f_diff}
Suppose the local functions satisfy Assumptions \ref{assum:smoothness}, \ref{assum:bdd_var}, \ref{assum:bdd_hetero}, \ref{assum:Lips_cont_x}.
Then the iterates generated by \cref{alg_NC_minimax}-\fedsgdaplus \  satisfy
\begin{align}
    & \avgtT \mbe \lb \TPhi(\bxt) - \TF(\bxt, \byt) \rb \leq \frac{1}{T} \sum_{s=0}^{T/S-1} \sum_{t=sS}^{(s+1)S-1} \mbe \lb \TPhi(\Hbxs) - \TF(\Hbxs, \byt) \rb \nn \\
    & \qquad + 2 \seff \lrsx G_{\bx} (S-1) \sqrt{\frac{\numclients}{\selclients}} \sqrt{\sumin \frac{w_i^2 \nai_2^2}{\nai_1^2} \lp \localvar^2 + \varscale^2 G_{\bx}^2 \rp + G_{\bx}^2 \lp \frac{\selclients-1}{\numclients-1} + \frac{\numclients - \selclients}{\numclients-1} \sumin w_i^2 \rp}. \nn
\end{align}
\end{lemma}

\subsection{Proof of \texorpdfstring{\cref{thm:NC_C}}{Theorem 2}}
\label{sec:NC_C_PCP_thm_proof}
For the sake of completeness, we first state the full statement of \cref{thm:NC_C} here.

\begin{theorem}
\label{thm:NC_C_appendix}
Suppose the local loss functions $\{ f_i \}$ satisfy Assumptions \ref{assum:smoothness}, \ref{assum:bdd_var}, \ref{assum:bdd_hetero}, \ref{assum:concavity}, \ref{assum:Lips_cont_x}. Further, let $\norm{\byt}^2 \leq R$ for all $t$. We run \cref{alg_NC_minimax}-\fedsgdaplus \  with client step-size $\lrcy$ such that $64 \lrcysq \TMa \Lf^2 \max \{ \heteroscale^2 \frac{\numclients}{\selclients}, 1 \} \leq 1$. Further, the server step-size $\lrsy$ satisfies
\begin{align*}
    2 \seff \lrsy \Lf \frac{\numclients}{\selclients} \max \{ \heteroscale^2, 1 \} \max \lcb \varscale^2 \max_i \frac{w_i \nai_2^2}{\nai_1^2}, \frac{\numclients - \selclients}{\numclients-1} \max_i w_i \rcb & \leq \frac{1}{8}, \\
    2 \seff \lrsy \Lf \frac{\numclients}{\selclients} \max \lcb \frac{\selclients - 1}{\numclients-1}, \varscale^2 \max_{i,k} \frac{w_i \aikt}{\nai_1} \rcb & \leq \frac{1}{8}.
\end{align*}
Then the iterates generated by \cref{alg_NC_minimax}-\fedsgdaplus \  satisfy
\begin{equation}
    \begin{aligned}
        & \avgtT \mbe \norm{\G \TPhi_{1/2\Lf} (\bxt)}^2 \leq \mco \lp \frac{\bar{\Delta}_{\TPhi}}{\seff \lrsx T} + \seff \lrsx \Lf \lb \frac{\Aw}{\selclients \seff} \lp \localvar^2 + \varscale^2 G_{\bx}^2 \rp + G_{\bx}^2 \lp \frac{\numclients (\selclients-1)}{\selclients (\numclients-1)} + \Fw \rp \rb \rp \\
        & \quad + \mco \lp \seff \lrsx \Lf G_{\bx} (S-1) \sqrt{ \frac{\Aw}{\selclients \seff} \lp \localvar^2 + \varscale^2 G_{\bx}^2 \rp + G_{\bx}^2 \lp \frac{\numclients (\selclients-1)}{\selclients (\numclients-1)} + \Fw \rp} \rp \\
        & \quad + \mco \lp \frac{\Lf R}{\seff \lrsy S} + \frac{\lrsy \Lf}{\selclients} \lp \Aw \localvar^2 + \hetero^2 \lp \seff \frac{\numclients - \selclients}{\numclients-1} \Ew + \Bw \varscale^2 \rp \rp \rp + \mco \lp \lp \lrcxsq + \lrcysq \rp \Lf^2 \lb \Cw \localvar^2 + D \lp G_{\bx}^2 + \hetero^2 \rp \rb \rp,
    \end{aligned}
    \label{eq:thm:NC_C_1}
\end{equation}
where $\TPhi_{1/2\Lf}(\bx) \triangleq \min_{\bx'} \TPhi (\bx') + \Lf \norm{\bx' - \bx}^2$ is the envelope function, $\bar{\Delta}_{\TPhi} \triangleq \TPhi_{1/2\Lf} (\bx_0) - \min_\bx \TPhi_{1/2\Lf} (\bx)$, $\Aw \triangleq n \seff \sumin \frac{w_i^2 \nai_2^2}{\nai_1^2}$, 
$\Bw \triangleq n \seff \max_i \frac{w_i \nai_2^2}{\nai_1^2}$, 
$\Ew \triangleq \numclients \max_i w_i$,
$\Cw \triangleq \sumin w_i ( \nai_2^2 - [\alpha^{(t,\sync_i-1)}_{i}]^2 )$, and 
$D \triangleq \max_i ( \norm{\mbf a_{i,-1}}_1^2 + \varscale^2\norm{\mbf a_{i,-1}}_2^2 )$, $\Fw \triangleq \frac{\numclients (\numclients - \selclients)}{\selclients (\numclients-1)} \sumin w_i^2$.
With the following parameter values:
$$\lrcx = \lrcy = \Theta \lp \frac{1}{\Lf \bar{\sync} T^{3/8}} \rp, \lrsx = \Theta \lp \frac{P^{1/4}}{(\seff T)^{3/4}} \rp, \quad \lrsy = \Theta \lp \frac{P^{3/4}}{(\seff T)^{1/4}} \rp, \quad S = \Theta \lp \sqrt{\frac{T}{\seff P}} \rp,$$
where $\bar{\sync} = \frac{1}{\numclients} \sumin \sync_i$, we can further simplify to
\begin{align*}
    & \avgtT \mbe \norm{\G \TPhi_{1/2\Lf} (\bxt)}^2 \nn \\
    & \leq \underbrace{\mco \lp \frac{(\bar{\sync}/\seff)^{1/4}}{(\bar{\sync} \selclients T)^{1/4}} \rp + \mco \lp \frac{(\seff \selclients)^{1/4}}{T^{3/4}} \rp}_{\text{Error with full synchronization}} + \underbrace{\mco \lp \lp \frac{\numclients - \selclients}{\numclients-1} \cdot \frac{\Ew}{\selclients T} \rp^{1/4} \rp}_{\text{Partial participation error}} + \underbrace{\mco \lp \frac{\Cw \localvar^2 + D (G_{\bx}^2 + \hetero^2)}{\bar{\sync}^2 T^{3/4}} \rp}_{\text{Error due to local updates}}.
\end{align*}
\end{theorem}

\begin{proof}
We sum the bound in \cref{lem:NC_C_PCP_Phi_smooth_decay_one_iter} over $t = 0$ to $T-1$ and rearrange the terms to get
\begin{align}
    & \avgtT \mbe \norm{\G \TPhi_{1/2\Lf} (\bxt)}^2 \nn \\
    & \leq \frac{8}{\seff \lrsx T} \sumtT \mbe \lb \TPhi_{1/2\Lf} (\bxt) - \TPhi_{1/2\Lf} (\bxtp) \rb \nn \\
    & \quad + 8 \seff \lrsx \Lf \frac{\numclients}{\selclients} \lb \sumin \frac{w_i^2 \nai_2^2}{\nai_1^2} \lp \localvar^2 + \varscale^2 G_{\bx}^2 \rp + G_{\bx}^2 \lp \frac{\selclients-1}{\numclients-1} + \frac{\numclients - \selclients}{\numclients-1} \sumin w_i^2 \rp \rb \nn \\
    & \quad + 16 \Lf \avgtT \mbe \lb \TPhi(\bxt) - \TF(\bxt, \byt) \rb + \frac{16}{T} \sum_{s=0}^{T/S-1} \sum_{t=sS}^{(s+1)S-1} \Lf^2 \sumin \frac{w_i}{\nai_1} \sumikt \aikt  \CExytk (i) \nn \\
    & \leq \frac{8 \lb \TPhi_{1/2\Lf} (\bx^{(0)}) - \TPhi_{1/2\Lf} (\bx^{(T)}) \rb}{\seff \lrsx T} + 8 \seff \lrsx \Lf \frac{\numclients}{\selclients} \lb \sumin \frac{w_i^2 \nai_2^2}{\nai_1^2} \lp \localvar^2 + \varscale^2 G_{\bx}^2 \rp + G_{\bx}^2 \lp \frac{\selclients-1}{\numclients-1} + \frac{\numclients - \selclients}{\numclients-1} \sumin w_i^2 \rp \rb \nn \\
    & \quad + 32 \lp \lrcxsq + \lrcysq \rp \Lf^2 \localvar^2 \sumin w_i \norm{\mbf a_{i,-1}}_2^2 + 64 \Lf^2 \TMa \lp \lrcxsq G_{\bx}^2 + \lrcysq \hetero^2 \rp \tag{From \cref{lem:NC_C_PCP_consensus_error}} \\
    & \quad + 128 \lrcysq \Lf^3 \TMa \heteroscale^2 \frac{1}{T} \sum_{s=0}^{T/S-1} \sum_{t=sS}^{(s+1)S-1} \mbe \lb \TPhi (\Hbxs) - \TF(\Hbxs, \byt) \rb + 16 \Lf \avgtT \mbe \lb \TPhi(\bxt) - \TF(\bxt, \byt) \rb \nn \\
    & \leq \frac{8 \bar{\Delta}_{\TPhi}}{\seff \lrsx T} + 8 \seff \lrsx \Lf \frac{\numclients}{\selclients} \lb \sumin \frac{w_i^2 \nai_2^2}{\nai_1^2} \lp \localvar^2 + \varscale^2 G_{\bx}^2 \rp + G_{\bx}^2 \lp \frac{\selclients-1}{\numclients-1} + \frac{\numclients - \selclients}{\numclients-1} \sumin w_i^2 \rp \rb \tag{where {$\bar{\Delta}_{\TPhi} \triangleq \TPhi_{1/2\Lf} (\bx_0) - \min_\bx \TPhi_{1/2\Lf} (\bx)$}} \\
    & \quad + 32 \lp \lrcxsq + \lrcysq \rp \Lf^2 \lb \localvar^2 \sumin w_i \norm{\mbf a_{i,-1}}_2^2 + 2 \TMa \lp G_{\bx}^2 + \hetero^2 \rp \rb \nn \\
    & \quad + 32 \seff \lrsx \Lf G_{\bx} (S-1) \sqrt{\frac{\numclients}{\selclients}} \sqrt{ \sumin \frac{w_i^2 \nai_2^2}{\nai_1^2} \lp \localvar^2 + \varscale^2 G_{\bx}^2 \rp + G_{\bx}^2 \lp \frac{\selclients-1}{\numclients-1} + \frac{\numclients - \selclients}{\numclients-1} \sumin w_i^2 \rp} \tag{From \cref{lem:NC_C_PCP_Phi_f_diff}} \\
    & \quad + 18 \Lf \lb \frac{4 R}{\seff \lrsy S} + \seff \lrsy \frac{\numclients}{\selclients} \lp \localvar^2 \sumin \frac{w_i^{{2}} \nai_2^2}{\nai_1^2} + 2 \hetero^2 \lp \frac{\numclients - \selclients}{\numclients-1} \max_i w_i + \varscale^2 \max_i \frac{w_i \nai_2^2}{\nai_1^2} \rp \rp \rb \tag{From \cref{lem:NC_C_PCP_local_SGA_concave}; using $A_m \leq \min \{ \frac{1}{2}, \frac{1}{16 \heteroscale^2} \}$} \\
    & \quad + 72 \lrcysq \Lf^2 \lb \localvar^2 \sumin w_i \norm{\mbf a_{i,-1}}_2^2 + 2 \hetero^2 \TMa \rb.
\end{align}
We can simplify the notation using the constants $\Aw \triangleq \numclients \seff \sumin \frac{w_i^2 \nai_2^2}{\nai_1^2}$, $\Bw \triangleq \seff \numclients \lp \max_i \frac{w_i \nai_2^2}{\nai_1^2} \rp$, $\Ew \triangleq \numclients \max_i w_i$, $\Cw \triangleq \sumin w_i \norm{\mbf a_{i,-1}}_2^2$, $D \triangleq \TMa$, $\Fw \triangleq \frac{\numclients}{\selclients} \sumin w_i^2$
and drop the numerical constants, for simplicity, to get
\begin{align}
    & \avgtT \mbe \norm{\G \TPhi_{1/2\Lf} (\bxt)}^2 \nn \\
    & \lesssim \frac{\bar{\Delta}_{\TPhi}}{\seff \lrsx T} + \seff \lrsx \Lf \lb \frac{\Aw}{\selclients \seff} \lp \localvar^2 + \varscale^2 G_{\bx}^2 \rp + G_{\bx}^2 \lp \frac{\numclients (\selclients-1)}{\selclients (\numclients-1)} + \frac{(\numclients - \selclients)}{(\numclients-1)} \Fw \rp \rb \nn \\
    & \quad + \seff \lrsx \Lf G_{\bx} (S-1) \sqrt{ \frac{\Aw}{\selclients \seff} \lp \localvar^2 + \varscale^2 G_{\bx}^2 \rp + G_{\bx}^2 \lp \frac{\numclients (\selclients-1)}{\selclients (\numclients-1)} + \frac{(\numclients - \selclients)}{(\numclients-1)} \Fw \rp} \nn \\
    & \quad + \frac{\Lf R}{\seff \lrsy S} + \frac{\lrsy \Lf}{\selclients} \lp \Aw \localvar^2 + \hetero^2 \lp \seff \frac{\numclients - \selclients}{\numclients-1} \Ew + \Bw \varscale^2 \rp \rp + \lp \lrcxsq + \lrcysq \rp \Lf^2 \lb \Cw \localvar^2 + D \lp G_{\bx}^2 + \hetero^2 \rp \rb \nn \\
    & = \frac{\bar{\Delta}_{\TPhi}}{\seff \lrsx T} + \lrsx \Lf \mc I_1^2 + \frac{\lrsy \Lf \mc I_2}{P} + \Lf \lb \seff \lrsx G_{\bx} (S-1) \mc I_1 + \frac{R}{\seff \lrsy S} \rb \nn \\
    & \quad + \lp \lrcxsq + \lrcysq \rp \Lf^2 \lb \Cw \localvar^2 + D (G_{\bx}^2 + \hetero^2) \rb, \label{eq_proof:thm_NC_C_PCP_1}
\end{align}
where in \eqref{eq_proof:thm_NC_C_PCP_1}, to simplify notation, we have defined $\mc I_1 \triangleq \sqrt{ \frac{\Aw}{\selclients \seff} \lp \localvar^2 + \varscale^2 G_{\bx}^2 \rp + G_{\bx}^2 \lp \frac{\numclients (\selclients-1)}{\selclients (\numclients-1)} + \frac{(\numclients - \selclients)}{(\numclients-1)} \Fw \rp}$, $\mc I_2 \triangleq \Aw \localvar^2 + (\Bw \varscale^2 + \seff \frac{\numclients - \selclients}{\numclients-1} \Ew) \hetero^2$.

Next, we optimize the algorithm parameters $S, \lrsx, \lrsy, \lrcy, \lrcy$ to achieve a tight bound on \eqref{eq_proof:thm_NC_C_PCP_1}. If $R = 0$, we let $S=1$. Else, let $S = \sqrt{\frac{R}{\seff^2 \lrsx \lrsy G_{\bx} \mc I_1}}$. Substituting this in \eqref{eq_proof:thm_NC_C_PCP_1}, we get
\begin{align}
    \avgtT \mbe \norm{\G \TPhi_{1/2\Lf} (\bxt)}^2 & \lesssim \frac{\bar{\Delta}_{\TPhi}}{\seff \lrsx T} + \seff \lrsx \Lf \mc I_1^2 + \frac{\lrsy \Lf \mc I_2}{P} + \Lf \sqrt{\frac{R \lrsx G_{\bx} \mc I_1}{\lrsy}} \nn \\
    & \quad + \lp \lrcxsq + \lrcysq \rp \Lf^2 \lb \Cw \localvar^2 + D (G_{\bx}^2 + \hetero^2) \rb, \label{eq_proof:thm_NC_C_PCP_2}
\end{align}
Next, we focus on the terms in \eqref{eq_proof:thm_NC_C_PCP_2} containing $\lrsy$: $\Lf \lb \frac{\lrsy \mc I_2}{P} + \sqrt{\frac{R \lrsx G_{\bx} \mc I_1}{\lrsy}} \rb$. To optimize these, we choose $\lrsy = \lp \frac{\selclients}{2 \mc I_2} \rp^{2/3} \lp R \lrsx G_{\bx} \mc I_1 \rp^{1/3}$. Substituting in \eqref{eq_proof:thm_NC_C_PCP_2}, we get
\begin{align}
    \avgtT \mbe \norm{\G \TPhi_{1/2\Lf} (\bxt)}^2 & \lesssim \frac{\bar{\Delta}_{\TPhi}}{\seff \lrsx T} + \seff \lrsx \Lf \mc I_1^2 + \Lf \lp \frac{\mc I_2}{\selclients} R \lrsx G_{\bx} \mc I_1 \rp^{1/3} \nn \\
    & \quad + \lp \lrcxsq + \lrcysq \rp \Lf^2 \lb \Cw \localvar^2 + D (G_{\bx}^2 + \hetero^2) \rb,
    \label{eq_proof:thm_NC_C_PCP_3}
\end{align}
Finally, we focus on the terms in \eqref{eq_proof:thm_NC_C_PCP_3} containing $\lrsx$: $\frac{\bar{\Delta}_{\TPhi}}{\seff \lrsx T} + \Lf \lp \frac{\mc I_2}{P} R \lrsx G_{\bx} \mc I_1 \rp^{1/3}$. We ignore the higher order linear term. With $\lrsx = \lp \frac{3 \bar{\Delta}_{\TPhi}}{\seff \Lf T} \rp^{3/4} \lp \frac{\mc I_2}{P} R G_{\bx} \mc I_1 \rp^{-1/4}$, 
and absorbing numerical constants inside $\mco (\cdot)$ we get,
\begin{align}
    & \avgtT \mbe \norm{\G \TPhi_{1/2\Lf} (\bxt)}^2 \lesssim \frac{\lp \mc I_1 \mc I_2 R G_{\bx} \rp^{1/4}}{(\seff \selclients T)^{1/4}} + \frac{(\seff \selclients)^{1/4}}{T^{3/4}} \lp \mc I_1 \mc I_2 R G_{\bx} \rp^{-1/4} \mc I_1^2 \nn \\
    & \qquad + \lp \lrcxsq + \lrcysq \rp \Lf^2 \lb \Cw \localvar^2 + D (G_{\bx}^2 + \hetero^2) \rb, \nn \\
    &= \mco \lp \frac{\lp \mc I_1 \mc I_2 \rp^{1/4}}{(\seff \selclients T)^{1/4}} \rp + \mco \lp \frac{(\seff \selclients)^{1/4}}{T^{3/4}} \mc I_2^{3/4} \rp + \mco \lp \lp \lrcxsq + \lrcysq \rp \Lf^2 \lb \Cw \localvar^2 + D (G_{\bx}^2 + \hetero^2) \rb \rp, \nn \\
    & \leq \mco \lp \frac{(\bar{\sync}/\seff)^{1/4}}{(\bar{\sync} \selclients T)^{1/4}} \rp + \mco \lp \frac{\lp \frac{\numclients (\numclients - \selclients)}{\numclients-1} \max_i w_i \rp^{1/4} \hetero}{(\selclients T)^{1/4}} \rp + \mco \lp \frac{(\seff \selclients)^{1/4}}{T^{3/4}} \rp \nn \\
    & \quad + \mco \lp \lp \lrcxsq + \lrcysq \rp \Lf^2 \lb \Cw \localvar^2 + D (G_{\bx}^2 + \hetero^2) \rb \rp,
    \label{eq_proof:thm_NC_C_PCP_4}
\end{align}
where in \eqref{eq_proof:thm_NC_C_PCP_4}, we have shown dependence only on $\sync, \numclients, \selclients, T$. Lastly, we specify the algorithm parameters in terms of $n,T, \seff, \bar{\sync}$.
$$\lrsx = \Theta \lp \frac{P^{1/4}}{(\seff T)^{3/4}} \rp, \quad \lrsy = \Theta \lp \frac{P^{3/4}}{(\seff T)^{1/4}} \rp, \quad S = \Theta \lp \sqrt{\frac{T}{\seff P}} \rp.$$
Finally, choosing the client learning rates $\lrcx = \lrcy = \frac{1}{\Lf \bar{\sync} T^{3/8}}$, we get
\begin{align*}
    \avgtT \mbe \norm{\G \TPhi_{1/2\Lf} (\bxt)}^2 & \leq \mco \lp \frac{(\bar{\sync}/\seff)^{1/4}}{(\bar{\sync} \selclients T)^{1/4}} \rp + \mco \lp \frac{1}{(\selclients T)^{1/4}} \lp \frac{\numclients (\numclients - \selclients)}{(\numclients-1)} \max_i w_i \rp^{1/4} \rp + \mco \lp \frac{(\seff \selclients)^{1/4}}{T^{3/4}} \rp \nn \\
    & \quad + \mco \lp \frac{\Cw \localvar^2 + D (G_{\bx}^2 + \hetero^2)}{\bar{\sync}^2 T^{3/4}} \rp.
\end{align*}
\end{proof}

\paragraph{Convergence in terms of $F$}
\label{app:NC_C_bias_conv}
\begin{proof}[Proof of \cref{cor:NC_C_obj_inconsistent}]
Following \cite{lin_GDA_icml20}, we define
\begin{align}
    \begin{matrix}
        \TPhi_{1/2\Lf} (\bx) \triangleq \min_{\bx'} \lcb \TPhi (\bx') + \Lf \norm{\bx' - \bx}^2 \rcb; \qquad & \qquad \widetilde{\bx} \triangleq \argmin_{\bx'} \lcb \TPhi (\bx') + \Lf \norm{\bx' - \bx}^2 \rcb, \\ 
        \Phi_{1/2\Lf} (\bx) \triangleq \min_{\bx'} \lcb \Phi (\bx') + \Lf \norm{\bx' - \bx}^2 \rcb; \qquad & \qquad \bar{\bx} \triangleq \argmin_{\bx'} \lcb \Phi (\bx') + \Lf \norm{\bx' - \bx}^2 \rcb.
    \end{matrix}
    \label{eq:defn_Phi_TPhi}
\end{align}
Also, it follows from Lemma~2.2 in \cite{davis19wc_siam} that $\G \TPhi_{1/2\Lf} (\bx) = 2 \Lf (\bx - \widetilde{\bx})$ and $\G \Phi_{1/2\Lf} (\bx) = 2 \Lf (\bx - \bar{\bx})$. Therefore,
\begin{align*}
    \norm{\G \Phi_{1/2\Lf} (\bx)}^2 \leq & 2 \norm{\G \Phi_{1/2\Lf} (\bx) - \G \TPhi_{1/2\Lf} (\bx)}^2 + 2 \norm{\G \TPhi_{1/2\Lf} (\bx)}^2 \nn \\
    &= 8 \Lf^2 \norm{\widetilde{\bx} - \bar{\bx}}^2 + 2 \norm{\G \TPhi_{1/2\Lf} (\bx)}^2
\end{align*}
Consequently, we obtain
\begin{align}
    \min_{t\in[T]} \norm{\nabla \Phi_{1/2\Lf} (\bx^{(t)})}^2
    & \leq \frac{1}{T} \sum_{t=0}^{T-1} \norm{\nabla \Phi_{1/2\Lf} (\bx^{(t)})}^2 \nn \\
    & \leq \frac{2}{T} \sum_{t=0}^{T-1} \lb \norm{\nabla \TPhi_{1/2\Lf} (\bx^{(t)})}^2 + 4 \Lf^2 \norm{\widetilde{\bx}^{(t)} - \bar{\bx}^{(t)}}^2 \rb.
\end{align}
where $\widetilde{\bx}^{(t)}, \bar{\bx}^{(t)}$ follow the same definition as in \eqref{eq:defn_Phi_TPhi}, with $\bx$ replaced with $\bxt$.
\end{proof}

\begin{proof}[Proof of \cref{cor:NC_C_comm_cost}]

If clients are weighted equally ($w_i = p_i = 1/n$ for all $i$), with each carrying out $\sync$ steps of local SGDA+, then \eqref{eq:thm:NC_C} reduces to
\begin{align}
    & \min_{t \in [T]} 
    \mbe \normb{\G \Phi_{1/2\Lf} (\bxt)}^2 \leq \mco \Big( \mfrac{1}{(\sync \selclients T)^{1/4}} + \mfrac{(\sync \selclients)^{1/4}}{T^{3/4}} \Big) + \mco \lp \mfrac{\localvar^2 + \sync (G_{\bx}^2 + \hetero^2)}{\sync T^{3/4}} \rp + \mco \Big( \Big( \mfrac{\numclients - \selclients}{\numclients-1} \cdot \mfrac{1}{\selclients T } \Big)^{1/4} \Big).
    \nn
\end{align}
\begin{itemize}
    \item For full client participation, this reduces to
    \begin{align}
        & \min_{t \in [T]} 
        \mbe \normb{\G \Phi_{1/2\Lf} (\bxt)}^2 \leq \mco \Big( \mfrac{1}{(\sync \numclients T)^{1/4}} + \mfrac{(\sync \numclients)^{1/4}}{T^{3/4}} \Big).
        \nn
    \end{align}
    To reach an $\epsilon$-stationary point, assuming $\numclients \sync \leq T$, the per-client gradient complexity is $T \sync = \mco \lp \frac{1}{\numclients \epsilon^8} \rp$. Since $\sync \leq T/n$, the minimum number of communication rounds required is $T = \mco \lp \frac{1}{\epsilon^4} \rp$.
    \item For partial participation, $\mco \Big( \Big( \mfrac{\numclients - \selclients}{\numclients-1} \cdot \mfrac{1}{\selclients T } \Big)^{1/4} \Big)$ is the dominant term, and we do not get any convergence benefit of multiple local updates. Consequently, per-gradient client complexity and number of communication rounds are both $T \sync = \mco \lp \frac{1}{\selclients \epsilon^8} \rp$, for $\sync = \mco (1)$. However, if the data across clients comes from identical distributions ($\hetero = 0$), then we recover 
    per-client gradient complexity of $\mco \lp \frac{1}{\selclients \epsilon^8} \rp$, and number of communication rounds $= \mco \lp \frac{1}{\epsilon^4} \rp$.
\end{itemize}
\end{proof}

\paragraph{Special Cases}
\begin{itemize}
    \item Centralized, deterministic case $(\localvar = \hetero = 0, \heteroscale = 1, \seff = n = 1)$: in this case $\Aw = \Bw = 1, \Cw = D = 0$. Also, $\mc I_1 = G_{\bx} \sqrt{\varscale^2+1}, \mc I_2 = 0$. The bound in \eqref{eq:thm:NC_C_1} reduces to
    \begin{align}
        \avgtT \mbe \norm{\G \TPhi_{1/2\Lf} (\bxt)}^2 & \leq \mco \lp \frac{\bar{\Delta}_{\TPhi}}{\lrsx T} + \lrsx \Lf G_{\bx}^2 \lb (\varscale^2+1) + (S-1) \sqrt{\varscale^2+1} \rb + \frac{\Lf R}{\lrsy S} \rp. \label{eq_proof:thm_NC_C_special_case_1}
    \end{align}
    For $\varscale = 0$, \eqref{eq_proof:thm_NC_C_special_case_1} yields the convergence result in \cite{lin_GDA_icml20}.
    \item Single node, stochastic case $(\hetero = 0, \heteroscale = 1, \seff = n = 1)$: in this case $\Aw = \Bw = 1, \Cw = D = 0$. Also, $\mc I_1 = \sqrt{\localvar^2 + (\varscale^2+1) G_{\bx}^2}, \mc I_2 = \localvar^2$. The bound in \eqref{eq:thm:NC_C_1} reduces to
    \begin{align}
        \avgtT \mbe \norm{\G \TPhi_{1/2\Lf} (\bxt)}^2 & \leq \Theta \lp \frac{\bar{\Delta}_{\TPhi}}{\lrsx T} + \lrsx \Lf (\localvar^2 + (\varscale^2+1) G_{\bx}^2) + \lrsy \Lf \localvar^2 \rp \nn \\
        & \quad + \Theta \lp \Lf \lb \lrsx G_{\bx} (S-1) \sqrt{\localvar^2 + (\varscale^2+1) G_{\bx}^2} + \frac{R}{\lrsy S} \rb \rp. \label{eq_proof:thm_NC_C_special_case_2}
    \end{align}
    Again, for $\varscale = 0$, \eqref{eq_proof:thm_NC_C_special_case_2} yields the convergence result in \cite{lin_GDA_icml20}.
    \item Multiple equally weighted $(w_i = 1/n, \forall \ i \in [n])$ clients, full client participation, stochastic case with synchronous client updates $(\seff = 1)$: in this case $\Aw = 1, \Bw = 1, \Cw = D = 0, 0$. The bound in \eqref{eq:thm:NC_C_1} reduces to
    \begin{align}
        & \avgtT \mbe \norm{\G \TPhi_{1/2\Lf} (\bxt)}^2 \leq \Theta \lp \frac{\bar{\Delta}_{\TPhi}}{\lrsx T} + \lrsx \Lf \lp G_{\bx}^2 + \frac{\localvar^2 + \varscale^2 G_{\bx}^2}{\numclients} \rp \rp \nn \\
        & \quad + \Theta \lp \frac{\lrsy \Lf}{\numclients} (\localvar^2 + \hetero^2 \varscale^2) + \lrsx \Lf G_{\bx} (S-1) \sqrt{G_{\bx}^2 + \frac{\localvar^2 + \varscale^2 G_{\bx}^2}{\numclients}} + \frac{R \Lf}{\lrsy S} \rp, \label{eq_proof:thm_NC_C_special_case_3b}
    \end{align}
    Note that unlike existing analyses of synchronous update algorithms \cite{woodworth2020minibatch_NeurIPS20, yun2021minibatch_ICLR22, sharma22FedMinimax_ICML}, the bound in \eqref{eq_proof:thm_NC_C_special_case_3b} depends on the inter-client heterogeneity $\hetero^2$. This is due to the more general noise assumption (\cref{assum:bdd_var}). In the existing works, $\varscale^2$ is assumed zero, in which case, the bound in \eqref{eq_proof:thm_NC_C_special_case_3b} is also independent of $\hetero^2$. See \cref{sec:aux} for a more detailed explanation.
    \item Multiple, equally weighted $(w_i = 1/n, \forall \ i \in [n])$ clients, full client participation, multiple, but equal number of client updates $(\sync_i = \seff = \sync, \forall \ i \in [n])$. In this case $\Aw = \Bw = 1, \Cw = \sync - 1, D = (\sync-1) (\sync-1+\varscale^2)$. The bound in \eqref{eq:thm:NC_C_1} then reduces to
    {\small
    \begin{align}
        & \avgtT \mbe \norm{\G \TPhi_{1/2\Lf} (\bxt)}^2 \leq \Theta \lp \frac{\bar{\Delta}_{\TPhi}}{\sync \lrsx T} + \sync \lrsx \Lf \lp G_{\bx}^2 + \frac{\localvar^2 + \varscale^2 G_{\bx}^2}{\numclients \sync} \rp + \frac{\lrsy \Lf (\localvar^2 + \varscale^2 \hetero^2)}{\numclients} \rp \label{eq_proof:thm_NC_C_special_case_4} \\
        & + \Theta \lp \Lf \lb \sync \lrsx G_{\bx} (S-1) \sqrt{G_{\bx}^2 + \frac{\localvar^2 + \varscale^2 G_{\bx}^2}{\numclients \sync}} + \frac{R}{\sync \lrsy S} \rb + (\sync-1) \lp \lrcxsq + \lrcysq \rp \Lf^2 \lb \localvar^2 + (\sync - 1 + \varscale^2) (G_{\bx}^2 + \hetero^2) \rb \rp. \nn
    \end{align}
    }%
    For $\varscale = \heteroscale = 0$, this setting reduces to the one considered in \cite{sharma22FedMinimax_ICML}. However, as stated earlier, our bound on the local update error is tighter.
\end{itemize}

\subsection{Proofs of the Intermediate Lemmas}
\label{sec:NC_C_PCP_int_results_proofs}

\begin{proof}[Proof of \cref{lem:NC_C_PCP_Phi_smooth_decay_one_iter}]
Using the definition in \eqref{eq:defn_Phi_TPhi} $\bbxt = \argmin_\bx \TPhi (\bx) + \Lf \norm{\bx - \bxt}^2$. Also, note that
\begin{align}
    \TPhi_{1/2\Lf} (\bxtp) & \leq \TPhi (\bbxt) + \Lf \norm{\bbxt - \bxtp}^2.
    \label{eq:lem:NC_C_PCP_Phi_smooth_decay_one_iter_1}
\end{align}
Using the $\bx$ updates in \eqref{eq:server_update_alg_NC_C_minimax_PCP},
\begin{align}
    & \mbe \norm{\bbxt - \bxtp}^2 = \mbe \norm{\bbxt - \bxt + \seff \lrsx \sumiS \twi \bdxit}^2 \nn \\
    &= \mbe \norm{\bbxt - \bxt}^2 + \seff^2 \lrsxsq \mbe \norm{\sumiS \twi \bdxit}^2 + 2 \seff \lrsx \mbe \lan \bbxt - \bxt, \sumin w_i \bhxit \ran \nn \\
    & \leq \mbe \norm{\bbxt - \bxt}^2 + 2 \seff \lrsx \mbe \lan \bbxt - \bxt, \Gx \TF(\bxt, \byt) \ran + \seff^2 \lrsxsq \mbe \norm{\sumiS \twi \bdxit}^2 \nn \\
    & \quad + \seff \lrsx \mbe \lb \frac{\Lf}{2} \norm{\bbxt - \bxt}^2 + \frac{2}{\Lf} \norm{\sumin \frac{w_i}{\nai_1} \sumikt \aikt \lp \Gx f_i (\bxitk, \byitk) - \Gx f_i (\bxt, \byt) \rp}^2 \rb \nn \\
    & \leq \mbe \norm{\bbxt - \bxt}^2 + \seff^2 \lrsxsq \mbe \norm{\sumiS \twi \bdxit}^2 \\
    & \quad + \seff \lrsx \mbe \lb \frac{\Lf}{2} \norm{\bbxt - \bxt}^2 + 2 \Lf \sumin \frac{w_i}{\nai_1} \sumikt \aikt  \CExytk (i) + 2 \lan \bbxt - \bxt, \Gx \TF(\bxt, \byt) \ran \rb,
    \label{eq:lem:NC_C_PCP_Phi_smooth_decay_one_iter_2}
\end{align}
where \eqref{eq:lem:NC_C_PCP_Phi_smooth_decay_one_iter_2} follows from $\Lf$-smoothness (\cref{assum:smoothness}) and Jensen's inequality.
From \eqref{eq_proof:lem:NC_SC_PCP_WOR_wtd_dir_norm_sq_1}, \eqref{eq_proof:lem:NC_SC_PCP_WOR_wtd_dir_norm_sq_2}, we can bound $\mbe \normb{\sum_{i \in \mc C^{(t')}} \twi \bd_{\mbf x, i}^{(t')}}^2$ as follows.
\begin{align}
    \mbe \norm{\sum_{i \in \mc C^{(t')}} \twi \bd_{\mbf x, i}^{(t')}}^2 & \leq \frac{\numclients}{\selclients} \sumin \frac{w_i^2}{\nai_1^2} \sumikt [ \aikt ]^2 \lb \localvar^2 + \varscale^2 \mbe \lnr \Gx f_i ( \bxitk, \byitk ) \rnr^2 \rb \nn \\
    & \quad + \frac{\numclients}{\selclients} \lp \frac{\selclients-1}{\numclients-1} \rp \mbe \lnr \sumin w_i \bhxit \rnr^2 + \frac{\numclients}{\selclients} \frac{\numclients - \selclients}{\numclients-1} \sumin w_i^2 \mbe \lnr \bhxit \rnr^2 \nn \\
    & \leq \frac{\numclients}{\selclients} \sumin \frac{w_i^2 \nai_2^2}{\nai_1^2} \lp \localvar^2 + \varscale^2 G_{\bx}^2 \rp + \frac{\numclients (\selclients-1)}{\selclients (\numclients-1)} G_{\bx}^2 + \frac{\numclients (\numclients - \selclients)}{\selclients (\numclients-1)} G_{\bx}^2 \sumin w_i^2, \label{eq:lem:NC_C_PCP_WOR_agg_grad_norm}
\end{align}
where the final inequality by using \cref{assum:Lips_cont_x}. Next, we bound the inner product term in \eqref{eq:lem:NC_C_PCP_Phi_smooth_decay_one_iter_2}. Using $\Lf$-smoothness of $F$ (\cref{assum:smoothness}):
\begin{align}
    & \mbe \lan \bbxt - \bxt, \Gx \TF(\bxt, \byt) \ran \leq \mbe \lb \TF(\bbxt, \byt) - \TF(\bxt, \byt) + \frac{\Lf}{2} \norm{\bbxt - \bxt}^2 \rb \nn \\
    & \leq \mbe \lb \TPhi(\bbxt) + \Lf \norm{\bbxt - \bxt}^2 \rb - \mbe \TF(\bxt, \byt) - \frac{\Lf}{2} \mbe \norm{\bbxt - \bxt}^2 \nn \\
    & \leq \mbe \lb \TPhi(\bxt) + \Lf \norm{\bxt - \bxt}^2 \rb - \mbe \TF(\bxt, \byt) - \frac{\Lf}{2} \mbe \norm{\bbxt - \bxt}^2 \tag{by definition of $\bbxt$} \\
    & \leq \mbe \lb \TPhi(\bxt) - \TF(\bxt, \byt) - \frac{\Lf}{2} \norm{\bbxt - \bxt}^2 \rb. \label{eq:lem:NC_C_PCP_Phi_smooth_decay_one_iter_3}
\end{align}
Substituting the bounds from \eqref{eq:lem:NC_C_PCP_Phi_smooth_decay_one_iter_2} and \eqref{eq:lem:NC_C_PCP_Phi_smooth_decay_one_iter_3} into \eqref{eq:lem:NC_C_PCP_Phi_smooth_decay_one_iter_1}, we get
\begin{align}
    & \mbe \lb \TPhi_{1/2\Lf} (\bxtp) \rb \leq \mbe \lb \TPhi (\bbxt) + \Lf \norm{\bbxt - \bxt}^2 \rb \nn \\
    & \quad + \seff^2 \lrsxsq \Lf \frac{\numclients}{\selclients} \lb \sumin \frac{w_i^2 \nai_2^2}{\nai_1^2} \lp \localvar^2 + \varscale^2 G_{\bx}^2 \rp + G_{\bx}^2 \lp \frac{\selclients-1}{\numclients-1} + \frac{\numclients - \selclients}{\numclients-1} \sumin w_i^2 \rp \rb \nn \\
    & \quad + 2 \seff \lrsx \Lf^2 \sumin \frac{w_i}{\nai_1} \sumikt \aikt  \CExytk (i) + 2 \seff \lrsx \Lf \mbe \lb \TPhi(\bxt) - \TF(\bxt, \byt) \rb - \frac{\seff \lrsx \Lf^2}{2} \mbe \norm{\bbxt - \bxt}^2 \nn \\
    & \leq \mbe \lb \TPhi_{1/2\Lf} (\bxt) \rb + \seff^2 \lrsxsq \Lf \frac{\numclients}{\selclients} \lb \sumin \frac{w_i^2 \nai_2^2}{\nai_1^2} \lp \localvar^2 + \varscale^2 G_{\bx}^2 \rp + G_{\bx}^2 \lp \frac{\selclients-1}{\numclients-1} + \frac{\numclients - \selclients}{\numclients-1} \sumin w_i^2 \rp \rb \nn \\
    & \quad + 2 \seff \lrsx \lcb \Lf^2 \sumin \frac{w_i}{\nai_1} \sumikt \aikt  \CExytk (i) + \Lf \mbe \lb \TPhi(\bxt) - \TF(\bxt, \byt) \rb \rcb - \frac{\seff \lrsx}{8} \mbe \norm{\G \TPhi_{1/2\Lf} (\bxt)}^2, \nn
\end{align}
where we use $\G \TPhi_{1/2\Lf} (\bx) = 2 \Lf (\bx - \bbx)$ from \eqref{eq:defn_Phi_TPhi}.
\end{proof}

\begin{proof}[Proof of \cref{lem:NC_C_PCP_consensus_error}]
We use the client update equations for individual iterates in \eqref{eq:client_update_alg_NC_C_minimax_PCP}. To bound $\CExytk (i)$, first we bound the $\bx$-error $\mbe \lnr \bxitk - \bxt \rnr^2$. 
Starting from \eqref{eq_proof:lem:NC_SC_PCP_consensus_error_1}, using \cref{assum:Lips_cont_x}, for $1 \leq k \leq \sync_i$,
\begin{align}
    \frac{1}{\nai_1} \sumikt \aikt \mbe \lnr \bxitk - \bxt \rnr^2 & \leq \frac{\lrcxsq}{\nai_1} \sumikt \aikt \lb \sumijk [\aijk]^2 \lp \localvar^2 + \varscale^2 G_{\bx}^2 \rp + \lp \sumijk \aijk \rp \sumijk \aijk G_{\bx}^2 \rb \nn \\
    & \leq \lrcxsq \lb \localvar^2 \norm{\mbf a_{i,-1}}_2^2 + G_{\bx}^2 \lp \norm{\mbf a_{i,-1}}_1^2 + \varscale^2 \norm{\mbf a_{i,-1}}_2^2 \rp \rb,
    \label{eq_proof:lem:NC_C_PCP_consensus_error_1a}
\end{align}
where we use \eqref{eq_proof:lem:NC_SC_PCP_consensus_error_2}.
Next, we bound $\mbe \lnr \byitk - \byt \rnr^2$, using the bound from \eqref{eq_proof:lem:NC_SC_PCP_consensus_error_4}, to get
\begin{align}
    \frac{1}{\nai_1} \sumikt \aikt \mbe \lnr \byitk - \byt \rnr^2 & \leq \lrcysq \localvar^2 \norm{\mbf a_{i,-1}}_2^2 + 2 \lrcysq \Lf^2 \lp \norm{\mbf a_{i,-1}}_1 + \varscale^2 \alpha \rp \sumikt \aikt \CEytk(i) \nn \\
    & \quad + 2 \lrcysq \lp \norm{\mbf a_{i,-1}}_1^2 + \varscale^2 \norm{\mbf a_{i,-1}}_2^2 \rp \mbe \lnr \Gy f_i ( \Hbxs, \byt ) \rnr^2.
    \label{eq_proof:lem:NC_C_PCP_consensus_error_1b}
\end{align}
Compared to \eqref{eq_proof:lem:NC_SC_PCP_consensus_error_4}, the difference is the presence of $\CEytk(i)$ in \eqref{eq_proof:lem:NC_C_PCP_consensus_error_1b}, rather than $\CExytk(i)$.
Taking a weighted sum over agents in \eqref{eq_proof:lem:NC_C_PCP_consensus_error_1b}, we get
\begin{align}
    \Lf^2 \sumin \frac{w_i}{\nai_1} \sumikt \aikt \CEytk (i) & \leq 2 \lrcysq \Lf^2 \lb \localvar^2 \sumin w_i \norm{\mbf a_{i,-1}}_2^2 + 2 \TMa \lp \heteroscale^2 \mbe \lnr \Gy \TF( \Hbxs, \byt ) \rnr^2 + \hetero^2 \rp \rb.
    \label{eq_proof:lem:NC_C_PCP_consensus_error_2}
\end{align}
where, we choose $\lrcy$ such that $A_m \triangleq 2 \Lf^2 \lrcysq \max_i \nai_1 \lp \norm{\mbf a_{i,-1}}_1 + \varscale^2 \alpha \rp \leq \frac{1}{2}$, and define $\TMa \triangleq \max_i \lp \norm{\mbf a_{i,-1}}_1^2 + \varscale^2 \norm{\mbf a_{i,-1}}_2^2 \rp$.
Next, it follows from $\Lf$-smoothness (\cref{assum:smoothness}) and \cref{lem:smooth} that
\begin{align*}
    \mbe \lnr \Gy \TF \lp \Hbxs, \byt \rp \rnr^2 & \leq 2 \Lf \mbe \lb \TPhi (\Hbxs) - \TF(\Hbxs, \byt) \rb.
\end{align*}
Subsequently, combining \eqref{eq_proof:lem:NC_C_PCP_consensus_error_1a} and \eqref{eq_proof:lem:NC_C_PCP_consensus_error_2}, we get
\begin{align}
    \Lf^2 \sumin \frac{w_i}{\nai_1} \sumikt \aikt \CExytk (i) & \leq 2 \lp \lrcxsq + \lrcysq \rp \Lf^2 \localvar^2 \sumin w_i \norm{\mbf a_{i,-1}}_2^2 + 4 \Lf^2 \TMa \lp \lrcxsq G_{\bx}^2 + \lrcysq \hetero^2 \rp \nn \\
    & \quad + 8 \lrcysq \Lf^3 \TMa \heteroscale^2 \mbe \lb \TPhi (\Hbxs) - \TF(\Hbxs, \byt) \rb. \nn
\end{align}
which finishes the proof.
\end{proof}

\begin{proof}[Proof of \cref{lem:NC_C_PCP_local_SGA_concave}]
We define $\by^* (\Hbxs) \in \argmax_\by \TF(\Hbxs, \by)$. Then,
\begin{align}
    & \mbe \norm{\bytp - \by^* (\Hbxs)}^2 \overset{\eqref{eq:server_update_alg_NC_C_minimax_PCP}}{=} \mbe \norm{\byt + \seff \lrsy \bdyt - \by^* (\Hbxs)}^2 \nn \\
    &= \mbe \norm{\byt - \by^* (\Hbxs)}^2 + \seff^2 \lrsysq \mbe \norm{\bdyt}^2 + 2 \seff \lrsy \mbe \lan \byt - \by^* (\Hbxs), \sumin w_i \bhyit \ran. \label{eq_proof:lem:NC_C_PCP_local_SGA_concave_1}
\end{align}
$\mbe \norm{\bdyt}^2$ is bounded in \eqref{eq_proof:lem:NC_SC_PCP_WOR_phi_f_diff_3b}. We only need to further bound $\mbe \lnr \sumin w_i \bhyit \rnr^2$ which appears in \eqref{eq_proof:lem:NC_SC_PCP_WOR_phi_f_diff_3b}.
\begin{align}
    \mbe \lnr \sumin w_i \bhyit \rnr^2 & \leq \mbe \lnr \sumin \frac{w_i}{\nai_1} \sumikt [ \aikt ] \lp \Gy f_i ( \Hbxs, \byitk ) - \Gy f_i ( \Hbxs, \byt ) + \Gy f_i ( \Hbxs, \byt ) \rp \rnr^2 \nn \\
    & \leq 2 \Lf^2 \sumin \frac{w_i}{\nai_1} \sumikt [ \aikt ] \CEytk (i) + 2 \lnr \Gy \TF (\Hbxs, \byt) \rnr^2 \tag{Jensen's inequality} \\
    & \leq 2 \Lf^2 \sumin \frac{w_i}{\nai_1} \sumikt [ \aikt ] \CEytk (i) + 4 \Lf \mbe \lb \TPhi (\Hbxs) - \TF (\Hbxs, \byt) \rb. \label{eq_proof:lem:NC_C_PCP_local_SGA_concave_2}
\end{align}
Next, we bound the third term in \eqref{eq_proof:lem:NC_C_PCP_local_SGA_concave_1}.
\begin{align}
    & \mbe \lan \byt - \by^* (\Hbxs), \sumin w_i \bhyit \ran = \mbe \lan \byt - \by^* (\Hbxs), \sumin \frac{w_i}{\nai_1} \sumikt [ \aikt ] \Gy f_i ( \Hbxs, \byitk ) \ran \nn \\
    &= \sumin \frac{w_i}{\nai_1} \sumikt [ \aikt ] \mbe \lb \lan \byt - \byitk, \Gy f_i ( \Hbxs, \byitk ) \ran + \lan \byitk - \by^* (\Hbxs), \Gy f_i ( \Hbxs, \byitk ) \ran \rb \nn \\
    & \leq \sumin \frac{w_i}{\nai_1} \sumikt [ \aikt ] \mbe \Bigg[ f_i ( \Hbxs, \byt ) - f_i ( \Hbxs, \byitk ) + \frac{\Lf}{2} \norm{\byt - \byitk}^2 \tag{$\Lf$-smoothness} \\
    & \qquad \qquad \qquad \qquad \qquad \qquad + f_i ( \Hbxs, \byitk ) - f_i ( \Hbxs, \by^* (\Hbxs) ) \Bigg] \tag{Concavity in $\by$} \\
    &= \frac{\Lf}{2} \sumin \frac{w_i}{\nai_1} \sumikt [ \aikt ] \CEytk (i) - \mbe \lb \TPhi(\Hbxs) - \TF(\Hbxs, \byt) \rb. \label{eq_proof:lem:NC_C_PCP_local_SGA_concave_3}
\end{align}
Substituting \eqref{eq_proof:lem:NC_SC_PCP_WOR_phi_f_diff_3b}, \eqref{eq_proof:lem:NC_C_PCP_local_SGA_concave_2}, \eqref{eq_proof:lem:NC_C_PCP_local_SGA_concave_3} in \eqref{eq_proof:lem:NC_C_PCP_local_SGA_concave_1}, we get
\begin{align}
    & \mbe \norm{\bytp - \by^* (\Hbxs)}^2 \nn \\
    & \leq \mbe \norm{\byt - \by^* (\Hbxs)}^2 + \seff^2 \lrsysq \lb \frac{\localvar^2 \numclients}{\selclients} \sumin \frac{w_i^{{2}} \nai_2^2}{\nai_1^2} + \frac{2 \hetero^2 \numclients}{\selclients} \lp \frac{\numclients - \selclients}{\numclients-1} \max_i w_i + \varscale^2 \max_i \frac{w_i \nai_2^2}{\nai_1^2} \rp \rb \nn \\
    & \quad - 2 \seff \lrsy \lp 1 - 2 \seff \lrsy \Lf \frac{\numclients}{\selclients} \lb \frac{\selclients - 1}{\numclients-1} + \heteroscale^2 \lp \frac{\numclients - \selclients}{\numclients-1} \max_i w_i + \varscale^2 \max_i \frac{w_i \nai_2^2}{\nai_1^2} \rp \rb \rp \mbe \lb \TPhi(\Hbxs) - \TF(\Hbxs, \byt) \rb \nn \\
    & \quad + \seff \lrsy \Lf \lb 1 + 2 \seff \lrsy \Lf \frac{\numclients}{\selclients} \lp \frac{\selclients - 1}{\numclients-1} + \frac{\numclients - \selclients}{\numclients-1} \max_i w_i + \varscale^2 \max_{i,k} \frac{w_i \aikt}{\nai_1} \rp \rb \sumin \frac{w_i}{\nai_1} \sumikt [ \aikt ] \CExytk (i), \label{eq_proof:lem:NC_C_PCP_local_SGA_concave_4}
\end{align}
since $\CEytk \leq \CExytk$. Using the bound on $\CExytk$ from \cref{lem:NC_C_PCP_consensus_error},
\begin{align}
    \sumin \frac{w_i}{\nai_1} \sumikt \aikt \CExytk (i) & \leq 2 \lp \lrcxsq + \lrcysq \rp \localvar^2 \sumin w_i \norm{\mbf a_{i,-1}}_2^2 + 4 \TMa \lp \lrcxsq G_{\bx}^2 + \lrcysq \hetero^2 \rp \nn \\
    & \quad + 8 \lrcysq \Lf \TMa \heteroscale^2 \mbe \lb \TPhi (\Hbxs) - \TF(\Hbxs, \byt) \rb. 
    \label{eq_proof:lem:NC_C_PCP_local_SGA_concave_5}
\end{align}
We substitute \eqref{eq_proof:lem:NC_C_PCP_local_SGA_concave_5} in \eqref{eq_proof:lem:NC_C_PCP_local_SGA_concave_4}, and simplify the terms using the choice of $\lrsy, \lrcy$ to get
\begin{align*}
    & \mbe \norm{\bytp - \by^* (\Hbxs)}^2 \nn \\
    & \leq \mbe \norm{\byt - \by^* (\Hbxs)}^2 + \seff^2 \lrsysq \lb \frac{\localvar^2 \numclients}{\selclients} \sumin \frac{w_i^{{2}} \nai_2^2}{\nai_1^2} + \frac{2 \hetero^2 \numclients}{\selclients} \lp \frac{\numclients - \selclients}{\numclients-1} \max_i w_i + \varscale^2 \max_i \frac{w_i \nai_2^2}{\nai_1^2} \rp \rb \nn \\
    & \quad - \seff \lrsy \mbe \lb \TPhi(\Hbxs) - \TF(\Hbxs, \byt) \rb + 4 \seff \lrsy \Lf ( \lrcxsq + \lrcysq) \lb \localvar^2 \sumin w_i \norm{\mbf a_{i,-1}}_2^2 + 2 \TMa (G_\bx^2 + \hetero^2) \rb.
\end{align*}
using $\lrsy, \lrcy$ that satisfy
\begin{align*}
    2 \seff \lrsy \Lf \frac{\numclients}{\selclients} \lp \frac{\selclients - 1}{\numclients-1} + \frac{\numclients - \selclients}{\numclients-1} \max_i w_i + \varscale^2 \max_{i,k} \frac{w_i \aikt}{\nai_1} \rp & \leq 1, \\
    2 \seff \lrsy \Lf \frac{\numclients}{\selclients} \lb \frac{\selclients - 1}{\numclients-1} + \heteroscale^2 \lp \frac{\numclients - \selclients}{\numclients-1} \max_i w_i + \varscale^2 \max_i \frac{w_i \nai_2^2}{\nai_1^2} \rp \rb & \leq \frac{1}{4}, \\
    2 \Lf \frac{\numclients}{\selclients} \lb 8 \lrcysq \TMa \Lf \heteroscale^2 \rb & \leq \frac{1}{4}
\end{align*}
Then the coefficient of $\mbe \lb \TPhi(\Hbxs) - \TF(\Hbxs, \byt) \rb$ can we bounded by $- \seff \lrsy$. 
Consequently, by rearranging the terms and summing over $t$, we get the result.
\begin{align}
    & \frac{1}{S} \sum_{t = sS}^{(s+1) S - 1} \mbe \lb \TPhi(\Hbxs) - \TF(\Hbxs, \byt) \rb \nn \\
    & \leq \frac{\mbe \norm{\by^{sS} - \by^* (\Hbxs)}^2}{\seff \lrsy S} + \seff \lrsy \frac{\numclients}{\selclients} \lb \localvar^2 \sumin \frac{w_i^{{2}} \nai_2^2}{\nai_1^2} + 2 \hetero^2 \lp \frac{\numclients - \selclients}{\numclients-1} \max_i w_i + \varscale^2 \max_i \frac{w_i \nai_2^2}{\nai_1^2} \rp \rb \nn \\
    & \quad + 4 \Lf ( \lrcxsq + \lrcysq) \lb \localvar^2 \sumin w_i \norm{\mbf a_{i,-1}}_2^2 + 2 \TMa (G_\bx^2 + \hetero^2) \rb. \nn
\end{align}
\end{proof}

\begin{proof}[Proof of \cref{lem:NC_C_PCP_Phi_f_diff}]
Let $t = sS, sS+1, \dots, (s+1) S - 1$, where $k$ is a positive integer. Let $\Hbxs$ is the latest snapshot iterate for the $\by$-update in \cref{alg_NC_minimax}-\fedsgdaplus \ . Then
\begin{align}
    & \mbe \lb \TPhi(\bxt) - \TF(\bxt, \byt) \rb \nn \\
    &= \mbe \lb \TF(\bxt, \by^*(\bxt)) - \TF(\Hbxs, \by^*(\Hbxs)) + \TF(\Hbxs, \by^*(\Hbxs)) - \TF(\Hbxs, \byt) + \TF(\Hbxs, \byt) - \TF(\bxt, \byt) \rb \nn \\
    & \leq \mbe \lb \TF(\bxt, \by^*(\bxt)) - \TF(\Hbxs, \by^*(\bxt)) \rb + \mbe \lb \TF(\Hbxs, \by^*(\Hbxs)) - \TF(\Hbxs, \byt) \rb + G_{\bx} \mbe \norm{\bxt - \Hbxs} \nn \\
    & \leq 2 G_{\bx} \mbe \norm{\bxt - \Hbxs} + \mbe \lb \TPhi(\Hbxs) - \TF(\Hbxs, \byt) \rb. \label{eq_proof:lem:NC_C_PCP_Phi_f_diff_1}
\end{align}
where, $\by^*(\cdot) \in \argmax_\by \TF(\cdot, \by)$ and \eqref{eq_proof:lem:NC_C_PCP_Phi_f_diff_1} follows from $G_{\bx}$-Lipschitz continuity of $F(\cdot, \by)$ (\cref{assum:Lips_cont_x}). Next, we see that
\begin{align}
    \mbe \norm{\Hbxs - \bxt} & \leq \sqrt{\mbe \norm{\Hbxs - \bxt}^2} \tag{Jensen's inequality} \nn \\
    & \overset{\eqref{eq:server_update_alg_NC_C_minimax_PCP}}{=} \sqrt{\mbe \norm{\seff \lrsx \sum_{t'=sS}^{t-1} \sum_{i \in \mc C^{(t')}} \twi \bd_{\mbf x, i}^{(t')}}^2} \nn \\
    & \leq \seff \lrsx \sqrt{(S-1) \sum_{t'=sS}^{t-1} \mbe \norm{\sum_{i \in \mc C^{(t')}} \twi \bd_{\mbf x, i}^{(t')}}^2} \nn \\
    & \leq \seff \lrsx (S-1) \sqrt{\frac{\numclients}{\selclients}} \sqrt{\sumin \frac{w_i^2 \nai_2^2}{\nai_1^2} \lp \localvar^2 + \varscale^2 G_{\bx}^2 \rp + G_{\bx}^2 \lp \frac{\selclients-1}{\numclients-1} + \frac{\numclients - \selclients}{\numclients-1} \sumin w_i^2 \rp}. \tag{from \eqref{eq:lem:NC_C_PCP_WOR_agg_grad_norm}}
\end{align}
Using this bound in \eqref{eq_proof:lem:NC_C_PCP_Phi_f_diff_1}, and summing over $t$, we get
\begin{align}
    & \frac{1}{S} \sum_{t=sS}^{(s+1)S-1} \mbe \lb \TPhi(\bxt) - \TF(\bxt, \byt) \rb \leq \frac{1}{S} \sum_{t=sS}^{(s+1)S-1} \mbe \lb \TPhi(\Hbxs) - \TF(\Hbxs, \byt) \rb \nn \\
    & \qquad + 2 \seff \lrsx G_{\bx} (S-1) \sqrt{\frac{\numclients}{\selclients}} \sqrt{\sumin \frac{w_i^2 \nai_2^2}{\nai_1^2} \lp \localvar^2 + \varscale^2 G_{\bx}^2 \rp + G_{\bx}^2 \lp \frac{\selclients-1}{\numclients-1} + \frac{\numclients - \selclients}{\numclients-1} \sumin w_i^2 \rp}. \nn
\end{align}
Finally, summing
over $s = 0$ to $T/S - 1$ we get the result.
\end{proof}

\subsection{Extending the result for Nonconvex One-Point-Concave (NC-1PC) Functions (\texorpdfstring{\cref{thm:NC_1PC}}{Theorem 3})}
\label{app:NC_1PC}
Carefully revisting the proof of \cref{thm:NC_C}, we notice that \cref{lem:NC_C_PCP_Phi_smooth_decay_one_iter} and \cref{lem:NC_C_PCP_consensus_error} do not rely on the concavity assumption. \cref{lem:NC_C_PCP_local_SGA_concave} does use concavity of local functions $\{ f_i \}$. However, it is only needed to derive \eqref{eq_proof:lem:NC_C_PCP_local_SGA_concave_3}. Further, this only requires concavity of local functions at a global point $\by^*(\widehat{\bx}^{(s)})$. Therefore, as mentioned earlier in \cref{rem:NC_C_localSGD}, it holds even for NC-1PC functions. This is an independent result in itself, since we have extended the existing convergence result of local stochastic gradient method for convex minimization (concave maximization) problems, to a much more general one-point-convex minimization (or one-point-convex maximization) problem. Therefore, we restate it here for the more general case.

\begin{lemma}[Local SG updates for One-Point-Concave Maximization]
\label{lem:NC_1PC_PCP_local_SGA_concave}
Suppose the local loss functions $\{ f_i \}$ satisfy Assumptions \ref{assum:smoothness}, \ref{assum:bdd_var}, \ref{assum:bdd_hetero}, \ref{assum:Lips_cont_x}. Suppose for all $\bx$, all the $f_i$'s satisfy \cref{assum:1pc_y} at a common global minimizer $\by^* (\bx)$, and that $\normb{\byt}^2 \leq R$ for all $t$. If we run \cref{alg_NC_minimax}-\fedsgdaplus \  with same conditions on the client and server step-sizes $\lrcy, \lrsy$ respectively, as in \cref{lem:NC_C_PCP_local_SGA_concave}, then the iterates generated by \cref{alg_NC_minimax}-\fedsgdaplus \  also satisfy the bound in \cref{lem:NC_C_PCP_local_SGA_concave}.
\end{lemma}
Next, \cref{lem:NC_C_PCP_Phi_f_diff} also holds irrespective of concavity. Therefore, the resulting convergence result in \cref{thm:NC_C} for nonconvex-concave minimax problems holds for a much larger class of functions. We restate the modified theorem statement briefly.

\begin{theorem*}
Suppose the local loss functions $\{ f_i \}$ satisfy Assumptions \ref{assum:smoothness}, \ref{assum:bdd_var}, \ref{assum:bdd_hetero}, \ref{assum:Lips_cont_x}. Suppose for all $\bx$, all the $f_i$'s satisfy \cref{assum:1pc_y} at a common global minimizer $\by^* (\bx)$, and that $\normb{\byt}^2 \leq R$ for all $t$. 
If we run \cref{alg_NC_minimax}-\fedsgdaplus \  with the same conditions on the client and server step-sizes $\lrcy, \lrsy$ respectively, as in \cref{thm:NC_C_appendix}, then the iterates generated by \cref{alg_NC_minimax}-\fedsgdaplus \  also satisfy the bound in \cref{thm:NC_C_appendix}.
\end{theorem*}

\begin{remark}
Again, choosing client weights $\{ w_i \}$ the same as in the original global objective $\{ p_i \}$, we get convergence in terms of the original objective $F$.
\end{remark}

%% file: Appendix/6_Experiments.tex

\section{Additional Experiments}
\label{app:add_exp}

For communicating parameters and related information amongst the clients, ethernet connections were used.
Our algorithm was implemented using parallel training tools in PyTorch 1.0.0 and Python 3.6.3.

For both robust NN Training and fair classification experiments, we use batch-size of $32$ in all the algorithms. Momentum parameter $0.9$ is used only in Momentum Local SGDA(+).


\paragraph{Robust NN Training}
Here we further explore performance of \fedsgdaplus \ on the robust NN training problem. We use VGG-11 model to classify CIFAR10 dataset. In \cref{fig:robustNN_vary_hetero}, we demonstrate the effect of increasing data heterogeneity across clients, whle in \cref{fig:robustNN_linear_speedup} we show the advantage of using multiple clients for the federated minimax problem. With $k$-fold increase in $\numclients$, we observe an almost $k$-fold drop in the number of communication rounds needed to reach a target test accuracy ($70\%$ here.).

\begin{figure}[h]
    \centering
    \includegraphics[width=0.35\textwidth]{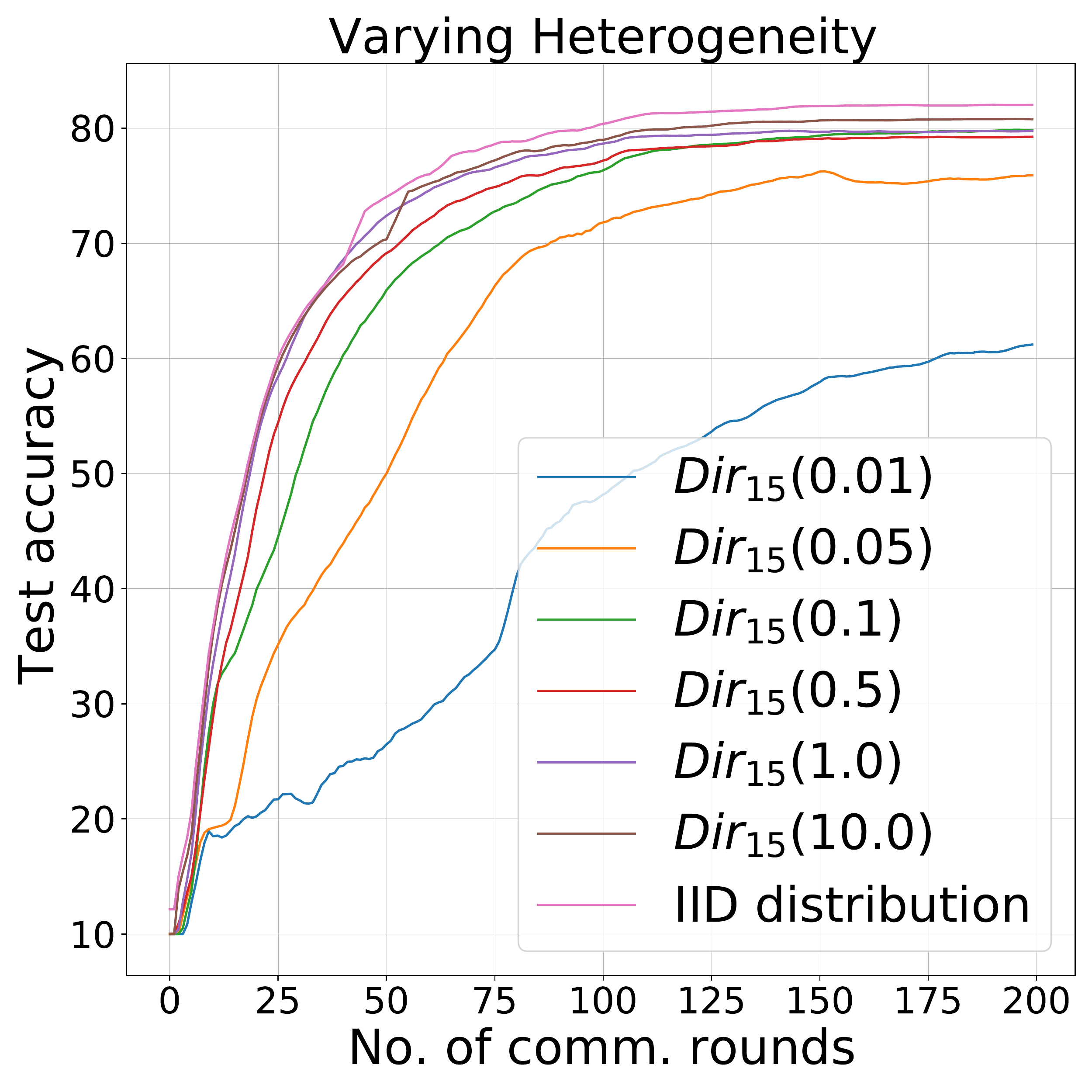}
    \caption{Effect of inter-client data heterogeneity (quantified by $\alpha$) on the performance of \fedsgdaplus. \label{fig:robustNN_vary_hetero}}
\end{figure}

\begin{figure}[h]
    \centering
    \includegraphics[width=0.25\textwidth]{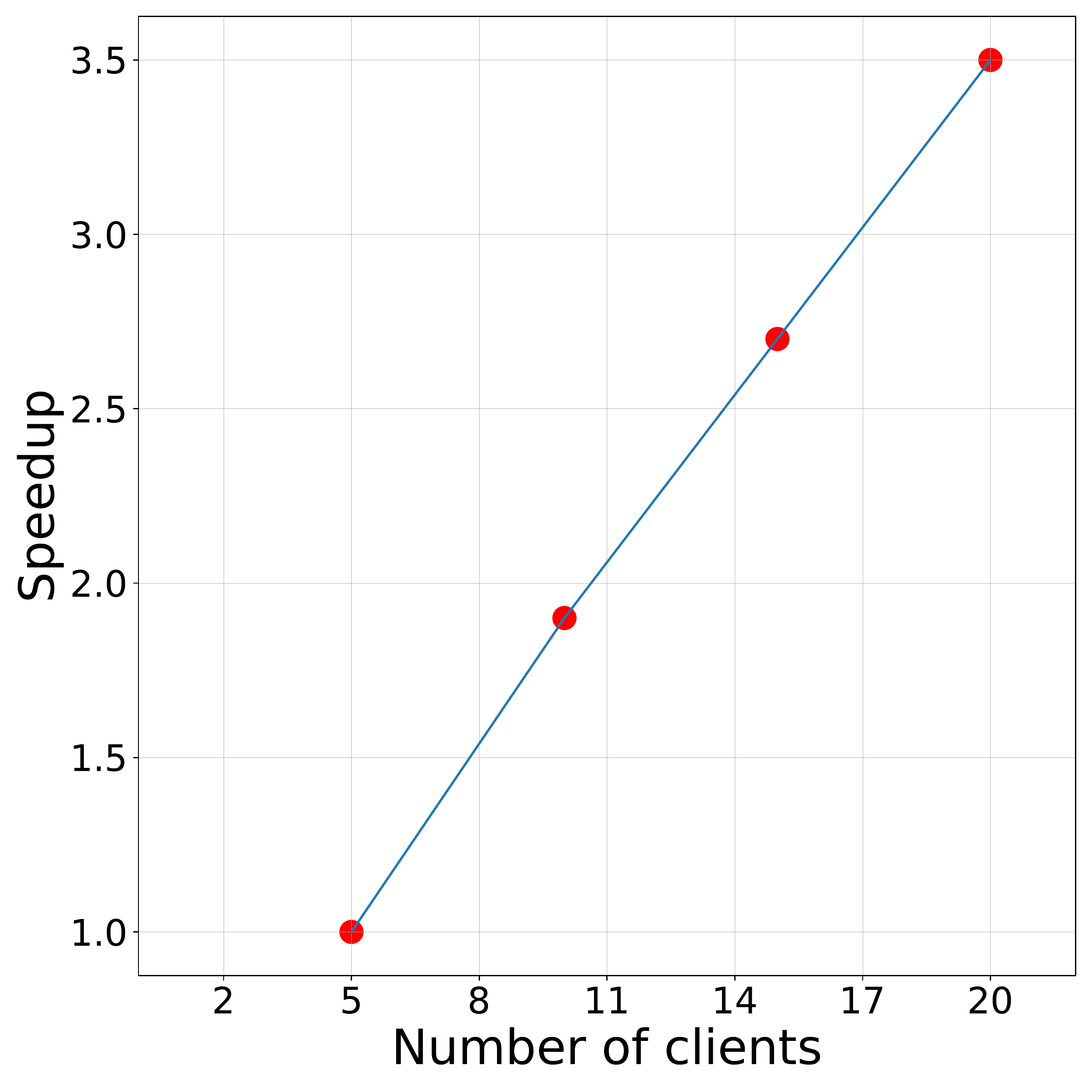}
    \caption{Effect of increasing client-set on the performance of \fedsgdaplus \ in a robust NN training task. \label{fig:robustNN_linear_speedup}}
\end{figure}

\paragraph{Fair Classification}
We also demonstrate the impact of partial client participation in the fair classification problem. \cref{fig:fairclass_partial} complements \cref{fig:fairclass_partial} in the main paper, evaluating fairness of a VGG11 model on CIFAR10 dataset. We have plotted the test accuracy of the model over the worst distribution. With an increasing number of participating clients, the performance consistently improves.

\begin{figure}[t]
    \centering
    \includegraphics[width=0.45\textwidth]{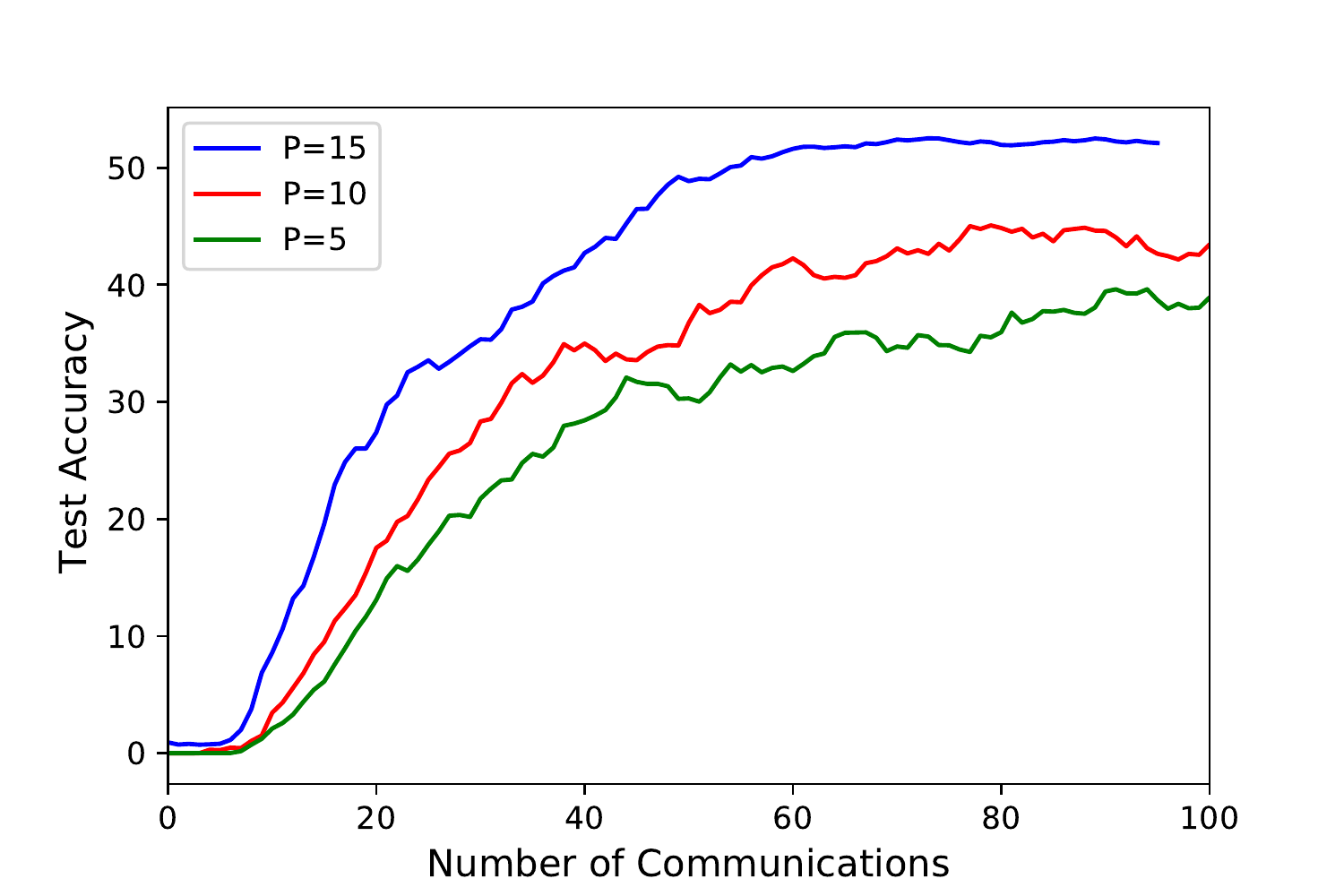}
    \caption{Effect of partial client participation on the performance of \fedsgda \ in a fair image classification task. \label{fig:fairclass_partial}}
\end{figure}